\documentclass[twoside,11pt]{article}
\usepackage{graphicx}
\usepackage{booktabs}
\usepackage{blindtext}
\usepackage{bm}
\usepackage[abbrvbib, preprint]{aux/jmlr2e}

\usepackage{mathtools}
\mathtoolsset{showonlyrefs}
\usepackage{enumerate,enumitem}
\usepackage{tikz}

\makeatletter
\newcommand{\doublewidetilde}[1]{{%
  \mathpalette\double@widetilde{#1}%
}}
\newcommand{\double@widetilde}[2]{%
  \sbox\z@{$\m@th#1\widetilde{#2}$}%
  \ht\z@=.9\ht\z@
  \widetilde{\box\z@}%
}
\makeatother

\newcommand{\inner}[1]{\left\langle #1 \right\rangle}

\newcommand{\norm}[1]{\left\|#1\right\|}

\newcommand{\by}{\boldsymbol{y}}

\newcommand{\bw}{\boldsymbol{w}}
\newcommand{\bepsilon}{\boldsymbol{\epsilon}}

\def\ddefloop#1{\ifx\ddefloop#1\else\ddef{#1}\expandafter\ddefloop\fi}
\def\ddef#1{\expandafter\def\csname c#1\endcsname{\ensuremath{\mathcal{#1}}}}

\ddefloop ABCDEFGHIJKLMNOPQRSTUVWXYZ\ddefloop
\def\ddef#1{\expandafter\def\csname s#1\endcsname{\ensuremath{\mathsf{#1}}}}
\ddefloop ABCDEFGHIJKLMNOPQRSTUVWXYZ\ddefloop

\ddefloop ABCDEFGHIJKLMNOPQRSTUVWXYZ\ddefloop
\def\ddef#1{\expandafter\def\csname b#1\endcsname{\ensuremath{\mathbf{#1}}}}
\ddefloop ABCDEFGHIJKLMNOPQRSTUVWXYZ\ddefloop

\renewcommand{\P}{\mathbb{P}}
\newcommand{\R}{\mathbb{R}}
\newcommand{\E}{\mathbb{E}}

\newcommand{\N}{\mathbb{N}}

\def\bI{{\mathbf I}}

\newcommand{\diag}{\text{diag}}
\newcommand{\tr}{\text{tr}}

\newcommand{\ones}{\mathbf{1}}

\def\bA{{\boldsymbol A}}
\def\bB{{\boldsymbol B}}

\def\bD{{\boldsymbol D}}

\def\bF{{\boldsymbol F}}
\def\bG{{\boldsymbol G}}

\def\vf{{\boldsymbol f}}

\def\bK{{\boldsymbol K}}
\def\bL{{\boldsymbol L}}

\def\bR{{\boldsymbol R}}

\def\bT{{\boldsymbol T}}
\def\bU{{\boldsymbol U}}
\def\bV{{\boldsymbol V}}

\def\bX{{\boldsymbol X}}

\def\bZ{{\boldsymbol Z}}
\def\bM{{\boldsymbol M}}

\def\bb{{\boldsymbol b}}

\def\bg{{\boldsymbol g}}

\def\bu{{\boldsymbol u}}
\def\bv{{\boldsymbol v}}
\def\bw{{\boldsymbol w}}
\def\bx{{\boldsymbol x}}
\def\by{{\boldsymbol y}}
\def\bz{{\boldsymbol z}}

\def\bbeta{{\boldsymbol \beta}}

\def\bDelta{{\boldsymbol \Delta}}

\def\bSigma{{\boldsymbol \Sigma}}

\def\supp{{\rm supp}}

\def\cV{{\mathcal V}}

\def\cP{{\mathcal P}}

\def\cE{{\mathcal E}}

\def\cV{{\mathcal V}}

\def\cP{{\mathcal P}}

\def\K{{\mathbb K}}

\def\T{{\mathbb T}}

\def\T{{\mathbb T}}
\def\K{{\mathbb K}}

\def\rank{{\rm rank}}

\newcommand{\bKquad}{\bK^{(2)}} 

\renewcommand{\bK}{\boldsymbol{K}}
\renewcommand{\K}{\text{K}}
\renewcommand{\T}{\boldsymbol{T}}
\newcommand{\F}{\boldsymbol{F}}
\newcommand{\V}{\boldsymbol{V}}

\renewcommand{\L}{\boldsymbol{L}}
\newcommand{\ignore}[1]{}

\newcommand{\Etrain}{\cE_{\textnormal{train}}}

\newcommand{\train}{\text{train}}

\newcommand{\Tr}{\operatorname{Tr}}
\newcommand{\x}{\boldsymbol{x}}
\newcommand{\y}{\boldsymbol{y}}
\newcommand{\z}{\boldsymbol{z}}

\newcommand{\vmu}{\mbi{\mu}}
\newcommand{\vnu}{\mbi{\nu}}
\newcommand{\mix}{\textnormal{mix}}
\newcommand{\w}{\mbi{w}}

\renewcommand{\P}{\mathbb{P}}

\newcommand{\1}{\mathbf{1}}

\renewcommand{\Re}{\operatorname{Re}}

\def\mbi#1{\boldsymbol{#1}}
\def\v#1{\mbi{#1}} 
\def\mbi#1{\boldsymbol{#1}} 

\def\mc#1{\mathcal{#1}}

\def\mc{\mathcal}
\def\msf{\mathsf}
\def\tran{^\top}

\def\Real{\mathbb{R}}
\def\Hilbert{\mc H}
\newcommand{\abs}[1]{\left|#1\right|}

\def\one{\mathbf{1}}

\newcommand{\round}[1]{\left(#1\right)}

\newcommand{\Exp}{\mathbb{E}}
\newcommand{\lbX}{\overline{\bX}^{(2)}}
\newcommand{\lbXt}{\overline{\bX}^{(2)\top}}

\newenvironment{proofoftheorem}[1]{\par\noindent{\bf Proof of Theorem #1\ }}{\hfill\BlackBox\\[2mm]}
\newenvironment{proofoflemma}[1]{\par\noindent{\bf Proof of Lemma #1\ }}{\hfill\BlackBox\\[2mm]}

\begin{document}

\title{Universality of Kernel Random Matrices and Kernel Regression in the Quadratic Regime}

\author{\name Parthe Pandit \email pandit@iitb.ac.in \\
       \addr Center for Machine Intelligence and Data Science\\
       Indian Institute of Technology Bombay\\
       Mumbai, Maharashtra 400076, India
       \AND
       \name Zhichao Wang \email zhichao.wang@berkeley.edu \\
       \addr International Computer Science Institute\\
       Department of Statistics\\
       University of California, Berkeley\\
       Berkeley, CA 94720, USA
       \AND
       \name Yizhe Zhu \email yizhezhu@usc.edu\\
       \addr Department of Mathematics\\
       University of Southern California\\
       Los Angeles, CA 90089, USA
}

\maketitle

\begin{abstract}%
Kernel ridge regression (KRR) is a popular class of machine learning models that has become an important tool for understanding deep learning.  Much of the focus thus far has been on studying the proportional asymptotic regime, $n \asymp d$, where $n$ is the number of training samples and $d$ is the dimension of the dataset. In the proportional regime, under certain conditions on the data distribution, the kernel random matrix involved in KRR exhibits behavior akin to that of a linear kernel. In this work, we extend the study of kernel regression to the quadratic asymptotic regime, where $n \asymp d^2$. In this regime, we demonstrate that a broad class of inner-product kernels exhibits behavior similar to a quadratic kernel. Specifically, we establish an operator norm approximation bound for the difference between the original kernel random matrix and a quadratic kernel random matrix with additional correction terms compared to the Taylor expansion of the kernel functions. The approximation works for general data distributions under a Gaussian-moment-matching assumption with a covariance structure. This new approximation is utilized to obtain a limiting spectral distribution of the original kernel matrix and characterize the precise asymptotic training and test errors for KRR in the quadratic regime when $n/d^2$ converges to a non-zero constant. The generalization errors are obtained for (i) a random teacher model, (ii) a deterministic teacher model where the weights are perfectly aligned with the covariance of the data. Under the random teacher model setting, we also verify that the generalized cross-validation (GCV) estimator can consistently estimate the generalization error in the quadratic regime for anisotropic data. Our proof techniques combine moment methods, Wick's formula, orthogonal polynomials, and resolvent analysis of random matrices with correlated entries. 
\end{abstract}



\section{Introduction}
Deep neural networks have become the dominant class of models in machine learning, breaking new benchmarks every few weeks. A certain architecture of deep neural networks, wide neural networks, is closely related to the kernel methods \citep{jacot2018neural}. Kernel methods \citep{scholkopf2002learning,williams2006gaussian} also exhibit many phenomena previously thought to be specific to deep neural networks \citep{belkin2018understand}.
Consequently, understanding kernel models in high-dimensional limits has gathered a lot of renewed attention due to their analytical traceability \citep{radhakrishnan2024mechanism}. 

In recent years, the study of kernel ridge regression (KRR) in high-dimensional settings has gained attention due to its relevance in understanding modern machine learning phenomena such as benign overfitting \citep{bartlett2020benign,tsigler2023benign,bartlett2021deep} and the double descent risk curve \citep{belkin2019reconciling,mei2019generalization}. High-dimensional asymptotics reveal that models can generalize well even in regimes where the number of parameters far exceeds the number of data points.  The multiple descent curves \citep{liang2020multiple} observed in some settings further enrich this landscape. Notably, the emergence of the neural tangent kernel (NTK) framework \citep{jacot2018neural} has provided a powerful framework to analyze the training dynamics and generalization behavior of overparameterized neural networks. NTK connects infinite-width neural networks to kernel methods, such as KRR, allowing for a tractable theoretical analysis and shedding light on how such overparameterized models exhibit generalization.

A particular line of attack towards understanding kernel methods has been using asymptotic analysis via random matrix theory \citep{el2010spectrum,mei2019generalization,bartlett2021deep,montanari2020interpolation}.  The key argumentative piece in these results is that kernel matrices in the \textit{proportional asymptotic regime}, i.e., $n\asymp d$ where $n$ is the sample size and $d$ is the feature dimension of the dataset $\bX$, are well approximated by the Gram matrix of the input data. 
Consequently, in this regime, the kernel models are somewhat degenerate and can only be as powerful as linear models \citep{bartlett2021deep,ba2022high}.
While this has provided us with many interesting insights, intuitions, and limitations of kernel methods, the scope of this asymptotic regime is limited. Many researchers have analyzed the more general polynomial regime of $n\asymp d^\ell$, for $\ell>1$, e.g., \cite{mei2021generalization,donhauser2021rotational,xiao2022precise,lu2022equivalence,dubova2023universality,wang2023overparameterized}. However, general covariance structures of the data distribution were not considered in most of the previous works beyond the linear regime. One of our motivating questions in this paper is to tackle this situation:
\begin{center}
    \textit{What is the asymptotic behavior of kernel regression beyond the proportional regime for general data distribution with a covariance structure?}
\end{center}

In this work, we make headway into this question in the \textit{asymptotic quadratic regime}, i.e., $n\asymp d^2$. For a large class of inner-product kernels, the kernel matrices for high-dimensional datasets are well approximated by a degree-2 polynomial kernel matrix, which depends on the data matrix $\bX$ and the kernel function $f$. Using this approximation, we derive the precise description of the limiting eigenvalue distribution of the kernel random matrix under this asymptotic quadratic regime and study the corresponding kernel regression problem with precise asymptotics for training and generalization errors.

\subsection{Main contributions}
We study a large class of inner-product kernels
\begin{align} \label{eq:kernel_function}
    K(\bx,\bz) = f\round{\frac{\inner{\x,\z}}{d}}, \quad  \bx, \bz  \in \mathbb R^d.
\end{align}
Consider independent random vectors $\bx_1,\dots, \bx_n$ in $\mathbb R^{d}$ with a  covariance  structure $\bSigma$.  Denote the data matrix by $\bX\in\R^{n\times d}$. The kernel function in \eqref{eq:kernel_function} applied to the dataset induces a kernel random matrix $\bK\in \mathbb R^{n\times n}$ such that 
$K_{ij}= f\round{\frac{\inner{\bx_i,\bx_j}}{d}}$. We prove that under regularity assumptions for $f$ and certain moment conditions on $\bx_i,$ for $i\in [n]$, when $n\asymp d^2$, the kernel matrix behaves as a quadratic kernel.

In summary, we show the following three main results:
\begin{itemize}
    \item When $n=O(d^2)$, with high probability, the kernel random matrix $\bK$ can be approximated by a quadratic kernel random matrix $\bK^{(2)}$ under the spectral norm, where 
\begin{align}\label{eq:Kquad_intro}
    \bKquad = a_0 \one\one\tran + a_1 \bX\bX\tran + a_2 (\bX\bX\tran)^{\odot 2} + a\bI_n,
\end{align}
and  $a_0,a_1,a_2,a$ are constants depending on $f$ and the covariance $\bSigma$ given in \eqref{eq:Kquad}. Here $(\bX\bX\tran)^{\odot 2}$ is the Hadamard product of $\bX\bX^\top$  with itself.   
 Our non-asymptotic concentration bound works for non-isotropic data under a mild moment-matching condition. In particular, it holds for  Gaussian data with a covariance matrix $\bSigma$. The precise statement is given in Theorem \ref{thm:concentration}. The spectral norm approximation bound shows that $\bK$ can be asymptotically decomposed as a low-rank part, a quadratic kernel, and a regularization term. The structural result is important for understanding kernel ridge regression (KRR) in the quadratic regime.
\item When $n\to\infty$ and $\frac{d^2}{2n}\to\alpha$, we show the limiting spectral distribution of $\bK$ is given by a deformed Marchenko-Pastur law, which depends on the aspect ratio $\alpha$ and the covariance structure $\bSigma$. The detailed statement can be found in Theorem~\ref{thm:globallaw}.
\item Based on the above results, we study the performance of KRR with the kernel function $K$ in \eqref{eq:kernel_function} and random training data $\bX$. Our analysis reveals that the training and generalization error for KRR with kernel $\bK$ can be approximated by the quadratic kernel  $\bK^{(2)}$.
The asymptotic training error is presented in Theorem~\ref{thm:train_limit}. The asymptotic generalization error is characterized in Theorems~\ref{thm:test_limit} and~\ref{thm:test_limit_deterministic} for different teacher models.
To fulfill the proofs in generalization error, we provide a novel concentration inequality for quadratic forms of centered random tensor vectors and a general deterministic equivalence for spectral functions of a centered version of $(\bX\bX\tran)^{\odot 2}$; see Section~\ref{sec:appendix_prelim} for more details.  
\end{itemize}

\subsection{Related work}

\paragraph{Kernel random matrices.} The study of kernel random matrices has been an important topic in random matrix theory and high-dimensional statistics. For inner-product kernels, in the proportional regime where $n\asymp d$, there are two types of random matrix models in the literature. For $\bK_{ij}=f(\langle \bx_i, \bx_j\rangle /\sqrt{d})$, the limiting spectral distribution was first studied by \cite{cheng2013spectrum,do2013spectrum}. The concentration of the spectral norm was then analyzed by \cite{fan2019spectral}. For a different scaling where $\bK_{ij}=f(\langle \bx_i, \bx_j\rangle /d)$, the limiting spectral distribution and spectral norm bound were investigated by \cite{do2013spectrum,el2010spectrum,el2010information,amini2021concentration}.  When $f=x^k$,  $\bK$ is related to random tensor models recently considered in random matrix literature \citep{ambainis2012random,bryson2021marchenko,collins2022spectral,yaskov2023marchenko,baslingker2023hadamard,goulart2022random,au2023spectral}. 
In the polynomial regime, recently, \cite{lu2022equivalence,dubova2023universality} considered the spectrum of inner-product kernel matrices and proved a spectral universality result. Their kernel matrix is of the form $\bK_{ij}=f(\langle \bx_i, \bx_j\rangle /\sqrt{d})$ whose scaling is different from ours, which is $\bK_{ij}=f(\langle \bx_i, \bx_j\rangle /d)$. Although their scaling may better exhibit the bulk information from the nonlinear function, our matrix concentration and limiting law results can be directly applied to kernel regression training and generalization errors. An example class of inner-product kernels is of the form $K(\bx,\bz) = \Exp_{\w}[\sigma(\w\tran \x)\sigma(\w\tran \z)],$
where $\w$ is drawn from an isotropic Gaussian distribution when data vectors are of unit length \citep{wang2021deformed,murray2023characterizing}.

\paragraph{Kernel ridge regression in the polynomial regime.}
 When $n\asymp d$, the spectral analysis of rotational invariant kernels including \eqref{eq:kernel_function}, as studied by \cite{el2010spectrum}, has been applied to the study of KRR by \cite{liang2020just,elkhalil2020risk,liu2021kernel,bartlett2021deep,sahraee2022kernel}. Under the same regime, kernel spectral clustering has also been analyzed by \cite{couillet2016kernel,liao2019inner,seddik2019kernel,seddik2019kernelb,liao2021sparse,li2025eigen} in terms of informative and non-informative eigenstructures in the kernel matrices induced by nonlinearity. Beyond the proportional case, for general data distribution, \cite{liang2020multiple,donhauser2021rotational,aerni2023strong,lu2023optimal} provided bias and variance bounds of the generalization error for the consistency of KRR; and under certain data assumptions, \cite{ghorbani2020neural,ghorbani2021linearized,mei2021generalization} precisely showed that KRR can only learn low-degree polynomials based on the sample complexity $n$. When $n\asymp d^k$, for $k\in\N$, the performance of inner-product kernel with data uniformly drawn from the unit sphere $\mathbb{S}^{d-1}$ has been recently studied by \cite{xiao2022precise}, then, \cite{misiakiewicz2024non} proved a dimension-free approximation of KRR via a non-asymptotic deterministic equivalence given some concentration of the eigenfunctions in the spectral decomposition of the kernel. Recently, \cite{barzilai2023generalization,cheng2024characterizing} considered a non-asymptotic generalization error bound for KRR under a general setting and obtained conditions for benign over-fitting. Building on the work of \cite{liang2020multiple,ghorbani2021linearized}, \cite{gavrilopoulos2024geometrical} provided a more precise upper bound for the test error of KRR under a sub-Gaussian design. This advancement has been applied to data-dependent conjugate kernels, contributing to the research on trained features in feature learning \citep{ba2022high,gavrilopoulos2024geometrical}.

\paragraph{Random feature models.} Random feature models, as an efficient approximation of limiting kernel random matrices \citep{rahimi2007random,liu2021random}, have gained significant interest in deep learning \citep{pennington2017nonlinear,louart2018random}. In the ultra-wide neural networks \citep{arora2019exact}, random feature ridge regression (RFRR) is asymptotically equivalent to a kernel ridge regression (KRR) model  \citep{jacot2018neural,novak2018bayesian,matthews2018gaussian,wang2021deformed,wang2023overparameterized}, whose kernel is in the form of $K(\bx,\bz) = \Exp_{\w}[\sigma(\w\tran \x)\sigma(\w\tran \z)],$ with a Gaussian random vector $\w$. When the width is proportional to $n$ and $d$, while the random feature matrix will not converge to the corresponding kernel, the asymptotic behavior of RFRR remains tractable via random matrix theory. \cite{mei2019generalization,adlam2020neural,liao2020random,gerace2020generalisation,goldt2022gaussian,hu2020universality} showed that it is comparable to that of a linear model. Moreover, \cite{hu2020universality} concerns Gaussian equivalence of random feature models beyond the regression setting and
proves a conjecture from \cite{gerace2020generalisation,goldt2022gaussian}  that in the proportional limit, Gaussian universality holds for random feature models beyond the square loss. In the proportional regime, deterministic equivalence and generalization errors of deep random features were studied in \citep{schroder2023deterministic,schroder2024asymptotic}. Notably, their random matrix results hold under general distributional assumptions of the feature vectors $\phi(\bx)$ in the proportional regime, while this work studies KRR in the quadratic regime under distributional assumptions on data vectors $\bx$.

Beyond the proportional regime, most of these results considered the RFRR with the data points independently drawn from a specific high-dimensional distribution, e.g., uniform measure on the hypercube or $\mathbb{S}^{d-1}$ \citep{ghorbani2021linearized,hu2024asymptotics} or under the hypercontractivity assumption from \cite{mei2021generalization}. 
Very recently, \cite{latourelle2023matrix} studied the generalization error of RFRR for deterministic datasets, and \cite{defilippis2024dimension} studied the deterministic equivalence of the generalization error under the concentration property of eigenfunctions.
The asymptotic spectra of these random features or empirical NTK in neural networks have been investigated by
\cite{pennington2017nonlinear,louart2018random,mei2019generalization,fan2020spectra,benigni2019eigenvalue,benigni2022largest,wang2021deformed,wang2024nonlinear,benigni2025eigenvalue,liao2025random}. Additionally, \cite{liao2018spectrum} studied the inner-product kernel induced by random features in the proportional limit.

\paragraph{Quadratic regime and learning a quadratic function.} 
The quadratic regime has appeared in various tasks as an extension of the linear regime. \cite{chetelat2019middle} analyzed phase transition behavior for the GOE approximation of Wishart distributions in the regimes where $d=n^{\frac{k+1}{k+3}},k\in\mathbb N$ with $k=1$ corresponding to the quadratic regime. As another example, the ellipsoid fitting conjecture \citep{saunderson2013diagonal} with a threshold $n=d^2/4$ lies within this regime and was resolved by \cite{hsieh2023ellipsoid,tulsiani2023ellipsoid,bandeira2023fitting} up to a constant. Here, \cite{hsieh2023ellipsoid} utilized a constructed random matrix closely related to our model \eqref{eq:Kquad_intro}. 
In our results, we evaluate KRR under the quadratic regime to learn a quadratic function. 
The classical phase retrieval model \citep{walther1963question,balan2006signal} belongs to this learning problem. The learning dynamic of two-layer neural networks 
to learn a quadratic target function has been studied by \cite{sarao2020complex,arnaboldi2023escaping,martin2024impact}. More closely related to our work, \cite{ghorbani2019limitations} examined the population loss of random features with quadratic activation functions to learn a quadratic teacher.

\subsection{Technical novelties} 
\begin{table}[htbp] 
\centering
\scriptsize
\begin{tabular}{@{}llllll@{}}
\toprule
\textbf{Paper} & \textbf{Regime} & \textbf{Data Assumptions} & \textbf{Kernel Approximation}  \\
\midrule
\cite{el2010spectrum} &  $n \sim d$ & General covariance & Firsr-order Taylor expansion \\
\midrule
\begin{tabular}{@{}c@{}}\cite{ghorbani2021linearized} \\
\& \cite{mei2021generalization}
\end{tabular}
&  $ d^{k+\delta}\le n \le d^{k+1-\delta}$ & \begin{tabular}{@{}c@{}}
Specific distributions with \\hypercontractivity conditions \end{tabular}
& $k$-th orthogonal polynomials \\
\midrule
\cite{xiao2022precise} &  $n \sim d^k$ & Uniform measure on the  sphere & 
Gegenbauer polynomials 
\\
\midrule
\textbf{This paper} & $n \sim d^2$ & \begin{tabular}{@{}c@{}}
General covariance \\ under the moment-matching condition
\end{tabular}
& \begin{tabular}{@{}c@{}}
Second-order Taylor  \\
expansion with corrections \end{tabular}\\
\bottomrule
\end{tabular}
\caption{Comparison of related work on KRR under the polynomial regimes ($\delta\in(0,\frac12)$).}\label{table:compare}
\end{table}

This paper advances the theoretical understanding of kernel ridge regression (KRR) by extending analysis beyond the commonly studied proportional regime (where sample size $n\asymp d$ to the quadratic regime $n\asymp d^2$. The central contribution is a rigorous approximation of a broad class of inner-product kernel matrices by a quadratic kernel matrix, under general covariance structures. 
This includes:
\begin{itemize}
    \item Spectral norm approximation: A non-asymptotic bound that shows kernel matrices behave like quadratic kernel matrices with correction terms, not just Taylor approximations.
    \item Limiting spectral distribution: A novel characterization of the eigenvalue distribution of the kernel matrix using deformed Marchenko-Pastur laws.
    \item Precise training and generalization error analysis: Asymptotic formulas for training and generalization errors of KRR with both random and deterministic quadratic teacher functions.
\end{itemize}

Compared to the existing work \citep{mei2019generalization,xiao2022precise,montanari2020interpolation,mei2021generalization} on the precise asymptotic performance of KRR under specific distribution assumptions, e.g., uniform measure on $\mathbb{S}^{d-1}$ and the hypercube, we make no specific distribution assumption and do not require all moments of the data distribution to be bounded. Instead, we require a moment-matching condition with a Gaussian distribution. Our result does not share the same condition as \cite{xiao2022precise} since their data satisfies the uniform measure on the sphere, whose first 8 moments do not match those of a Gaussian. But formally, our asymptotic generalization error formula in Theorem~\ref{thm:test_limit}, when taking $\bSigma=\bI$, agrees with their result in the quadratic regime $n\asymp d^2$.  Our result is new even for isotropic Gaussian data when $n\asymp d^2$. We provide the first asymptotic analysis of KRR beyond the linear regime for anisotropic data with a covariance structure. Our technical assumption is the \textit{Gaussian moment matching} condition, which is necessary in our moment method proof of kernel approximation in Theorem~\ref{thm:concentration}. It is used to explore the orthogonal properties of the Hermite polynomial in the proof of Theorem~\ref{thm:test_limit}. In addition, compared with \cite{xiao2022precise,ghorbani2021linearized,mei2021generalization}, we impose a stronger smoothness condition on the kernel function $f$. We view this as a technical assumption that will likely be relaxed in future work. We summarized the comparison in Table~\ref{table:compare}.

To prove the concentration result,  we revisit the idea of Taylor expansion of kernel functions in \citep{el2010spectrum}. Different from \cite{el2010spectrum}, the higher-order error terms from the Taylor expansion are more challenging to bound, and new ``correction terms" not seen from the Taylor approximation appear in our corresponding quadratic kernel $\bK^{(2)}$. 
We then apply a trace method to control the error from higher-order expansion. 
Although a direct Hermite expansion relies on weaker regularity assumptions on the kernel function $f$ \citep{mei2021generalization}, without the isotropic Gaussian data assumption, controlling the approximation error of the $\mathbf{K}-\mathbf{K}^{(2)}$ becomes more challenging since each degree$-\ell$-Hermite polynomial contains lower order terms and it's difficult to argue they have negligible contribution.

Under the spectral norm, we can approximate $\bK$ by a simpler quadratic kernel $\bK^{(2)}$ defined in \eqref{eq:Kquad_intro}. By standard perturbation analysis,  $(\bX\bX\tran)^{\odot 2}$ is the leading term in the limiting spectrum of $\bK$. With the ``kernel trick" (see, e.g., \cite[Exercise 3.7.4]{vershynin2010introduction}), we can write $(\bX\bX\tran)^{\odot 2}$ as a Gram matrix with tensor vectors $\bx_i^{\otimes 2}, i\in [n]$.  We then use the result of \cite{bai2008large} for sample covariance matrices to study their limiting spectrum.

Finally, equipped with the random matrix results above, we characterize the asymptotic performance of KRR. The analysis relies on the connection between the spectrum of $\bK$ and the prediction risks of KRR. We carefully quantify the approximation error when replacing $\bK$ with $\bK^{(2)}$ in the training and generalization errors for KRR with $\bK$. After this simplification, we analyze the asymptotic behavior of KRR with a quadratic kernel $\bK^{(2)}$. Then, the challenge becomes to establish the deterministic equivalences of some functional of $\bK^{(2)}$ and its resolvent. To fulfill this, we establish a new concentration inequality (Lemma~\ref{lem:quad_high_power}) related to random quadratic forms of $\bx_i^{\otimes 2}$.

\subsection{Preliminaries} \label{sec:prelim}
 \paragraph{Notation.} We refer to vectors in boldcase ($\bx$), matrices in bold uppercase ($\bX$), scalars in normalcase ($x$). We use $\|\bx\|$ as the $\ell_2$-norm of a vector. For a matrix $\bX$, $\norm{\bX}$ is its operator norm and $\norm{\bX}_{\msf F}$ is its Frobenius norm. We use $K$ to represent a kernel function and $\bK$ to denote a kernel random matrix. $\bI_n$ denotes the $n\times n$ identity matrix.  $\mathbb E_{\bx}[\cdot]$ means the expectation is only taken over the random vector $\bx$, conditioned on everything else.  we use $a_n\lesssim b_n$ to indicate $a_n\leq Cb_n$ for some  constant $C$ independent of $n,d$.
 
  For a  vector $\bx\in\Real^{d}$ we denote its \textit{tensor product} by $\bx^{ \otimes 2}\in\Real^{ d^2}$ whose index set is $\{(i,j): i,j\in [d] \}$ such that  $\left(\bx^{ \otimes 2}\right)_{i,j}=\bx(i)\bx(j)$, where $\bx(j)$ is the $j$-th entry of vector $\bx$.
 For a matrix $\bA $ whose $(i,j)$-th entry is $a_{i,j}$, we denote the $k$-th \textit{Hadamard product} of $\bA$ as $\bA^{\odot k}$ whose $(i,j)$-th entry is $a_{ij}^k$, for any $k\in\N$. We will use the following equation: given a matrix $\bX\in \R^{n\times d}$, the $(i,j)$-th entry of $(\bX\bX\tran)^{\odot k}$ is
\begin{align}\label{eq:tensor_representation}
    [(\bX\bX\tran)^{\odot k}]_{ij} := \inner{\x_i,\x_j}^k = \inner{\x_i^{\otimes k},\x_j^{\otimes k}},
\end{align}
for $i,j\in[n]$, where $\x_i\tran$ is the $i$-th row of $\bX$, and the the inner product between $\x_i^{\otimes k}$ and $\x_j^{\otimes k}$ is the vector inner product in $\mathbb R^{d^k}$.

\paragraph{Random matrix theory.}
We include several definitions from random matrix theory.  For any $n\times n$ Hermitian matrix $\bA_n$ with eigenvalues $\lambda_1,\dots, \lambda_n$, the empirical spectral distribution of $\bA_n$ is defined by 
$ \mu_{\bA_n}=\frac{1}{n}\sum_{i=1}^n \delta_{\lambda_i}$. If $\mu_{\bA_n}\to\mu$ weakly as $n\to\infty$, then we call $\mu$ the limiting spectral distribution of $\bA_n$.  The \textit{Marchenko-Pastur law} \citep{marchenko1967distribution} with a parameter $\gamma\in (0,+\infty)$ has a  density:
\begin{align}\label{eq:def_mualpha}
   \mu_{\gamma}^{\mathrm{MP}}&= \begin{cases}
        (1-\gamma^{-1})\delta_0+\nu_{\gamma},  & \gamma>1,\\
        \nu_{\gamma}, & \gamma\in (0,1], 
    \end{cases} \quad \text{ where}\\
 d\nu_{\gamma}(x)&=\frac{1}{2\pi}\frac{\sqrt{(\gamma_{+}-x)(x-\gamma_-)}}{\gamma x} \mathbf{1}_{x\in [\gamma_{-},\gamma_+]} dx,  \quad \gamma_{\pm} :=(1\pm \sqrt{\gamma})^2. \label{eq:def_nualpha}
\end{align}
Note that when $\gamma>1$, the total mass of $\nu_{\gamma}$ is $\gamma^{-1}$ and when $\gamma \in (0,1)$, its total mass  is 1.

\subsection{Organization of the paper} The rest of the paper is organized as follows. Precise and detailed statements of our main results are given in Section~\ref{section:main_result}. Additional definitions and lemmas are given in Appendix~\ref{sec:appendix}. Proof of the result for spectral norm approximation (Theorem \ref{thm:concentration}) is given in Appendix \ref{sec:concentration}.  The proof of the limiting spectral distribution (Theorem \ref{thm:globallaw}) is provided in Appendix \ref{sec:limiting_law}. In Appendices \ref{sec:krr} and \ref{sec:krr_test}, we provide the proof for the results on training error (Theorem \ref{thm:train_limit}) and generalization error (Theorem \ref{thm:test_limit} and Theorem~\ref{thm:test_limit_deterministic}) for kernel ridge regression, respectively.

\section{Main results}
\label{section:main_result}

\subsection{Quadratic approximation of inner-product kernel matrices}
Consider kernel function of the form $K(\bx,\bz) = f\round{\frac{\inner{\bx,\bz}}{d}}$, where $f$ is a function independent of $n, d$.
 Let $\bx_i$ be independent random vectors in $\mathbb R^d$ $i\in [n]$. Consider random kernel matrix $\bK\in\R^{n\times n}$ such that it $(i,j)$-th entry is defined by  $\bK_{ij} = K(\x_i,\x_j), i,j\in[n]$.

Our results will be stated under the following assumptions on the data distribution and the kernel function $f$.
\begin{assumption}\label{assump:ratio}
We assume that, for some absolute  constant $C_1>0$, $\frac{n}{d^{2}}\leq C_1$.
\end{assumption}
 \begin{assumption}\label{assump:data}
We assume that $\bx_i=\bSigma^{1/2} \bz_i\in\R^d$, where $\bSigma$  is a $d\times d$ positive semi-definite  matrix, and $\z_i\in\Real^d$ is a random vector with independent entries. Furthermore, for $i\in [n]$,  $k\in [d]$, 
$\mathbb E[(\bz_{i}(k))^{t}]=\mathbb E[g^t], t=1,2,\dots, 8$, where  $g\sim \cN(0,1)$. And $\mathbb E[|\bz_{i }(k)|^{90}]^{\frac{1}{90}}\leq C_2$  for some constant $C_2> 0$, and $\z_1,\dots,\z_n$ are independent.
\end{assumption}
Note that in Assumption~\ref{assump:data}, $\z_1,\dots,\z_n$ can have different distributions. Similar to Assumption~\ref{assump:data}, Gaussian moment matching assumptions also appear in non-Gaussian component analysis \citep{dudeja2022statistical} and the universality of local spectral statistics in random matrix theory \citep{tao2011random}.  We did not try to optimize the bounded moment assumption. The finite $90$-th moment condition in Assumption~\ref{assump:data} is convenient for deriving a $1-O(d^{-1/2})$ probability tail bound in Theorem~\ref{thm:concentration}.

\begin{assumption}\label{assump:sigma} $\|\bSigma\|\leq C_3$ for some constant $C_3>0$, and there exists   $\tau>0$ such that 
$\tau=\lim_{d\to\infty} \frac{\Tr\bSigma}{d}$.
\end{assumption}
\begin{assumption}\label{assump:nonlinear_f}
Kernel function $f: \mathbb R \to \mathbb R$ is a $C^{2}$-function in a neighborhood of $\tau$, and is $C^{5}$ in a neighborhood of $0$.
\end{assumption}

Denote the data matrix by $\bX\in\R^{n\times d}$, where all row vectors in $\bX$ are independent and satisfy Assumption~\ref{assump:data}. Under all the assumptions above, we introduce the following \textit{quadratic kernel matrix} $\bK^{(2)}$ as an approximation of $\bK$, where
\begin{align}
\bKquad=&\left(f(0)-\frac{f^{(4)}(0)(\Tr(\bSigma^2))^2}{8d^4}\right)\mathbf{1}\mathbf{1}^\top+\left(\frac{f'(0)}{d}+\frac{f^{(3)}(0)\Tr(\bSigma^2)}{2d^3} \right)\bX\bX ^\top 
\\&+\left(\frac{f''(0)}{2d^2}+ \frac{f^{(4)}(0)\Tr(\bSigma^2)}{4d^4}\right)   \left(\bX\bX^\top \right)^{\odot 2} \notag\\
&+ \left[f\left(\frac{\Tr\bSigma}{d} \right)-f(0)-f'(0)\frac{\Tr\bSigma}{d}-\frac{f''(0)}{2}    \left(\frac{\Tr\bSigma}{d}\right)^2\right] \bI_n.     \label{eq:K2}
\end{align}
For ease of notation, we write \eqref{eq:K2} as
\begin{align}\label{eq:Kquad}
\bKquad= a_0 \one\one\tran + a_1 \bX\bX\tran + a_2 (\bX\bX\tran)^{\odot 2} + a\bI_n,
\end{align}
where
\begin{align}
    a_0 &:= f(0)-\frac{f^{(4)}(0)(\Tr(\bSigma^2))^2}{8d^4}, \label{def:a0}\\  
    a_1 &:= \frac{f'(0)}{d}+\frac{f^{(3)}(0)\Tr(\bSigma^2)}{2d^3}, \label{def:a1}\\
    a_2 &:= \frac{f''(0)}{2d^2}+ \frac{f^{(4)}(0)\Tr(\bSigma^2)}{4d^4}, \label{def:a2}\\
    a &:= f\left(\frac{\Tr\bSigma}{d} \right)-f(0)-f'(0)\frac{\Tr\bSigma}{d}-\frac{f''(0)}{2}    \left(\frac{\Tr\bSigma}{d}\right)^2.\label{def:a}
\end{align}
Here, $a_0,a_1,a_2$ and $a$ are of different orders depending on $d$. These parameters are important to yield a sharp approximation of $\bK$. Notably, these coefficients are different from a direct, entrywise Taylor approximation of $\bK$. In $a_0, a_1$, and $a_2$, the first terms $f(0), \frac{f'(0)}{d}, $ and $\frac{f''(0)}{2d^2}$ are from Taylor expansion of $f$ at $0$, respectively. The additional terms in \eqref{def:a0}-\eqref{def:a2} appear in the proof when we aim to minimize the approximation error under the \textit{spectral norm}. 

Our first result is a non-asymptotic approximation error bound of $\bKquad-\bK$. 
\begin{theorem}[Quadratic kernel approximation]\label{thm:concentration}
Under Assumptions \ref{assump:ratio}-\ref{assump:nonlinear_f}, there exist constants $c,C>0$ depending only on $f, C_1,C_2$, and $C_3$ from the assumptions such that with probability at least $1-cd^{-1/2}$, we have 
    \begin{align}\label{eq:K_matrix}
        \norm{\bK-\bKquad}\leq Cd^{-\frac{1}{12}}.
    \end{align}
\end{theorem}

Theorem \ref{thm:concentration} shows that for sufficiently large $n$, the random kernel matrix $\bK$ can be approximated by a much simpler quadratic kernel matrix $\bK^{(2)}$, which can be decomposed into a low-rank part, a Hadamard product term, and a regularization term. This extends the linear approximation result of \cite{el2010spectrum,couillet2016kernel,bartlett2021deep,sahraee2022kernel,ardakan2022equivalence,couillet2022random}. The polynomial error rate $d^{-\frac{1}{12}}$ might not be optimal (see Figure~\ref{fig:approx}); however, it suffices to have an $o(1)$ error bound for the asymptotic analysis of kernel ridge regression.

\subsection{The limiting eigenvalue distribution for the kernel matrix}
Since the asymptotic structure of $\bK$ can be represented by $\bK^{(2)}$, from standard perturbation analysis in random matrix theory \citep{bai2008large}, we can compute the limiting spectral distribution of $\bK$ by understanding the limiting spectral distribution of the Hadamard product $(\bX\bX\tran)^{\odot 2} $. 

From the tensor representation given in \eqref{eq:tensor_representation}, it suffices to study sample covariance matrices with independent row vectors given by $\bx_i^{\otimes 2}$. 
For any $k,\ell\in [d]$, 
   $(\bx_1^{\otimes 2})_{k\ell}= \bx_1(k) \bx_1(\ell)=(\bx_1^{\otimes 2})_{\ell k}$,
hence there are only $\binom{d+1}{2}$ many distinct coordinates in $\bx_1^{\otimes 2}$. We can define a \textit{reduced tensor product} (introduced by \cite{yaskov2023marchenko}), $\bx_i^{(2)}\in \mathbb R^{\binom{d+1}{2}}$ indexed by $\{(k,\ell): 1\le k\leq \ell\le d \}$ such that 
    \begin{equation}\label{eq:def_reduced}
    \bx_{i}^{(2)} (k,\ell)=
    \begin{cases}
        \sqrt{2}\bx_i(k)\bx_i(\ell) & k<\ell,\\
        \abs{\bx_i(k)}^2 & k=\ell.
    \end{cases}
    \end{equation}
Note that $\bx_i^{(2)}$ is not centered, e.g., if $\bSigma$ is diagonal, then for $k\le \ell\in [d]$, \begin{align}\label{eq:Ex2}
    \E[\bx_{i}^{(2)}(k,\ell)]=\delta_{k,\ell}\bSigma_{kk}.
    \end{align}
With \eqref{eq:def_reduced},  the following identity holds while reducing the dimension of the tensor vectors:
    \begin{align}\label{eq:tensor_product_identity}
        \langle \bx_i^{\otimes 2}, \bx_j^{\otimes 2}\rangle=\langle \bx_i^{(2)}, \bx_j^{(2)}\rangle.
    \end{align}
Let $\bSigma^{(2)}:=\E\left[(\bx_1^{(2)}-\E\bx_1^{(2)})  (\bx_1^{(2)}-\E\bx_1^{(2)}) ^\top \right]\in \mathbb R^{\binom{d+1}{2}\times \binom{d+1}{2}}$. This matrix encodes the covariance information of $\bx_1^{(2)}$. Under the Gaussian moment matching condition for $\bz_1$ in Assumption \ref{assump:data} and an additional assumption that  $\bSigma$ is diagonal, a quick calculation implies 
\begin{align} \label{eq:defSigma2}
\bSigma_{ij, k\ell}^{(2)}=\begin{cases}
    0 & \text{ if } (i,j)\not=(k,\ell),\\
    2\bSigma_{ii}\bSigma_{jj} &\text{ if } i\not=j, (i,j)=(k,\ell), \\
    3\bSigma_{ii}^2 &\text{ if } i=j=k=\ell.
\end{cases}
\end{align} 
 When  $\bSigma=\E\bx_1 \bx_1^\top$ is diagonal with bounded operator norm,  the matrix $\bSigma^{(2)}$ is also diagonal and has a bounded operator norm.   In this section, we need the following additional assumptions for our asymptotic analysis.
        
\begin{assumption}\label{assump:limitratio}
There exists $\alpha>0$ such that
  $\lim_{d\to\infty}\frac{d^2}{2n}=\alpha$.  
\end{assumption}

\begin{assumption}\label{assump:limitsigma}
We assume that $f''(0)\not=0$, $\bSigma$ is a diagonal matrix, and  $\bSigma^{(2)}$ has a limiting spectral distribution denoted by $\mu_{\bSigma^{(2)}}$.
\end{assumption}

Our next theorem characterizes the limiting eigenvalue distribution of $\bK$ after proper centering and scaling. 

\begin{theorem} [Limiting eigenvalue distribution]\label{thm:globallaw}
Under Assumptions \ref{assump:data}-\ref{assump:nonlinear_f} and Assumptions \ref{assump:limitratio}-\ref{assump:limitsigma},
the empirical spectral distribution of  $\frac{4\alpha}{f''(0)}(\bK-a \bI_n)$ converges in probability to a deformed Marchenko-Pastur law $\mu_{\alpha,\bSigma^{(2)}}$ defined as
   \begin{align}\label{eq:defmu}
       \mu_{\alpha,\bSigma^{(2)}}=  \begin{cases}
             (1-\alpha)\delta_0+\alpha\left(\nu_{\alpha}\boxtimes \mu_{\bSigma^{(2)}}\right) & \text{ if }\quad  0<\alpha <1,\\
            \alpha\left(\nu_{\alpha}\boxtimes \mu_{\bSigma^{(2)}}\right) & \text{ if } \quad \alpha\geq 1,
         \end{cases}
     \end{align}   
     where $\boxtimes$ denotes the multiplicative free convolution defined in Definition \ref{def:free_convolution} and $\nu_{\alpha}$ is defined in \eqref{eq:def_nualpha}. The same limit  holds for $\frac{4\alpha}{f''(0)}(\bK^{(2)}-a \bI_n)$.
     In particular, when $\bSigma=\bI_d$,  the empirical spectral distribution of  $\frac{2\alpha}{f''(0)}(\bK-a \bI_n)$ converges in probability to a distribution given by
     $
     \mu=
     \begin{cases}
             (1-\alpha)\delta_0+\alpha \nu_{\alpha} &\text{ if } \quad 0<\alpha <1,\\
             \alpha\nu_{\alpha}& \text{ if } \quad\alpha\geq 1,
         \end{cases}
     $
     where $\nu_\alpha$ is defined by \eqref{eq:def_nualpha}. 
\end{theorem}

See~Figure~\ref{fig:cos} for a simulation of the result in Theorem~\ref{thm:globallaw} when $f=\cos(x)$. We consider both the isotropic case when $\bSigma=\bI_{d}$ and the anisotropic case with \begin{equation}\label{eq:Sigma_0}
    \bSigma=\bSigma_0=\operatorname{diag}(\sigma_1, \dots, \sigma_d), \quad \text{where} \quad
\sigma_i =
\begin{cases}
0.1, & \text{for } i = 1, \dots, 0.2d \\
1.0, & \text{for } i = 0.2d + 1, \dots, 0.6d \\
1.5, & \text{for } i = 0.6d + 1, \dots, d
\end{cases}.
\end{equation} For more simulations, see Section~\ref{sec:simulation}.

\begin{figure}
\centering
\begin{minipage}[t]{0.48\linewidth}
\centering
{\includegraphics[width=1\textwidth]{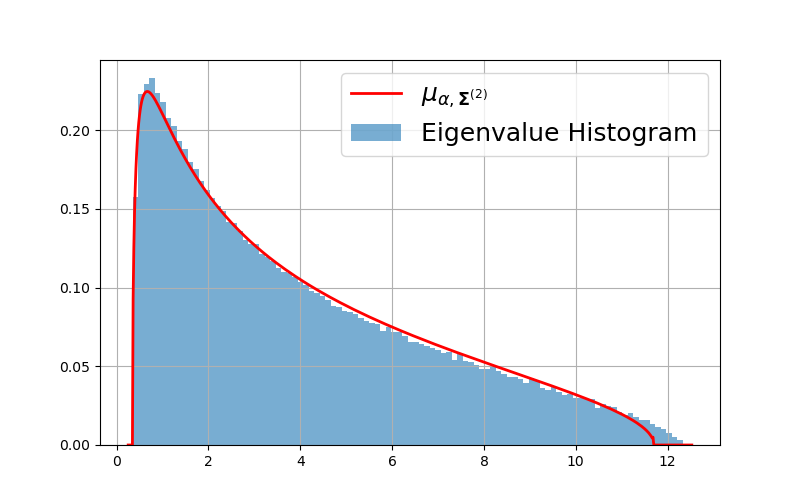}}  \\
\small (a) $\bSigma=\bI_d$. 
\end{minipage}
\begin{minipage}[t]{0.48\linewidth}
\centering
{\includegraphics[width=1\textwidth]{Fig/cosn10000d200_isotropic.png}}  \\
\small (b) $\bSigma=\bSigma_0$. 
\end{minipage}
\caption{Spectral distributions of $\frac{2\alpha}{f''(0)}(\bK-a \bI_n)$ for $f(x)=\cos(x)$, $n=10000$ and $d=200$, and limiting density function of \eqref{eq:defmu} in red curves. For dataset $\bX$, we use Gaussian data with population covariance: $\bSigma=\bI_d$ and $\bSigma= \bSigma_0$ which is defined by \eqref{eq:Sigma_0}.}
    \label{fig:cos}
\end{figure}

\subsection{Training and generalization errors for kernel ridge regression}
Consider a  dataset $\bX=[\bx_1,\ldots\bx_n]^\top$ with $\bx_1,\dots,\bx_n$ satisfying Assumption~\ref{assump:data}. Let 
\begin{equation}\label{eq:training_labels}
    \by=[y_,\ldots,y_n]^\top=[f_*(\bx_1),\ldots,f_*(\bx_n)]^\top+\bepsilon
\end{equation}
be noisy training labels generated by an unknown teacher function $f_*:\R^d\to\R$, and $\bepsilon\in\R^n$ where $\bepsilon_i$ are i.i.d. sub-Gaussian random variable with  \begin{align}\label{eq:def_sigma_eps}
  \E\bepsilon_i=0,  \quad \E\bepsilon_i^2=\sigma^2_{\bepsilon}.   
\end{align} 
With dataset $\bX$ and training labels $\by$, we are interested in the asymptotic behavior of kernel ridge regression (KRR) 
 $$\hat f^{(\K)}_\lambda = \underset{f\in\Hilbert}{\rm argmin}\ \sum_{i=1}^n (y_i - f(\bx_i))^2 + \lambda\norm{f}_\Hilbert^2,$$
for certain Reproducing Kernel Hilbert Spaces $\Hilbert(\Real^d)$, associated with inner product kernels, under the quadratic regime $n\asymp d^2.$ 
Here, $\lambda\ge 0$ is called the ridge parameter in KRR.
The estimator of KRR can be written as 
$$ \hat f^{(\K)}_\lambda(\bx)= K(\bx,\bX)(\bK+\lambda\bI_n)^{-1}\by,$$
where $K(\bx,\bX)=[K(\bx,\bx_1),\ldots,K(\bx,\bx_n)]\in\R^n$ and $\bK$ is defined by \eqref{eq:K_matrix} on dataset $\bX$. In the following sections, we present the asymptotic training and generalization errors of KRR, given some conditions of $f_*$.

\subsubsection{Training  errors}
The prediction of KRR on the training dataset $\bX$ is a $n$-dimensional vector given by
\begin{align}\label{eq:predictor}
\hat f^{(\K)}_\lambda( \bX)=(\hat f^{(\K)}_\lambda(\bx_1),\ldots,\hat f^{(\K)}_\lambda(\bx_n))^\top=\bK(\bK+\lambda\bI_n)^{-1}\by.
\end{align}
Then, we can define the \textit{training error} for this KRR as
\begin{equation}\label{eq:Etrain}
    \cE_{\train}(\lambda):=\frac{1}{n}\|\hat f^{(\K)}_\lambda( \bX)- \by\|_2^2=\frac{\lambda^2}{n}\by^\top(\bK +\lambda\bI_n)^{-2} \by.
\end{equation}
Recall the coefficient $a$ defined in \eqref{def:a}. We need the following additional assumption on the kernel function $f$.
\begin{assumption}\label{assump:analytic}
Assume that $a_0\ge 0, a_1\ge 0$ and $a_2\ge0$ for sufficiently large $d$, where $a_0,a_1,a_2$ are defined in \eqref{def:a0}-\eqref{def:a2}, and $f$ defined by \eqref{eq:kernel_function} satisfies Assumptions~\ref{assump:nonlinear_f} and ~\ref{assump:limitsigma}. We denote that
\begin{align}\label{def:a_star}
    a_*:=\lim_{n\to\infty} a=f(\tau)-f(0)-f'(0)\tau -\frac{1}{2} f''(0)\tau^2.
\end{align} 
\end{assumption}

In this paper, we aim to show that Kernel Ridge Regression (KRR) in the quadratic regime can learn more complex functions compared to the proportional regime \citep{el2010spectrum, bartlett2021deep}. The simplest setting to observe this difference is with a quadratic teacher function. Therefore, we adopt the following assumption for the teacher model, which is similar to the one from \cite{mei2019generalization}.
\begin{assumption}\label{assump:teacher}
Assume that the teacher model $f_*:\R^d\to\R$ is defined by
\begin{equation}\label{eq:teacher} 
    f_*(\bx):= c_0+c_1\langle \bx, \bbeta\rangle +\frac{c_2}{d} \bx^\top\bG\bx.
\end{equation}
where $c_0, c_1,c_2\in \R$ are constants independent of $n,d$, $\bbeta\in \R^d$ is a deterministic vector with $\|\bbeta\|=1$,  and $\bG\in \mathbb R^{d\times d}$ is a symmetric random matrix with independent sub-Gaussian entries of mean zero, variance 1.  
\end{assumption}

The asymptotic training error can be obtained in the next theorem.

\begin{theorem}[Asymptotic training error]\label{thm:train_limit}
Suppose $\lambda+a_*> 0$. Under the assumptions in Theorem~\ref{thm:globallaw} and Assumptions~\ref{assump:analytic} and \ref{assump:teacher}, as $d^2/(2n)\to \alpha\in (0,\infty)$ and $n,d\to \infty$, we have, in probability,
    \begin{align}\label{eq:train_limit}
    \cE_{\textnormal{train}}(\lambda)\to \lambda^2\int \frac{\frac{c_2^2}{\alpha} x+\sigma_{\bepsilon}^2}{\left(\frac{f''(0)}{4\alpha}x+a_*+\lambda\right)^2}~d\mu_{\alpha,\bSigma^{(2)}}(x),
    \end{align}
 where $a_*$ is defined in \eqref{def:a_star},
  $\mu_{\alpha,\bSigma^{(2)}}$ is  defined in \eqref{eq:defmu}, and $\sigma_{\bepsilon}^2$ is defined in \eqref{eq:def_sigma_eps}.
\end{theorem}

Theorem \ref{thm:train_limit} covers the ridge-less case when $\lambda=0$. In the ridge-less case, the training error is $0$, and $\bK$ is invertible since $a_*$ can be seen as an additional ridge regularizer to $\bK^{(2)}$ in \eqref{eq:Kquad}. Note that the limit in \eqref{eq:train_limit} does not depend on the constant and linear terms of $f$ or $f_*$. In the quadratic regime, the kernel $\bK$ can completely fit the linear component of $f_*$ even for $\lambda>0$.

\subsubsection{Generalization errors}

Given a new data point $(\bx,f_*(\bx))$ where $\bx\in\R^d$ is independent with all training data points $\bx_i$,  the \textit{generalization error} of KRR estimator $ \hat f^{(\K)}_\lambda(\bx)$ in \eqref{eq:predictor} can be computed by
\begin{align}
    \cR(\lambda):=\E[(f^{(\K)}_\lambda(\bx)-f_*(\bx))^2|\bX],\label{eq:test_K}
\end{align}
conditioning on the training dataset $\bX$.  We make the following assumption on the distribution of test data $\bx\in\R^d$.  
\begin{assumption}[Test data assumption]\label{assumption: testdata}
    Assume the testing data point satisfies $\bx=\bSigma^{1/2}\bz$,  where $\bz\in\R^d$ is a random vector with independent entries (independent with $\bX$). For  $k\in [d]$, we assume that
$\mathbb E[\bz(k)^{t}]=\mathbb E[g^t], t=1,2,\dots, 18$, where  $g\sim \cN(0,1)$. 
\end{assumption}
Note that $\bx$ does not need to have the same distribution as the training data $\bx_1,\dots,\bx_n$.  
\begin{assumption}\label{assumption:C8}
 Suppose that kernel function $f$ in \eqref{eq:kernel_function} satisfies Assumption~\ref{assump:analytic} and the 9-th derivative satisfies $|f^{(9)}(x)|\leq C$ for all $x\in \mathbb R$. And we further assume that $f'(0)=f^{(3)}(0)=0$ and $f''(0)>0$.
\end{assumption}

Let $\lambda_*>0$ be the unique positive solution to
\begin{align}\label{eq:self_consist}
   \frac{1}{\alpha}- \frac{4 (a_*+\lambda) }{f''(0)\lambda_*} = \int \frac{x}{x+\lambda_*}d\mu_{\bSigma^{(2)}}(x),
\end{align}
where $\alpha, \mu_{\bSigma^{(2)}}$, and $a_*$ are defined in Assumptions~\ref{assump:ratio},~\ref{assump:limitsigma}, and~\ref{assump:analytic}, respectively. 
Then, given $\lambda_*>0$, we can define 
\begin{align}
    \cV(\lambda_*) &:= \frac{\alpha\int_{\R}\frac{x^2}{(x+\lambda_*)^2} d\mu_{\bSigma^{(2)}}( x)}{1-\alpha\int_{\R}\frac{x^2}{(x+\lambda_*)^2}d\mu_{\bSigma^{(2)}}( x)},\label{eq:limit_var}\\
    \mathcal{B}(\lambda_*) &:= \frac{ \lambda_*^2\int_{\R}\frac{x }{(x+\lambda_*)^2}d\mu_{\bSigma^{(2)}}(x)}{1-\alpha\int_{\R}\frac{x^2}{(x+\lambda_*)^2}d\mu_{\bSigma^{(2)}}( x)}.\label{eq:limit_bias}
\end{align} 

\begin{theorem}[Asymptotic generalization error for random $f_*$]\label{thm:test_limit}
Suppose  in \eqref{eq:training_labels}, $f_*$ is a pure quadratic function given by $f_*(\bx)=\bx^\top\bG\bx/d$, where $\bG\in\R^{d\times d} $ is a symmetric random matrix with independent entries satisfying $\E[\bG_{i,j}]=0, \E[\bG_{i,j}^2]=1$
for all $i,j\in[n]$. 
Then, under  the assumptions in Theorem~\ref{thm:globallaw}, Assumptions~\ref{assump:analytic},~\ref{assumption: testdata} and~\ref{assumption:C8}, as $d^2/(2n)\to \alpha\in (0,\infty)$ and $n,d\to \infty$, the generalization error of KRR satisfies 
\begin{equation}
    \cR(\lambda)-\sigma_{\bepsilon}^2\cV(\lambda_*)- \mathcal{B}(\lambda_*)\to 0
\end{equation}
in probability, for any $\lambda\ge 0$, where $\cV(\lambda_*)$ and $\cB(\lambda_*)$ are defined by~\eqref{eq:limit_var} and~\eqref{eq:limit_bias}.
\end{theorem}
Both Theorem \ref{thm:train_limit} and Theorem \ref{thm:test_limit} apply to the case when $f_*(\bx)=\bx^\top\bG\bx/d$ and $\bG$ is a symmetric random matrix with independent sub-Gaussian entries of mean zero, variance 1.

\begin{figure}
    \centering
    \begin{minipage}[t]{0.49\linewidth}
\centering
{\includegraphics[width=1\linewidth]{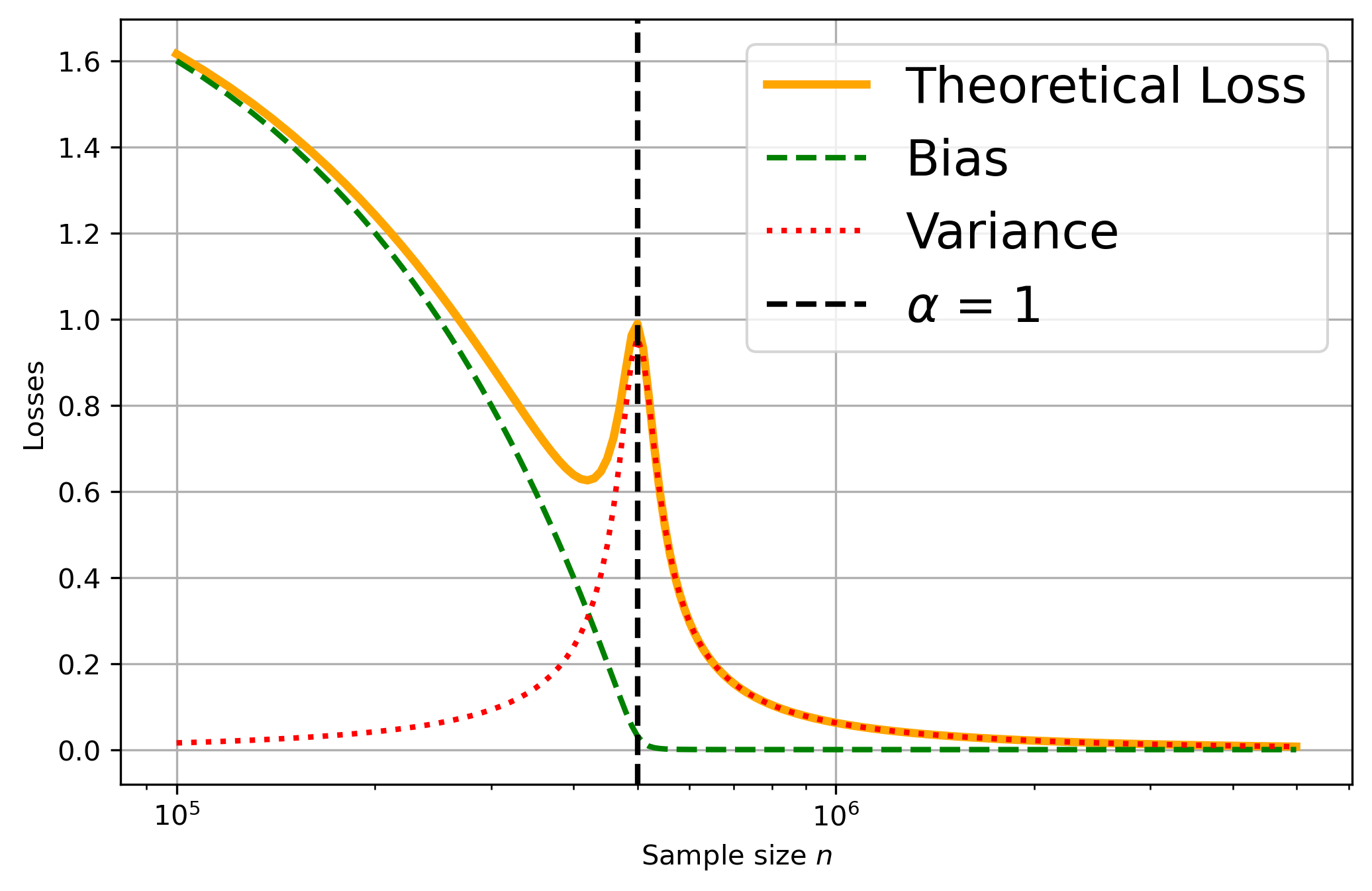}}\\
\small (a) $\lambda=0.001$.
\end{minipage}
\begin{minipage}[t]{0.49\linewidth}
\centering
{\includegraphics[width=1\linewidth]{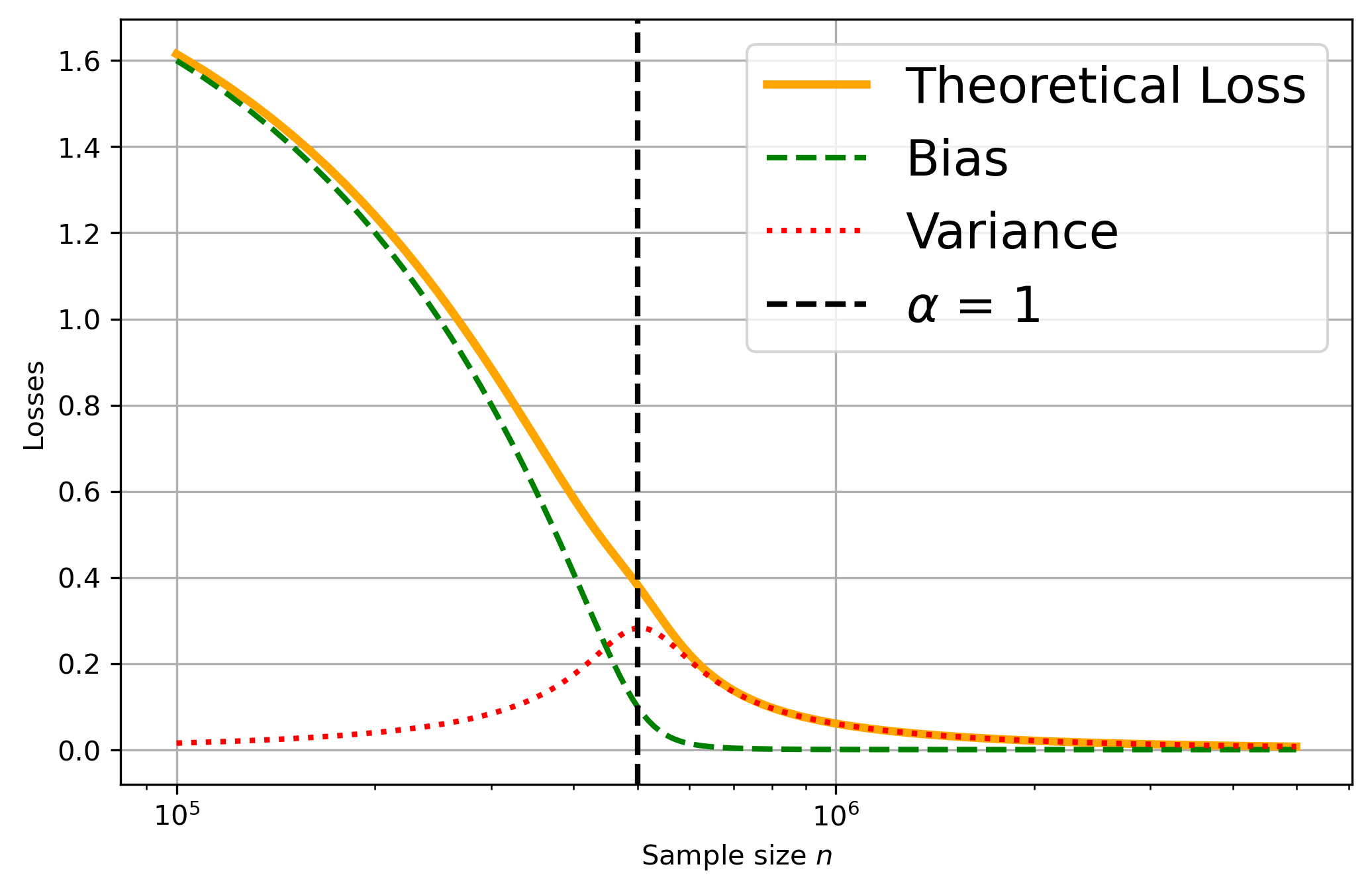}}\\
\small (b) $\lambda=0.01$.
\end{minipage}
    \caption{Theoretical curves of bias term $\cB(\lambda_*)$ (green), variance term $\sigma_{\bepsilon}^2 \cV(\lambda_*)$ (red), and the generalization error (yellow) from Theorem~\ref{thm:test_limit}. We fix $d=1000$ and vary the sample size $n$. The ridge parameter $\lambda = 10^{-3},10^{-2}$ and noise level $\sigma_{\bepsilon} = 0.25$. The plot reveals a double-descent phenomenon in the quadratic regime $n \propto d^2$.}
    \label{fig:theory_curve}
\end{figure}

In Figure~\ref{fig:theory_curve}, we plot the limiting bias, variance, and generalization error curves in Theorem~\ref{thm:test_limit} for different aspect ratios when $f''(0)=2$ and $a_*=0$ with two distinct values of $\lambda$. This figure shows that the bias decreases monotonically, while the variance first increases and then decreases. Their combined effect produces a double-descent curve for the generalization error under the quadratic regime when $\lambda$ is small.

\begin{remark}[Connection to double descent and multiple descent]
 The double descent phenomenon concerns the behavior of generalization error in the proportional regime $n \propto d$ \citep{bartlett2021deep}. More recently, the multiple descent phenomenon has been observed: when $n \propto d^{\ell}$, the generalization error for kernel ridge regression (KRR) decreases as $\ell$ increases \citep{xiao2022precise}. Our work is related to these phenomena in the following way: we show that in the regime $n \propto d^2$, the generalization error is smaller than in the $n \propto d$ case. Moreover, when $\tfrac{2n}{d^2} \to \alpha^{-1}$, the generalization error as a function of $\alpha$ exhibits a double descent curve (see Figure~\ref{fig:theory_curve}).
\end{remark}

In the setting of Theorem~\ref{thm:test_limit}, the limiting bias and variance terms of KRR are \eqref{eq:limit_var} and~\eqref{eq:limit_bias}, respectively. In the regime $n\asymp d$, similar characterizations are also presented by \cite{hastie2022surprises,bartlett2021deep}.  In the quadratic regime $n\asymp d^2$, our asymptotic formula matches the proportional regime by changing $\bSigma$ to $\bSigma^{(2)}$. More intuitively, we showed that KRR in the quadratic regime is asymptotically equivalent to linear ridge regression with reduced tensor product features $x_i^{(2)}\in \mathbb R^{\binom{d+1}{2}}$ defined in \eqref{eq:def_reduced}. 

We expect the same asymptotic generalization error formula to hold also for the general quadratic target in \eqref{eq:teacher} beyond the purely quadratic target case (see Figure~\ref{fig:feature_maps_vs_kernel}). However, it is technically challenging to prove that the effect of the linear component $c_0+c_1\langle \bx,\bbeta\rangle$ is negligible for the generalization error in the quadratic scaling limit. We leave it as an open question for future work.

\begin{remark}
Although \cite{mei2021generalization,misiakiewicz2024non,gavrilopoulos2024geometrical} cover the quadratic regime, our data assumptions are more universal. \cite{misiakiewicz2024non} presented a non-asymptotic deterministic equivalence of general KRR similar to \eqref{eq:limit_var} and~\eqref{eq:limit_bias}, but it requires a certain concentration of eigenfunctions in the kernel's eigendecomposition, which is challenging to verify in our context, especially for anisotropic data. \cite{gavrilopoulos2024geometrical} aligns more closely with our setting but necessitates sub-Gaussian $\bx_i$, but only offers an upper bound for prediction risk.
\end{remark}

When the teacher model $f_*$ is not a random function but a deterministic quadratic function depending on the covariance matrix $\bSigma$ of $\bx$, the bias term in the generalization error vanishes, as stated in the following theorem.  This setting is different from Theorem~\ref{thm:test_limit}, since the generalization error is not taken over the randomness of the teacher model $f_*$.

\begin{theorem}[Asymptotic generalization error for deterministic $f_*$]\label{thm:test_limit_deterministic}
Suppose that teacher function in \eqref{eq:training_labels} is $f_*(\bx)=\bx^\top\bSigma\bx/d$. Then, under  the assumptions in Theorem~\ref{thm:globallaw}, Assumptions~\ref{assump:analytic},~\ref{assumption: testdata} and~\ref{assumption:C8}, as $d^2/(2n)\to \alpha\in (0,\infty)$ and $n,d\to \infty$, the generalization error of KRR satisfies 
 $$\cR(\lambda)-\sigma_{\bepsilon}^2\cV(\lambda_*)\to 0$$
in probability, for any $\lambda\ge 0$, where $\cV(\lambda_*)$ is defined by \eqref{eq:self_consist} and \eqref{eq:limit_var}.
\end{theorem}

\begin{remark}
Compared to the result in the proposal regime of \citep{bartlett2021deep}, Theorem~\ref{thm:test_limit_deterministic} demonstrates the advantage of KRR in a quadratic regime. When the teacher model $f_*$ is a quadratic function perfectly aligned with the covariance matrix $\mathbf{\Sigma}$ of $\bx$, the bias term in the generalization error vanishes. Our result is consistent with \cite[Theorem 10]{ghorbani2019limitations}, where the authors studied population loss (i.e., first take $n\to\infty$ while keeping the width and $d$ fixed) of random features to learn a deterministic noiseless quadratic function with isotropic Gaussian datasets. 
When the teacher model perfectly aligns with $\bSigma$, our result is applicable for more general data distributions.

\end{remark}

\subsubsection{Generalized cross-validation estimators}
The recent work of \cite{misiakiewicz2024non} established a dimension-free deterministic equivalence of the generalized cross-validation (GCV) estimator and the generalization error, and their approximation is uniform over a range of the ridge parameter $\lambda$. Different from our setting, they assumed abstract conditions on the kernel matrices and feature vectors $\phi(\mathbf x)$, while our assumptions are on the nonlinear function $f$ and data vectors $\mathbf x$.
The GCV estimator \citep{hastie2022surprises,wei2022more} is defined as 
\begin{align}
    \mathrm{GCV_{\lambda}(\bK,\by)}=\frac{n\by^\top (\bK+\lambda \bI)^{-2} \by}{\Tr ((\bK+\lambda \bI)^{-1})^{2}} =\frac{\frac{1}{n}\by^\top (\bK+\lambda \bI)^{-2} \by}{(\frac{1}{n}\Tr (\bK+\lambda \bI)^{-1})^{2}},
\end{align}which does not depend on the test dataset.
With the proof of Theorem~\ref{thm:train_limit} and Theorem~\ref{thm:test_limit}, we are able to establish the following approximation:
\begin{corollary}\label{cor:gcv}
Under the assumptions of Theorems~\ref{thm:train_limit} and~\ref{thm:test_limit}, we can get 
\begin{align}
        \mathrm{GCV_{\lambda}(\bK,\by)}-\cR(\lambda)\to 0,
    \end{align}
in probability, as $n\to\infty,$ where $\cR(\lambda)$ is defined by \eqref{eq:test_K}.
\end{corollary}
Corollary~\ref{cor:gcv} verifies the GCV approximation beyond the linear regime considered by \cite{hastie2022surprises,wei2022more}. To the best of our knowledge, this is the first GCV approximation for KRR with anisotropic data in the quadratic regime.

\section{Numerical simulations}\label{sec:simulation}
In this section, we provide several simulations to illustrate our theoretical results.

\begin{figure}[h!]
\centering
\begin{minipage}[t]{0.48\linewidth}
\centering
{\includegraphics[width=1\textwidth]{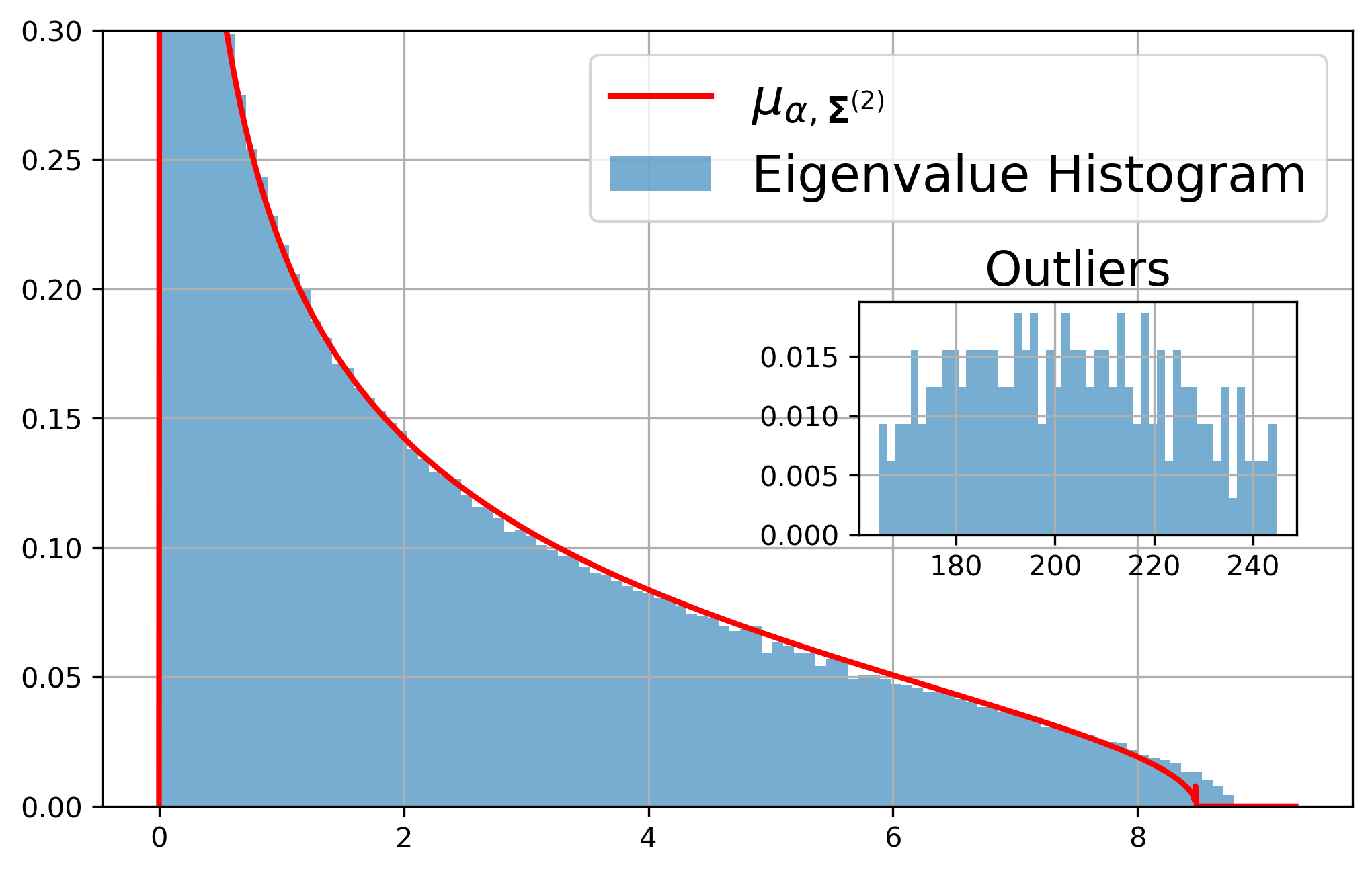}}  \\
\small (a) $\bSigma=\bI_d$, $n=18000$ and $d=200$. 
\end{minipage}
\begin{minipage}[t]{0.48\linewidth}
\centering
{\includegraphics[width=1\textwidth]{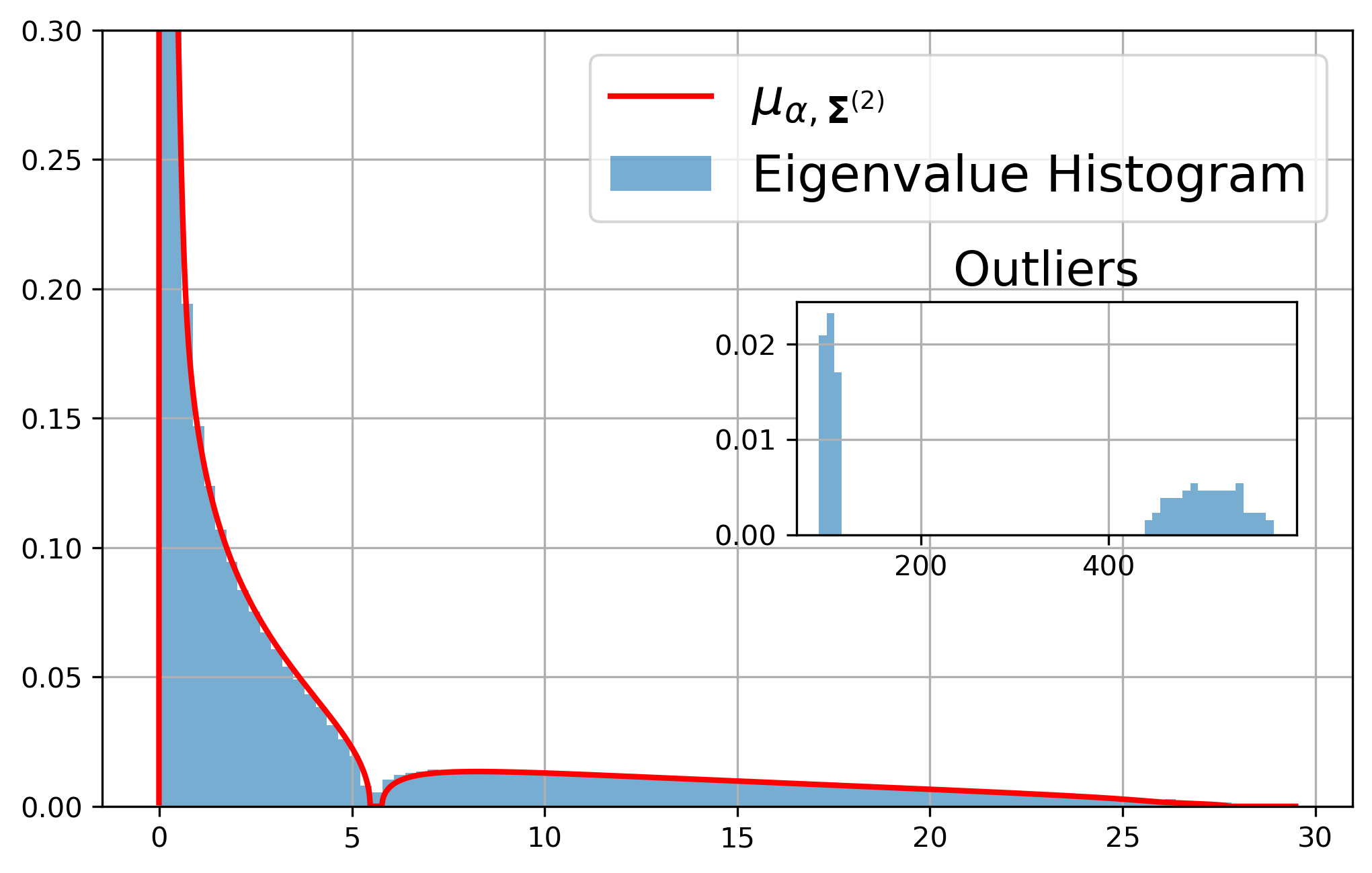}}  \\
\small (b) $\bSigma=\bSigma_0$, $n=18000$ and $d=200$. 
\end{minipage}
\centering
\begin{minipage}[t]{0.48\linewidth}
\centering
{\includegraphics[width=1\textwidth]{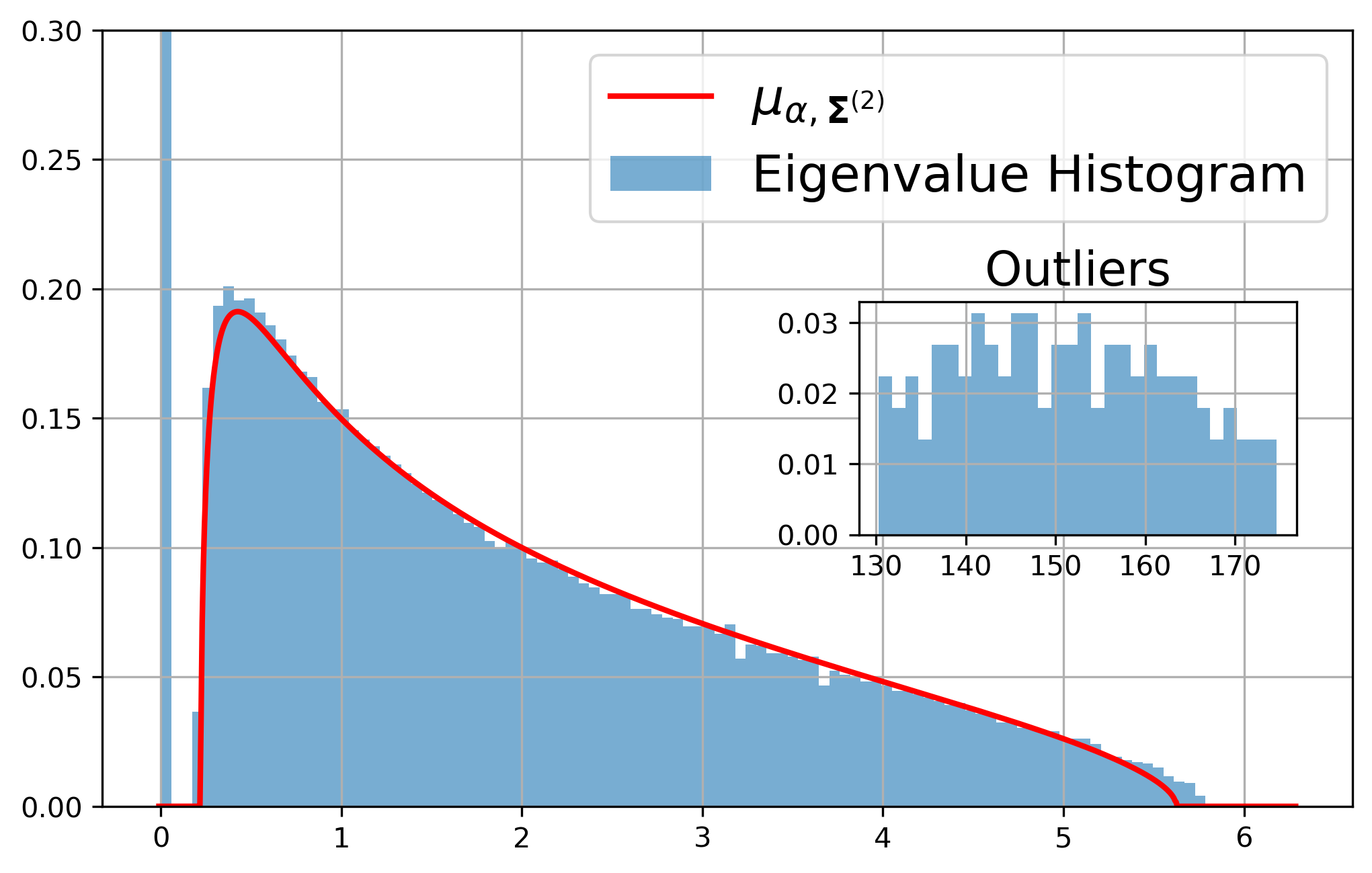}}  \\
\small (c) $\bSigma=\bI_d$, $n=25250$ and $d=150$. 
\end{minipage}
\begin{minipage}[t]{0.48\linewidth}
\centering
{\includegraphics[width=1\textwidth]{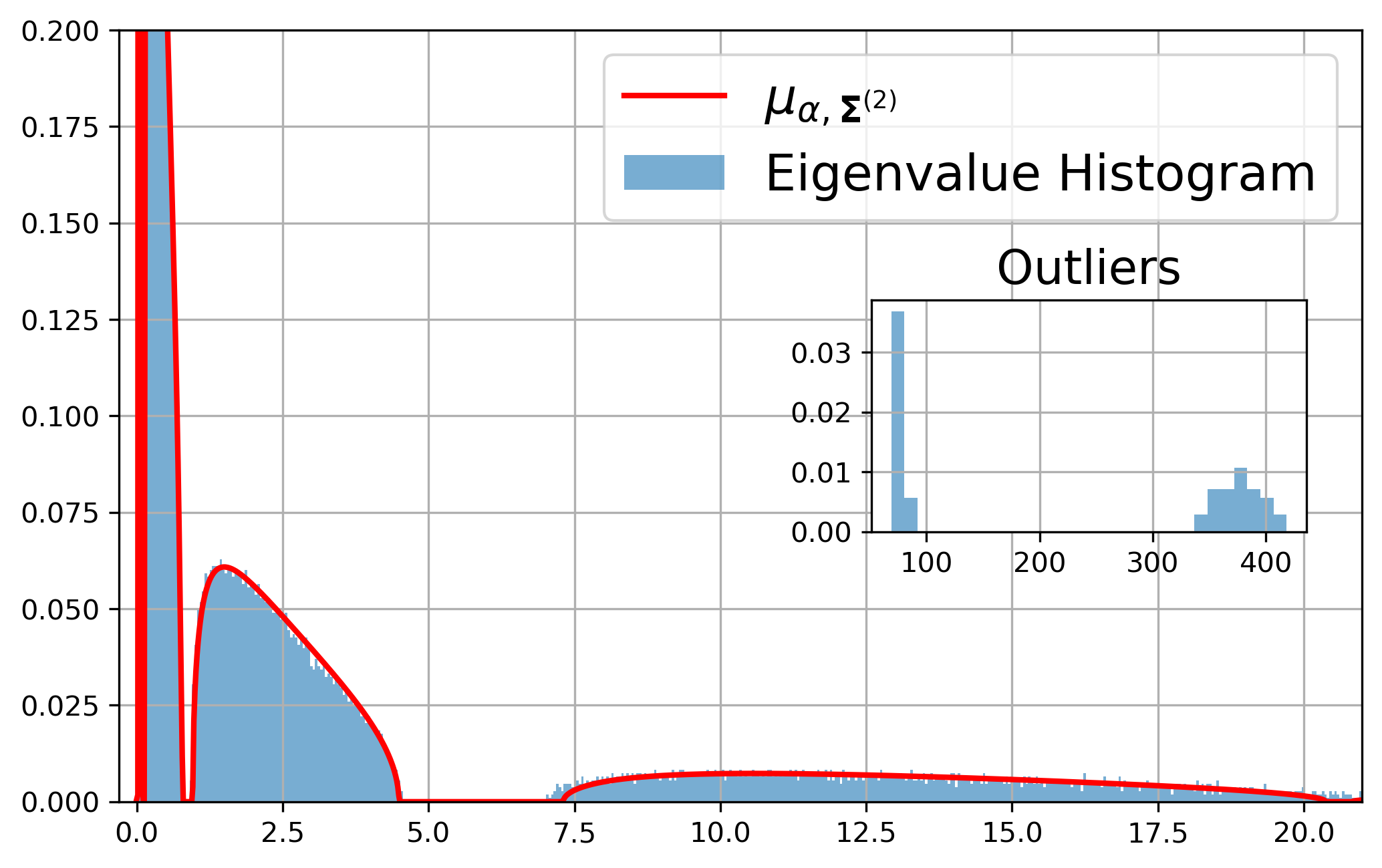}}  \\
\small (d) $\bSigma=\bSigma_0$, $n=25250$ and $d=150$. 
\end{minipage}
    \caption{Spectral distributions for kernel function $f(x)=x^2+x$ with isotropic and anisotropic Gasussian datasets. The red curves are given by the limiting spectral distribution obtained from Theorem~\ref{thm:globallaw}. The number of outliers is $O(d)$ plotted in the subfigures, due to the low-rank terms in $\bK^{(2)}$; see \eqref{eq:K2}.}
    \label{fig:poly}
\end{figure}

\paragraph{Limiting spectral distributions for $\bK$.}
Following Figure~\ref{fig:cos}, we provide additional simulations for the spectral distribution of the kernel matrix in Figure~\ref{fig:poly} for a quadratic kernel function $f(x)=x^2+x$ with isotropic Gaussian dataset and anisotropic Gaussian dataset with population covariance $\bSigma_0$ defined in \eqref{eq:Sigma_0}. For an anisotropic Gaussian, the limiting spectral distribution could have multiple disjoint bulks in Figure~\ref{fig:poly}(d). For these simulations, we also observe $O(d)$ outliers presented in the subfigures. These outliers may come from the terms $\ones\ones^\top$ and $\bX\bX^\top$ in our $\bK^{(2)}$ approximation from Theorem~\ref{thm:concentration}.

\paragraph{Approximation error $\|\bK-\bK^{(2)}\|$.}
In Figure~\ref{fig:approx}, we consider the approximation error under the spectral norm between the kernel random matrix $\bK$ and the quadratic kernel random matrix $\bK^{(2)}$ defined in \eqref{eq:K2} where the kernel function is $f(x)=e^x$. We fix the ratio $\frac{d^2}{2n}=1.2$ and $0.8$, and vary the values of $d$. The simulation suggests the order of the approximation error is between $d^{-1}$ and $d^{-1/2}$.

\begin{figure}[h!]
\centering
\begin{minipage}[t]{0.49\linewidth}
\centering
{\includegraphics[width=1\textwidth]{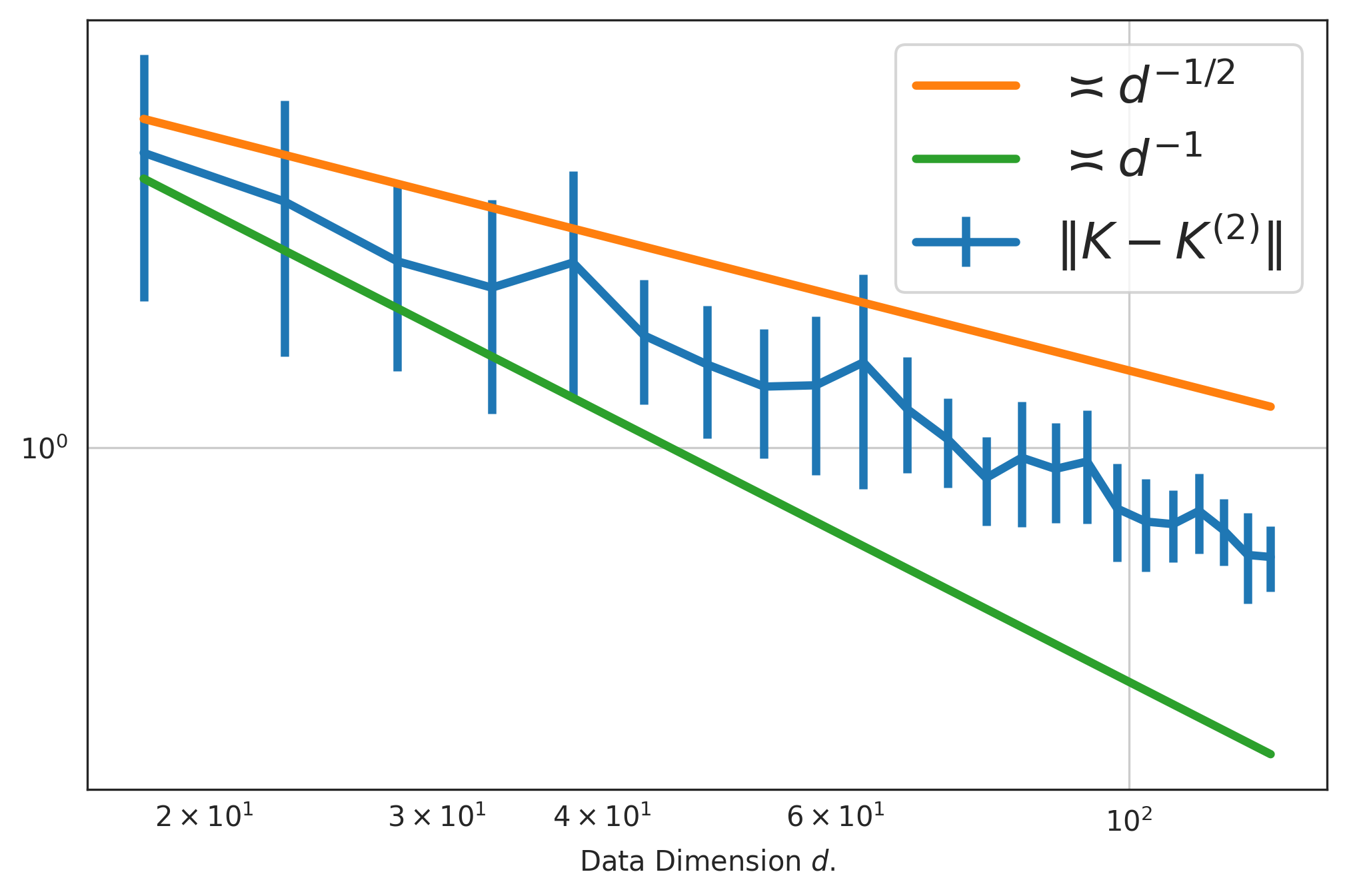}} \\
\small (a) $\alpha = 1.2$.
\end{minipage}
\begin{minipage}[t]{0.49\linewidth}
\centering
{\includegraphics[width=1\textwidth]{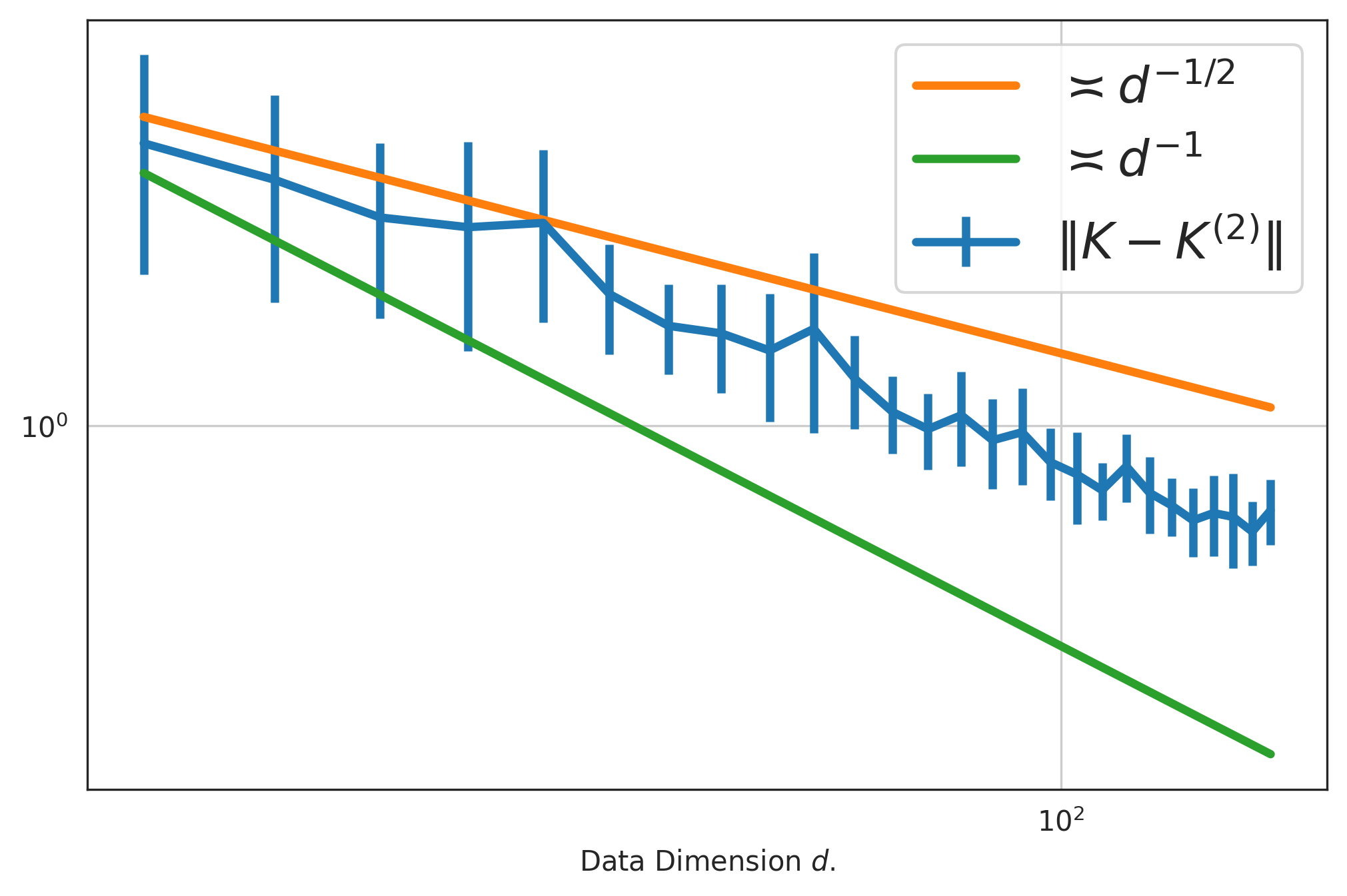}}  
\\
\small (b) $\alpha =0.8$.
\end{minipage}
\caption{Numerical simulations for the operator norm $\|\bK-\bKquad\|$ for exponential kernel $f(x)=\exp(x)$ when varying $d$ and fixing the ratio $\alpha = \frac{d^2}{2n} = 1.2 $ and $0.8$. For each $n$ and $d$, we take 15 trials to average the error.}
    \label{fig:approx}
\end{figure}

\paragraph{Generalization errors for KRR.}
In Figure~\ref{fig:feature_maps_vs_kernel}(a), we present a simulation for the test losses of KRR, as $n$ is increasing, and theoretical prediction from Theorem~\ref{thm:test_limit} when the teacher model $f_*$ is random, defined by \eqref{eq:teacher}. We fix $d = 160$, and use isotropic Gaussian data, polynomial kernel $f(x)=(1+x)^2$, $\lambda = 0.01$, and $\sigma_{\bepsilon} = 0.5$. This simulation also demonstrates the double descent phenomenon.
In Figure~\ref{fig:feature_maps_vs_kernel}(b), we present a simulation to empirically justify Theorem~\ref{thm:test_limit_deterministic} for test losses. The set up is same as Figure~\ref{fig:feature_maps_vs_kernel}(a) but using a deterministic teacher model $f_*(\bx):= 1+2\langle \bx, \bbeta\rangle +\frac{1}{d} \|\bx\|^2$ where $\bbeta$ is a fixed unit norm vector. For both cases, we can observe the peak of the test loss around $\alpha=1.0$.

\begin{figure}[ht!]
\centering
\begin{minipage}[t]{0.49\linewidth}
\centering
{\includegraphics[width=1\textwidth]{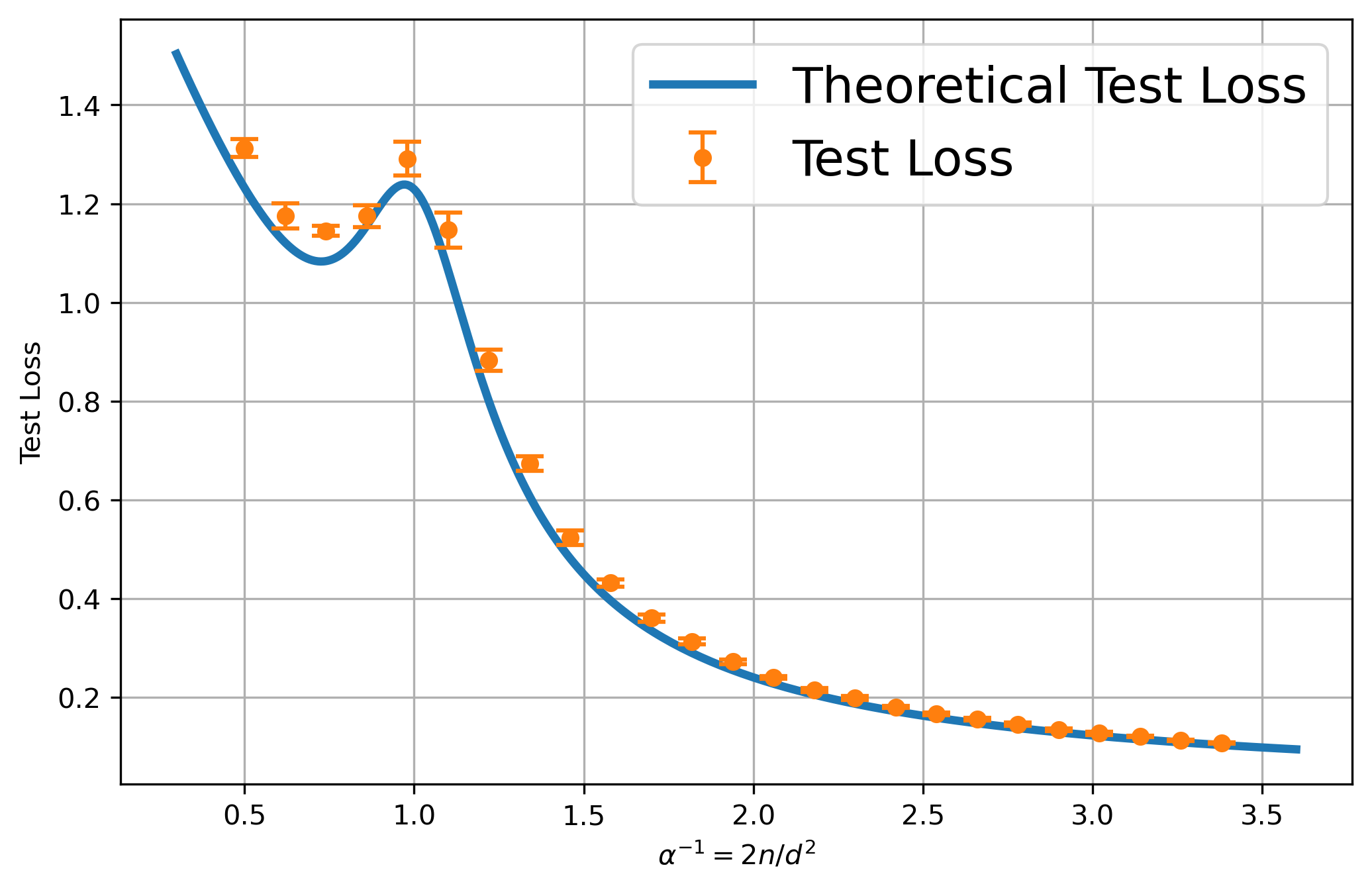}}  \\
\small (a) Random target $f_*$.
\end{minipage}
    \centering
\begin{minipage}[t]{0.49\linewidth}
\centering
{\includegraphics[width=1\textwidth]{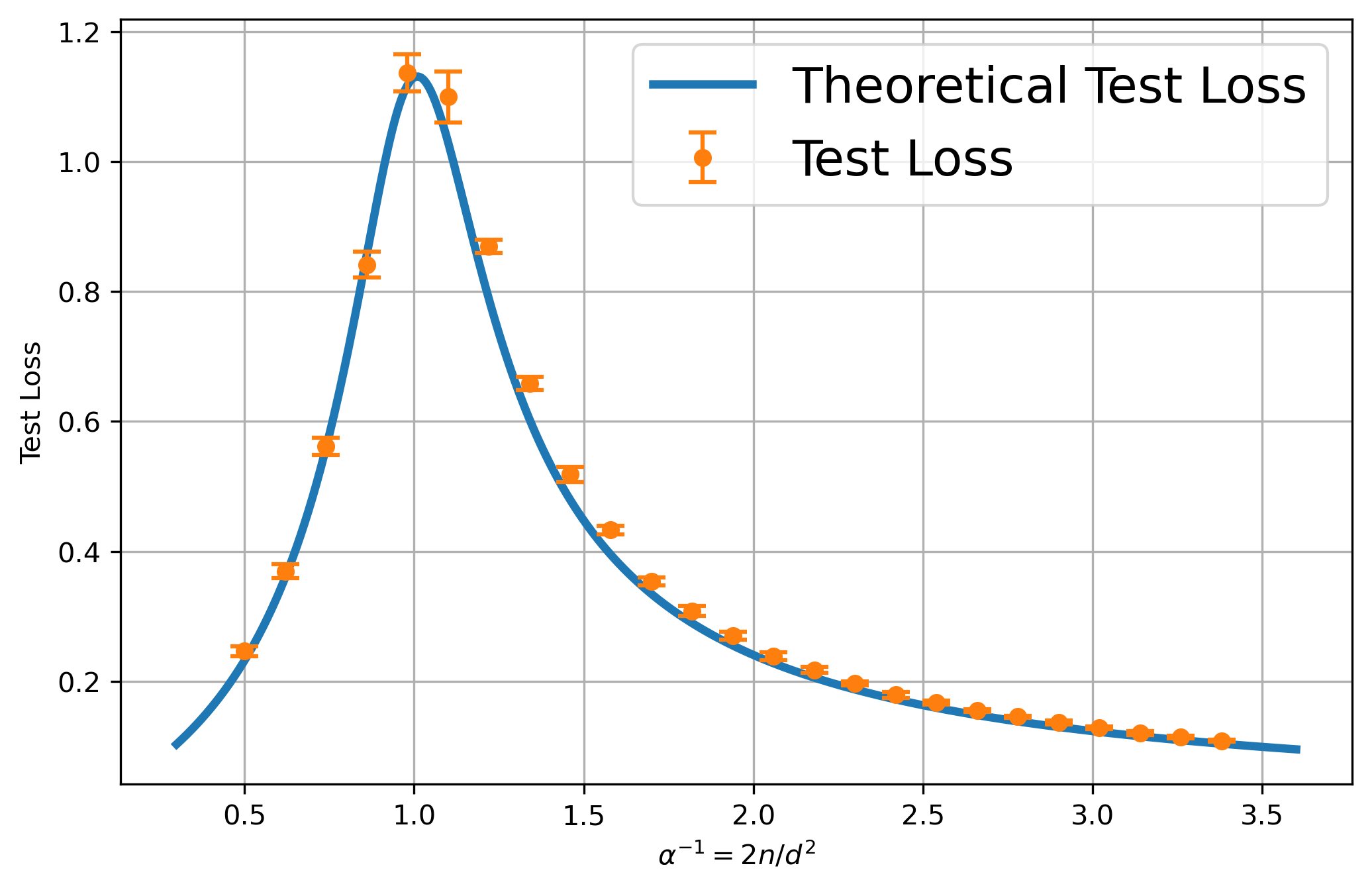}}  \\
\small (b) Deterministic target $f_*$. 
\end{minipage}
    \caption{Test losses (orange points) and theoretical prediction (blue lines) of $\cR(\lambda)$ for different aspect ratios $\alpha$ and teacher models $f_*$. Fix $d = 160$, noise level $\sigma_{\bepsilon} = 0.5$, and ridge parameter $\lambda = 0.01$. We choose the kernel function as $f(x)=(1+x)^2$. For each simulation point, we take 8 averages.  (a) The teacher model $f_*$ is defined by \eqref{eq:teacher} with coefficients $c_0 =1,c_1=2,c_2=1$ and the theoretical curve is given by Theorem~\ref{thm:test_limit}. (b) The teacher model $f_*$ is identical to (a) but replaces $\bG$ in \eqref{eq:teacher} with $\bI_d$ and the theoretical curve is derived from Theorem~\ref{thm:test_limit_deterministic}.}
    \label{fig:feature_maps_vs_kernel}
\end{figure}

\section{Conclusion}
This paper extends the theoretical understanding of kernel methods by analyzing kernel ridge regression in the quadratic regime, where the number of samples scales quadratically with the data dimension. Through a novel quadratic approximation of kernel matrices under general covariance structures, we derive precise asymptotic characterizations for both training and generalization errors. These results highlight that, unlike in the proportional regime, kernel methods in the quadratic regime retain their nonlinear expressive power and can fully capture quadratic target functions. Our analysis relies on new concentration inequalities and moment methods, providing tools that can be extended to more general polynomial regimes. This work bridges a critical gap in understanding the behavior of kernel models in high-dimensional settings beyond linear approximations and isotropic data distribution.

Our method, particularly using Wick’s formula, is not only tailored to our problem but is also broadly generalizable to a wide class of high-dimensional random matrix problems involving non-linear kernels or polynomial functions of Gaussian-like data.
With our proof technique, for general polynomial scaling $n\asymp d^{\ell}$, we expect the $k$-th moments matching condition would grow linearly with $k\asymp l$. This method, along with our trace-based error bounding techniques, can be adapted to analyze other models, including random feature models and learning dynamics of neural networks with polynomial activations. Moreover, our approach sheds light on the structure of random tensor products, which is increasingly relevant in modern high-dimensional learning theory.

Several promising directions remain for future research. One is to extend our results to higher-order polynomial regimes ($n\asymp d^k$ for any $k\in\N$). Another is to relax the Gaussian moment-matching condition to more general sub-Gaussian assumptions. Furthermore, we anticipate applying our theoretical insights to real-world high-dimensional learning tasks and revealing novel practical implications for different scalings of sample size, data dimension, and the size of machine learning models.

\acks{The authors are listed in alphabetical order. P.P. was partially supported by grants from Schmidt Sciences, DST, Amazon, SBI, and, in the initial stage of this work, by the HDSI-Simons postdoctoral fellowship.
Z.W. was partially supported by
NSF DMS-2055340, NSF DMS-2154099, and NSF DMS-1928930, while Z.W. was in residence at the Simons Laufer Mathematical Sciences Institute in Berkeley, California, during the Spring 2025.
Y.Z. was partially supported by NSF-Simons Research
Collaborations on the Mathematical Foundations of Deep Learning, the AMS-Simons Travel Grant, and the Simons Grant MPS-TSM-00013944. Part of the work was done when the three authors visited the Simons Institute for the Theory
of Computing during the Deep Learning Theory program in the Summer of 2022.}

\newpage
\appendix


\section{Additional definitions and lemmas}\label{sec:appendix}

\subsection{Additional definitions}
\begin{definition}[Stieltjes transform]
 Let $\mu$ be a probability measure on $\R$. The Stieltjes transform of $\mu$ is a function $m(z)$ defined on $\mathbb C\setminus \supp{(\mu)}$ by
$
     m(z)=\int_{\R} \frac{1}{x-z} d\mu(x).$
\end{definition}Notice that the Stieltjes transform $m(z)$ uniquely determines this probability measure $\mu$ \citep[Appendix B.2]{bai2010spectral}. For any $n\times n$ Hermitian matrix $\bA_n$, the Stieltjes transform of the empirical spectral distribution of $\bA_n$ can be written as $\tr (\bA_n-z\bI)^{-1}$. We call $(\bA_n-z \bI)^{-1}$ the resolvent of $\bA_n$.

\begin{definition}[Deformed Marchenko-Pastur law]\label{def:free_convolution}
For a probability measure $\nu$, we can define a deformed  Marchenko-Pastur  probability measure  denoted by $\mu_{\alpha}^{\mathrm{MP}}\boxtimes\nu$ via its Stieltjes transform $m(z)$, for any $z\in \mathbb C^+\cup \R_{-}$. Then $m(z)$ is recursively defined by 
\[m(z)=\int \frac{1}{x(1-\alpha-\alpha\cdot zm(z))-z}d\nu(x).\]
This is also called the Marchenko-Pastur equation with aspect ratio $\alpha\in(0,\infty)$, see also results by \cite{marchenko1967distribution,bai2010spectral,yao2015sample}.
Additionally, let us define the companion Stieltjes transform $\widetilde m(z):=\alpha m(z)+(1-\alpha)(-1/z).$ Then, we have a fixed point equation of $\widetilde m(z)$, for any $z\in \mathbb C^+\cup \R_{-}$,
\begin{equation}\label{eq:companion}
    z = -\frac{1}{\widetilde m(z)}+\alpha \int \frac{x}{1+x\widetilde m(z)}d\nu(x).
\end{equation}

\end{definition}
For a full description of free independence and free multiplicative convolution, see \citep[Lecture 18]{nica2006lectures} and \citep[Section 5.3.3]{anderson2010introduction}. The free multiplicative convolution $\boxtimes$ was first introduced by \cite{voiculescu1987multiplication}, which later has many applications for products of asymptotic free random matrices.

An example of this deformed Marchenko-Pastur law can be obtained by the following matrix model \citep{marchenko1967distribution}. Let $\bX\in \mathbb R^{d\times n}$ with aspect ratio $n/d\to \alpha$, where each entry in $\bX$ is i.i.d. $\mathcal N(0,1/d)$. Let $\bSigma\in \R^{n\times n}$ be a deterministic PSD matrix with limiting spectral distribution $\nu$. Then the limiting spectral distribution of $\bSigma^{1/2}\bX^\top \bX \bSigma^{1/2}$ is $\mu_{\alpha}^{\mathrm{MP}}\boxtimes\nu$.

\begin{definition}[Hermite polynomials]\label{def:hermitepolynomial}
The normalized $r$-th normalized Hermite polynomial is given by 
\begin{align} \label{eq:hermitepolynomial}
h_r(x)=\frac{1}{\sqrt {r!}} (-1)^r e^{x^2/2} \frac{d^r}{dx^r} e^{-x^2/2}.
\end{align}
Here $\{h_r\}_{r=0}^{\infty}$ form an orthonormal basis of $L^2(\mathbb R, \Gamma)$, where $\Gamma$ denotes the standard Gaussian distribution. For $\sigma_1,\sigma_2\in L^2(\mathbb R, \Gamma)$,  the inner product is defined by 
$$\langle \sigma_1,\sigma_2\rangle =\int_{-\infty}^{\infty} \sigma_1(x)\sigma_2(x) \frac{e^{-x^2/2}}{\sqrt{2\pi}}dx.$$  Every function $\sigma \in L^2(\mathbb R, \Gamma)$ can be expanded as a Hermite polynomial expansion    $\sigma(x)=\sum_{r=0}^{\infty}\zeta_r(\sigma)h_r(x)$,
where $\zeta_r(\sigma)$ is the $r$-th Hermite coefficient defined by 
$$\zeta_r(\sigma):=\int_{-\infty}^{\infty} \sigma(x)h_r(x)\frac{e^{-x^2/2}}{\sqrt{2\pi}}dx.
$$
\end{definition}

\subsection{Auxiliary lemmas}

\begin{lemma}[Lemma D.2 in \citep{nguyen2020global}] \label{lem:NM20D2} Let $\x, \y\in \R^d$ such that $\|\x\|=\| \y\|=1$ and $\w\sim \mathcal N(0, I_{d})$. Let $h_j$ be the $j$-th normalized Hermite polynomial in \eqref{eq:hermitepolynomial}. Then 
$
    \E_{\w}[h_j(\langle \w,\x\rangle)h_k(\langle \w,\y\rangle)  ]= \delta_{jk} \langle \x, \y \rangle ^k$.
\end{lemma}

\begin{lemma}[Theorem A.45 in \citep{bai2010spectral}]\label{lem:BS10A45}
Let $\bA, \bB$ be two $n\times n$ Hermitian matrices. If $\|\bA-\bB\|\to 0$ as $n\to\infty$, then $\bA$ and $\bB$ have the same limiting spectral distribution.
\end{lemma}

\begin{lemma}[Theorem  A.43 in \citep{bai2010spectral}]\label{lem:lowrank}
   Let $\bA, \bB$ be two $n\times n$ Hermitian matrices. If $\frac{1}{n}\rank(\bA-\bB)\to 0$ as $n\to\infty$, then $\bA$ and $\bB$ have the same limiting spectral distribution.
\end{lemma}

\begin{lemma}[Wick's formula for Gaussian vectors]\label{lemm:wick}
Assume that $\bx=\bSigma^{1/2}\bz$, where $\E[\bz]=0$, $\E[\bz\bz^\top]=\bI_{d}$,  and $\bz$ matches the first $(a+b)$-th joint moments with the standard Gaussian vector $\bg\sim \cN(0,\bI_d)$, for some $a,b\in\N$ and $\w=\bSigma^{1/2}\bg$. Then, for any two deterministic vectors $\bu$ and $\bv$, 
\begin{align}
\E_{\bx}[\langle \bx,\bu\rangle^a  \langle \bx,\bv\rangle^b]&=\E_{\bw}[\langle \bw,\bu\rangle^a  \langle \bw,\bv\rangle^b]\\
&=\sum_{\pi\in\cP_2(a+b)}\prod_{\substack{(l,j)\in\pi\\ l,j\in [a]}}\bu^\top\bSigma\bu \prod_{\substack{(l,j)\in\pi\\ a+1\leq l,j\leq a+b }}\bv^\top\bSigma\bv \prod_{\substack{(l,j)\in\pi\\ l\in [a], a+1\leq j\leq a+b}}\bu^\top\bSigma\bv,
\end{align}
where $\cP_2(a+b)$ is collection of all pairwise matchings on  $[a+b]$, and $(\ell,j)\in \pi$ means the index $\ell$ is matched with $j$.
\end{lemma}

\begin{proofoflemma}{\ref{lemm:wick}}
    The first identity comes from the moment matching condition between $\bg$ and $\bz$, and the second one is from Wick's formula \citep{wick1950evaluation} and the fact that
$\text{Cov}(\langle \bw,\bu\rangle, \langle \bw,\bv\rangle)= \bu^\top \bSigma\bv.$ 
\end{proofoflemma}

\begin{lemma}[Whittle's inequality, Theorem 2 in \citep{whittle1960bounds}]\label{lem:quadratic_moments}
 Let $\bx\in \mathbb R^d$ be a random vector with independent entries and zero mean. Let $\gamma_{j}(s)=\mathbb E[|\bx_j|^s]^{1/s}$. Let $\bA=(a_{jk})_{j,k\in [d]}\in \mathbb R^{d\times d}$ be a deterministic matrix. We have for $s\geq 2$  and  a numerical constant $C(s)$ depending on $s$,
 \begin{align}
     \E| \bx^\top \bA \bx -\E [\bx^\top \bA \bx]|^{s}\leq C(s) \left( \sum_{j,k} a_{jk}^2 \gamma_j^2(2s) \gamma_k^2(2s)\right)^{s/2}.
 \end{align}
\end{lemma}

\begin{lemma}[Theorem 1.1 in \citep{bai2008large}]\label{lem:BaiZhou}
Let $\bx\in \mathbb R^p$ be a random vector and $\bX$ be a $p\times n$ matrix with i.i.d. columns and $\bSigma=\mathbb E [\bx\bx^\top]$ with bounded operator norm, and its limiting ESD is given by $\mu_{\bSigma}$. If $p/n\to\alpha$ and 
 $\mathbb E\left|\bx^\top \bA\bx-\Tr[\bA\bSigma]\right|^2=o(p^2)$
for $\bA\in \mathbb R^{p\times p}$ with $\|\bA\|\leq 1$,
then the empirical spectral distribution of $\frac{1}{n} \bX \bX^\top$ converges in probability to a deformed Marchenko-Pastur law $\mu_{\alpha}^{\mathrm{MP}}\boxtimes \mu_{\bSigma}$, where $\mu_{\alpha}^{\mathrm{MP}}$ is defined in \eqref{eq:def_mualpha}.
\end{lemma}

\begin{lemma}[Lemma 2.2 in  {\citep{magnus1978moments}}]\label{lem:quadratic_moments_eq}
    Let $\bA$ be a $d\times d$ real symmetric matrix, $\bg\sim\mc N(0,\bI)$ be a $d$-dimensional Gaussian vector,  and  $\alpha_s=\mathbb E[(\bg^\top \bA \bg)^s]$. We have 
    \begin{align*}
        \alpha_2 &= (\Tr\bA)^2 +2\Tr(\bA^2),\quad 
        \alpha_3 =(\Tr\bA)^3+6\Tr\bA (\Tr\bA^2)^2+8\Tr\bA^3,\\
        \alpha_4&=(\Tr\bA)^4+32\Tr\bA\Tr\bA^3+12(\Tr\bA^2)^2+12(\Tr\bA)^2(\Tr\bA^2)+48\Tr\bA^4.
    \end{align*}

\end{lemma}

\begin{lemma}\label{lem:quadratic_moments_eq2}
     Let $\bA,\bB$ be two real symmetric $d\times d$ matrices, and $\bg\sim\mc N(0,I_d)$ be a $d$-dimensional Gaussian vector. Then, we have   
$        \mathbb E[(\bg^\top \bA \bg)(\bg^\top \bB \bg)] =\Tr\bA\cdot\Tr\bB+2\Tr(\bA\bB).$
\end{lemma}
\begin{proof}
    \begin{align*}
        \mathbb E[(\bg^\top \bA \bg)(\bg^\top \bB \bg)]&=\sum_{i,j,k,l} \bA_{ij}  \bB_{kl}\E[\bg_i\bg_j\bg_k\bg_l]=\sum_{i,j,k,l} \bA_{ij}  \bB_{kl} \left( \delta_{ij}\delta_{kl}+\delta_{ik}\delta_{jl}+\delta_{il}\delta_{jk} \right),\\
        &=\Tr(\bA)\Tr(\bB)+\Tr(\bA\bB^\top)+\Tr (\bA\bB)=\Tr\bA\cdot\Tr\bB+2\Tr(\bA\bB),
    \end{align*}
    where the second identity is due to Wick's formula \citep{wick1950evaluation}.
\end{proof}

\section{Proof of Theorem \ref{thm:concentration}}\label{sec:concentration}
To track the dependence on model parameters, in this section, we use $a_n\lesssim b_n$ to indicate $a_n\leq Cb_n$ for some numerical constant $C$ independent of any other model parameters including $n,d$, $f$ in \eqref{eq:kernel_function}, and we assume $C_1, C_2, C_3>1 $ in Assumptions \ref{assump:ratio}-\ref{assump:sigma} for convenience.

 We first apply the Taylor expansion of  $f$ in Section \ref{sec:taylor}. Since the off-diagonal entries of $\bK$ are concentrated around $0$ and the diagonal entries are concentrated around $\frac{\Tr \bSigma}{d}$, we expand $f$ at $0$ and $\frac{\Tr \bSigma}{d}$ respectively. In Section \ref{sec:off_diagonal_term}, we divide the off-diagonal part of $\bK$ into three matrices and control their spectral norms by the moment method. This is the most technical part of the proof. Section \ref{sec:diag_term} deals with the diagonal terms in $\bK$. Combining the three parts, we finish the proof of Theorem \ref{thm:concentration} in Section \ref{sec:together}.

\subsection{Taylor expansion of the kernel matrix}\label{sec:taylor}

We begin with a Taylor expansion of $\bK$. Since $f$ is $C^5$ around $0$,   through Taylor expansion at $0$, we have for $i\not=j$,
\begin{align}\label{eq:taylor_off}
\bK_{ij}=&f(0)+\frac{f'(0)}{d}\langle \bx_i,\bx_j\rangle+ \frac{f''(0)}{2d^2}\langle \bx_i,\bx_j\rangle^2 +\frac{f^{(3)}(0)}{6d^3}\langle \bx_i,\bx_j\rangle^3 \notag \\
&+\frac{f^{(4)}(0)}{24d^4}\langle \bx_i,\bx_j\rangle^4+\frac{f^{(5)}(\zeta_{ij})}{120d^5}\langle \bx_i,\bx_j\rangle^5,
\end{align}
where $\zeta_{ij}$ is between $0$ and $\frac{1}{d}\langle \bx_i, \bx_j\rangle$.  Similarly, since $f$ is $C^2$ around $\tau$, for sufficiently large $d$, $\frac{\Tr\bSigma}{d}$ is close to $\tau$ by Assumption~\ref{assump:sigma}, and  we can expand $f$ at $\frac{\Tr\bSigma}{d}$ to obtain that
\begin{align}\label{eq:taylor_diag}
    \bK_{ii}=&f\left( \frac{\norm{\bx_i}^2}{d}\right)=f\left(\frac{\Tr\bSigma}{d}\right)+f'\left( \frac{\Tr\bSigma}{d}\right)\left( \frac{\norm{\bx_i}^2}{d}-\frac{\Tr\bSigma}{d}\right)\\
    &+\frac{f''(\xi_{ii})}{2}\left( \frac{\norm{\bx_i}^2}{d}-\frac{\Tr\bSigma}{d}\right)^2.
\end{align}
where $\xi_{ii}$ is between $0$ and $\frac{\norm{\x_i}^2}{d}$. Next, we control the error of this approximation from diagonal and off-diagonal terms in Sections \ref{sec:off_diagonal_term} and \ref{sec:diag_term}, respectively.

\subsection{Controlling the error in the off-diagonal terms}\label{sec:off_diagonal_term}

For $i\not=j\in[n]$, we have from \eqref{eq:taylor_off} and \eqref{eq:K2},
\begin{align}
\bK_{ij}-\bKquad_{ij}=~&\frac{f^{(3)}(0)}{6d^3}\left(\langle \bx_i,\bx_j\rangle^3-3\Tr\bSigma^2\cdot\langle \bx_i,\bx_j\rangle \right)\nonumber\\
~&+\frac{f^{(4)}(0)}{24d^4} \left(\langle \bx_i,\bx_j\rangle^4-6\Tr\bSigma^2\cdot \langle \bx_i,\bx_j\rangle ^2+3(\Tr\bSigma^2)^2\right)\\
~&+\frac{f^{(5)}(\zeta_{ij})}{120d^5}\langle \bx_i,\bx_j\rangle^5
:=~\widetilde\T(i,j)+\widetilde\F(i,j)+\widetilde\V(i,j),\label{eq:K-K_2}
\end{align}
where $\widetilde \bT, \widetilde \bF$, and $\widetilde \V$ are three matrices with $(i,j)$-entry \ignore{consider renaming $\T$ to $\bT$ (third), $\F$ to $\bF$ (fourth), and $\V$ to $\bV$ (fifth)}
\begin{subequations}
\begin{align}
\widetilde \T(i,j)&=\mathbf{1} \{i\not=j\}\frac{f^{(3)}(0)}{6d^3}\left(\langle \bx_i,\bx_j\rangle^3-3\Tr\bSigma^2\langle \bx_i,\bx_j\rangle \right),\label{def:tildeT}\\
\widetilde \F(i,j)&=\mathbf{1} \{i\not=j\}\frac{f^{(4)}(0)}{24d^4}\left(\langle \bx_i,\bx_j\rangle^4-6\Tr\bSigma^2\cdot \langle \bx_i,\bx_j\rangle ^2+3(\Tr\bSigma^2)^2\right), \label{def:tildeF}\\
\widetilde \V(i,j)&=\mathbf{1} \{i\not=j\}\frac{f^{(5)}(\zeta_{ij})}{120d^5}\langle \bx_i,\bx_j\rangle^5,\label{def:tildeV}
\end{align}
\end{subequations}
which correspond to the third, fourth, and higher-order terms in the approximation error.
Here $\widetilde \T$ and $\widetilde\F$ correspond to the third and fourth normalized Hermite polynomial $h_3(x)=x^3-3x$ and $h_4(x)=x^4-6x^2+3$, respectively. See Definition~\ref{def:hermitepolynomial} for more details.

\subsubsection{Third-order approximation}
We bound the spectral norm of $\widetilde\bT$ by applying the trace method. For $i\neq j$, define
\begin{align}
    \T_{ij}&:=\langle \bx_i,\bx_j\rangle^3 -3\Tr\bSigma^2\cdot\langle \bx_i,\bx_j\rangle\label{eq:F_ij}.
\end{align}
We have 
\begin{align}
      \E\|\widetilde\T\|^6&\leq \E\Tr({\widetilde\T}^6)\lesssim \frac{|f^{(3)}(0)|^6}{d^{18}} \sum_{i_1,i_2,i_3,i_4,i_5,i_6\in [n]} \E[\T_{i_1i_2}\T_{i_2i_3}\T_{i_3i_4}\T_{i_4i_5}\T_{i_5i_6}\T_{i_6i_1}].\label{eq:F_3_four}
\end{align}
There are five different cases in terms of the number of distinct indices among $i_1,i_2,i_3,i_4,$ $i_5,i_6\in [n]$ in the summation. In the following, we control each case separately.
\begin{enumerate}[label=\textbf{Case (\roman*).},leftmargin=0ex, itemindent=12ex]
\item \textbf{$i_1,i_2,i_3,i_4,i_5,i_6\in [n]$ are distinct}. Conditioned on $\bx_{i_1},\bx_{i_3}$ and $\bx_{i_5}$, we have 
\begin{align} &\E[\T_{i_1i_2}\T_{i_2i_3}\T_{i_3i_4}\T_{i_4i_5}\T_{i_5i_6}\T_{i_6i_1}| \bx_{i_1},\bx_{i_3},\bx_{i_5}] \notag \\
=~&\E[\T_{i_1i_2}\T_{i_2i_3}|\bx_{i_1},\bx_{i_3}]\E[\T_{i_3i_4}\T_{i_4i_5}|\bx_{i_3},\bx_{i_5}]\E[\T_{i_5i_6}\T_{i_6i_1}|\bx_{i_1},\bx_{i_5}].\label{eq:third_moment}
\end{align}
We calculate the two conditional expectations separately.

To evaluate \eqref{eq:third_moment}, we notice that each conditional expectation is a degree-3 polynomial of random vector inner products. By our moment matching Assumption \ref{assump:data}, we can easily calculate them due to Wick's formula in Lemma \ref{lemm:wick}. Denote by $\bw_i:=\bSigma^{1/2}\bx_i=\bSigma \bz_i, i\in [n]$. 
With Lemma~\ref{lemm:wick}, since $\bz_i$ has the first 8 moments  matched with the Gaussian distribution, we can compute the following expectations explicitly, where $\bx$ is an i.i.d. sample independent of $\bx_i,\bx_k$ for any $i,k\in[n]$: 
\begin{align}
    \E_{\bx}[\langle \bx,\bx_i\rangle \langle \bx,\bx_k\rangle]=&~\bx_k^\top\bSigma\bx_i= \langle \bw_i,\bw_k\rangle \label{eq:moment11}\\
      \E_{\bx}[\langle \bx,\bx_i\rangle^2 \langle \bx,\bx_k\rangle^2]=&~2\bx_k^\top\bSigma\bx_i\cdot \bx_k^\top\bSigma\bx_i+ \bx_k^\top\bSigma\bx_k\cdot \bx_i^\top\bSigma\bx_i \notag \\
      =&~ 2\langle \bw_i,\bw_k\rangle^2+\norm{\bw_i}^2\norm{\bw_k}^2 \label{eq:moment22}\\
    \E_{\bx}[\langle \bx,\bx_i\rangle^3  \langle \bx,\bx_k\rangle]=&~3\bx_k^\top\bSigma\bx_i\cdot\bx_i^\top\bSigma\bx_i=3\langle \bw_i,\bw_k\rangle \norm{\bw_i}^2 \label{eq:moment31}\\
    \E_{\bx}[\langle \bx,\bx_i\rangle^3  \langle \bx,\bx_k\rangle^3]=& ~9\bx_k^\top\bSigma\bx_i\cdot\bx_i^\top\bSigma\bx_i\cdot\bx_k^\top\bSigma\bx_k+6\left(\bx_k^\top\bSigma\bx_i\right)^3 \notag \\
    =& ~9\langle \bw_i,\bw_k\rangle\norm{\bw_i}^2\norm{\bw_k}^2+6\langle \bw_i,\bw_k\rangle^3.  \label{eq:moment33}\\
      \E_{\bx}[\langle \bx,\bx_i\rangle^4  \langle \bx,\bx_k\rangle^4]= & ~72\langle \bw_i,\bw_k\rangle^2\norm{\bw_i}^2\norm{\bw_k}^2+24\langle \bw_i,\bw_k\rangle^4+9\norm{\bw_i}^4\norm{\bw_k}^4 \label{eq:moment44}\\
       \E_{\bx}[\langle \bx,\bx_i\rangle^4  \langle \bx,\bx_k\rangle^2]=&~12\langle \bw_i,\bw_k\rangle^2\norm{\bw_i}^2+3\norm{\bw_i}^4\norm{\bw_k}^2. \label{eq:moment42}
\end{align}

With Assumptions~\ref{assump:data} and \ref{assump:sigma},  we can also obtain for $i\not=k$, any integer $1\leq s\leq 45$, 
\begin{align}
    \E\left[\langle \bw_i,\bw_k\rangle^{2s}\right]=\E[(\bz_i\bSigma^2 \bz_k)^{2s}] \lesssim &  C_2^{2s}C_3^{4s}d^s. \label{eq:wiwk}
\end{align}
Similarly, we have for $1\leq s\leq 45$,
\begin{align}\label{eq:wi}
  \E\left[\norm{\bw_i}^{2s}\right] =\E [\| \bSigma \bz_i\|^{2s}]\lesssim  C_2^{4s} C_3^{2s}d^{s}.
\end{align}
From Whittle's inequality \citep{whittle1960bounds} in  Lemma \ref{lem:quadratic_moments}, with Assumptions \ref{assump:data} and  \ref{assump:sigma}, we have for any integer $1\leq s\leq  45$,
\begin{align}
     \E\left[\left(\norm{\bw_i}^2-\Tr\bSigma^2\right)^{2s} \right]=\E (\bz_i^\top\bSigma^2 \bz_i-\Tr\bSigma^2)^{2s}\lesssim  \|\bSigma^2\|_{\msf F}^{2s} C_2^{4s}  \lesssim C_3^{2s} C_2^{4s}d^s,\label{eq:centered6}
\end{align}
where we use the inequality $\|\bSigma^2\|_{\msf F}\leq \sqrt{d} \|\bSigma^2\|\leq C_3^2 \sqrt{d}$. For convenience, we denote $t:=\Tr\bSigma^2=\E[\norm{\bw_i}^2]$, and from Assumption \ref{assump:sigma},
\begin{align}\label{eq:bound_t}
  t\leq C_3^2d.  
\end{align}
To bound \eqref{eq:third_moment}, it suffices to consider $\E [\T_{ij}\T_{jk}| \bx_i,\bx_k]$ for $j\not=i,k$. We have
\begin{align}
    \E [\T_{ij}\T_{jk}| \bx_i,\bx_k]=~& \E[\langle \bx_i,\bx_j\rangle ^3\langle \bx_k,\bx_j\rangle ^3 \mid \bx_i,\bx_k] -3t \E [\langle \bx_i,\bx_j\rangle ^3\langle \bx_k,\bx_j\rangle  \mid \bx_i,\bx_k] \notag \\
    &-3t \E [\langle \bx_i,\bx_j\rangle \langle \bx_k,\bx_j\rangle^3  \mid \bx_i,\bx_k]+9t^2  \E [\langle \bx_i,\bx_j\rangle \langle \bx_k,\bx_j\rangle  \mid \bx_i,\bx_k] \notag\\ 
    =~ &9\langle \bw_i,\bw_k\rangle \norm{\bw_i}^2\norm{\bw_k}^2-9t\cdot\langle \bw_k,\bw_i\rangle\left(\norm{\bw_i}^2+\norm{\bw_k}^2\right) \notag\\
    &+9t^2\langle \bw_k,\bw_i\rangle+6\langle \bw_k,\bw_i\rangle^3\nonumber\\
    =~& 9\langle \bw_i,\bw_k\rangle\left(\norm{\bw_i}^2-t\right)\left(\norm{\bw_k}^2-t\right)+6\langle \bw_k,\bw_i\rangle^3,\label{eq:FijFjk}
\end{align}
where in the second equation, we use the explicit moment calculations from \eqref{eq:moment33}, \eqref{eq:moment31}, and \eqref{eq:moment11}.  We now denote $W_{i,k}: = \E [\T_{ij}\T_{jk}| \bx_i,\bx_k]$ for any $j\not=i, j\not=k$.  Thus, for distinct indices $i_1,\dots, i_6$, we have 
\begin{align}
&\E[\E[\T_{i_1i_2}\T_{i_2i_3}\T_{i_3i_4}\T_{i_4i_5}\T_{i_5i_6}\T_{i_6i_1}| \bx_{i_1},\bx_{i_3},\bx_{i_5}]]=\E\left[W_{i_1,i_3}W_{i_5,i_3}W_{i_1,i_5}\right]  \notag \\
\le~ &\frac{1}{3} \left( \E|W_{i_1,i_3}|^{3}+ \E|W_{i_5,i_3}|^{3}+\E|W_{i_1,i_5}|^{3}\right)=\E|W_{i,k}|^{3} \notag \\
\lesssim ~ &\E\left[\left|\langle \bw_i,\bw_k\rangle^3\left(\norm{\bw_i}^2-t\right)^3\left(\norm{\bw_k}^2-t\right)^3\right|\right]+ \E[\left|\langle \bw_k,\bw_i\rangle\right|^9]  \notag  \\
 \lesssim ~&\E\left[\langle \bw_i,\bw_k\rangle^6\right]^{1/2}\E\left[\left(\norm{\bw_i}^2-t\right)^6\right]+ \E[|\langle \bw_k,\bw_i\rangle|^9]
   \lesssim ~ C_3^{18}C_2^{18}d^{4.5}, \label{eq:general_trace}
\end{align}
 In the second inequality, we use \eqref{eq:FijFjk}, and the third inequality is due to Hölder's inequality. In the last inequality,  we apply the estimates in \eqref{eq:wiwk} and \eqref{eq:centered6}.  This concludes that
 \begin{align}\label{eq:distinct_T}
 \frac{1}{d^{18}} \sum_{i_1,\dots,  i_6 \text{ distinct }} \E[\T_{i_1i_2}\T_{i_2i_3}\T_{i_3i_4}\T_{i_4i_5}\T_{i_5i_6}\T_{i_6i_1}]\lesssim \frac{n^6}{d^{18}} C_3^{18} C_2^{18}d^{4.5}\lesssim C_1^{6}C_2^{18} C_{3}^{18}d^{-1.5},
 \end{align}
where we use the assumption that $n\leq C_1d^2$ in Assumption~\ref{assump:ratio}.
 
\item  \textbf{Terms involving five different indices.}  By symmetry of the indices in sum, it suffices to consider the case where $i_1=i_3$ and $(i_1,i_2,i_4,i_5,i_6)$ are all distinct.  Then analogous to \eqref{eq:general_trace}, we have
\begin{align} \E[\T_{i_1i_2}\T_{i_2i_3}\T_{i_3i_4}\T_{i_4i_5}\T_{i_5i_6}\T_{i_6i_1}]
    &= \E[\E[\T_{i_1i_2}^2 \T_{i_1i_4}\T_{i_4i_5}\T_{i_5i_6}\T_{i_6i_1}| \bx_{i_1},\bx_{i_5}]] \notag \\
    &= \E[W_{i_1,i_1}W_{i_1,i_5}^2]\le \E[W_{i_1,i_1}^2]^{1/2}\E[W_{i_1,i_5}^4]^{1/2}. \label{eq:case2}
\end{align}
where the second line is due to Hölder's inequality.  With \eqref{eq:wi}, \eqref{eq:wiwk}, and \eqref{eq:centered6}, we find
\begin{align}
    \E[W_{i_1,i_1}^2]& =\E\left(  9 \norm{\bw_i}^2 (\norm{\bw_i}^2-t )^2 +6 \norm{\bw_i}^6\right)^2  \notag \\
     &\lesssim (\E\|\bw_i\|^8)^{1/2} (\E[\|\bw_i\|^2-t)^8])^{1/2} +\E \|\bw_i\|^{12} \lesssim C_4d^{6}, \label{eq:wiwi2}
\end{align}
where $C_4$ is a constant depends polynomially on $C_2, C_3$. Throughout the entire proof of Theorem \ref{thm:concentration}, we can take $C_4=(C_2C_3)^{90}$. With \eqref{eq:FijFjk}, we have 
\begin{align}
    \E[W_{i_1,i_5}^4]&=\E\left[ 9\langle \bw_i,\bw_k\rangle\left(\norm{\bw_i}^2-t\right)\left(\norm{\bw_k}^2-t\right)+6\langle \bw_k,\bw_i\rangle^3\right]^4 \notag  \\
    &\lesssim \E \left[\langle \bw_i,\bw_k\rangle^4\left(\norm{\bw_i}^2-t\right)^4\left(\norm{\bw_k}^2-t\right)^4 \right]+ \E\langle \bw_k,\bw_i\rangle ^{12} \notag \\
    &\lesssim \E \left[\langle \bw_i,\bw_k\rangle^8\right]^{1/2} \E \left(\norm{\bw_i}^2-t\right)^4+ C_4d^6\lesssim  C_4 d^6. \label{eq:wi4}
\end{align}
Therefore, \eqref{eq:case2} satisfies $\E[\T_{i_1i_2}\T_{i_2i_3}\T_{i_3i_4}\T_{i_4i_5}\T_{i_5i_6}\T_{i_6i_1}]\lesssim C_4 d^6$. 
We can conclude that 
\begin{align}
 \frac{1}{d^{18}} \sum_{ i_1,\dots, i_6 \text{ 5 distinct indices}} \E[\T_{i_1i_2}\T_{i_2i_3}\T_{i_3i_4}\T_{i_4i_5}\T_{i_5i_6}\T_{i_6i_1}] \leq \frac{n^5}{d^{18}}  C_4d^6 \lesssim C_1^5 C_4 d^{-2}.\label{eq:case22}
\end{align}

\item \textbf{Terms involving four different indices.} By symmetry, there are only three cases we need to consider here: 
\begin{enumerate}[label=(\alph*)]
    \item $i_1=i_3=i_5$ and $(i_1,i_2,i_4,i_6)$ are all distinct.
    \item $i_1=i_3,i_2=i_4$ and $(i_1,i_2,i_5,i_6)$ are all distinct.
    \item $i_1=i_3,i_4=i_6$ and $(i_1,i_2,i_4,i_5)$ are all distinct.
\end{enumerate} 
For (a), we have 
\begin{align}
&\E[\T_{i_1i_2}\T_{i_2i_3}\T_{i_3i_4}\T_{i_4i_5}\T_{i_5i_6}\T_{i_6i_1}]=\E[\T_{i_1i_2}^2\T_{i_1i_4}^2\T_{i_1i_6}^2] \notag \\
    =~ & \E[\E[\T_{i_1i_2}^2 \T_{i_1i_4}^2\ \T_{i_6i_1}^2| \bx_{i_1} ]] 
    =  \E[W_{i_1,i_1}^3]
    \lesssim  C_4 d^9, \label{eq:four_distinct_1}
\end{align}
where the last inequality follows the same way as in \eqref{eq:wiwi2}. Now, we consider Case (b).  We first give an upper bound for the fourth moment of $\T_{ij}$ for $i\not=j$ defined in \eqref{eq:F_ij}: 
\begin{align}
    \E[\T_{ij}^4]&\lesssim \E[\langle \bx_i,\bx_j\rangle^{12}]+t^4\E [\langle \bx_i,\bx_j\rangle^4]
    \lesssim  C_4d^{6},\label{eq:F_ij_4}
\end{align}
where we use the estimate  \begin{align}\label{eq:y_moment}
    \E[\langle \bx_i,\bx_j\rangle^{2s}]\lesssim C_2^{2s} C_3^{2s} d^s.
\end{align}
Based on \eqref{eq:F_ij_4}, we know in Case (b),
\begin{align}
&\E[\T_{i_1i_2}\T_{i_2i_3}\T_{i_3i_4}\T_{i_4i_5}\T_{i_5i_6}\T_{i_6i_1}]=\E[\T_{i_1i_2}^2\T_{i_2i_5}\T_{i_5i_6}\T_{i_1i_6}] \notag \\
    =~ & \E[\T_{i_1i_2}^2\T_{i_2i_5}\E[\T_{i_5i_6}\T_{i_1i_6}| \bx_{i_1},\bx_{i_2}, \bx_{i_5} ]] \notag \\
    =~& \E[\T_{i_1i_2}^2\T_{i_2i_5}W_{i_1,i_5}]\le \E[\T_{i_1i_2}^4]^{1/2}\E[\T_{i_2i_5}^4]^{1/4}\E[W_{i_1,i_5}^4]^{1/4}\lesssim C_4 d^6,\label{eq:four_distinct_2}
\end{align}
where in the last inequality we use the estimate from \eqref{eq:F_ij_4} and \eqref{eq:wi4}. Similarly, with \eqref{eq:F_ij_4}, we can also get a bound for Case (c) by  
\begin{align}
&\E[\T_{i_1i_2}\T_{i_2i_3}\T_{i_3i_4}\T_{i_4i_5}\T_{i_5i_6}\T_{i_6i_1}]=\E[\T_{i_1i_2}^2\T_{i_1i_4}^2\T_{i_4i_5}^2] = \E[\T_{i_1i_4}^2\E[\T_{i_1i_2}^2 \ \T_{i_4i_5}^2| \bx_{i_1},\bx_{i_4} ]] \notag \\
    =~& \E[\T_{i_1i_4}^2W_{i_1,i_1}W_{i_4,i_4}] \le \E[\T_{i_1i_4}^4]^{1/2}\E[W_{i_1,i_1}^2] \lesssim C_4 d^9,\label{eq:four_distinct_3}
\end{align}
where in the last inequality, we use \eqref{eq:wiwi2}. Combining \eqref{eq:four_distinct_1}, \eqref{eq:four_distinct_2} and \eqref{eq:four_distinct_3}, we can conclude that for Case (iii),
\begin{align}
     \frac{1}{d^{18}} \sum_{i_1,\dots, i_6 \text{ have 4 distinct indices}} \E[\T_{i_1i_2}\T_{i_2i_3}\T_{i_3i_4}\T_{i_4i_5}\T_{i_5i_6}\T_{i_6i_1}] \leq \frac{n^4}{d^{18}}  C_4d^9 \lesssim C_1^4 C_4 d^{-1}. \label{eq:case3}
\end{align}

\item \textbf{Terms involving three different indices}.  By symmetry, we only need to consider the case where $i_1=i_3=i_5,i_2=i_4$ and $(i_1,i_2,i_6)$ are distinct. In this case,  
\begin{align}
&\E[\bT_{i_1i_2}\bT_{i_2i_3}\bT_{i_3i_4}\bT_{i_4i_5}\bT_{i_5i_6}\bT_{i_6i_1}]=\E[\bT_{i_1i_2}^4\bT_{i_1i_6}^2] 
    = \E[\bT_{i_1i_2}^4\E[\bT_{i_1i_6}^2 | \bx_{i_1}  ]] \notag \\
     =~& \E[\bT_{i_1i_2}^4W_{i_1,i_1} ]   \leq (\E\bT_{i_1i_2}^8)^{1/2} (\E W_{i_1,i_1}^2)^{1/2}  
     \lesssim~ C_4 d^{9},
\end{align}
where in the last inequality, we use  \eqref{eq:wiwi2} and the following estimate similar to \eqref{eq:F_ij_4}
\begin{align}
    \E[\T_{ij}^8]&\lesssim \E[\langle \bx_i,\bx_j\rangle^{24}]+t^8\E [\langle \bx_i,\bx_j\rangle^8]\lesssim C_4 d^{12}. \notag 
\end{align}

Thus, we can conclude that for Case (iv), we have 
\begin{align}
    \frac{1}{d^{18}} \sum_{i_1\neq i_2  \neq i_6\in [n]} \E[\bT_{i_1i_2}^4\bT_{i_1i_6}^2]\lesssim C_1^3 C_4 d^{-3}. \label{eq:case4}
\end{align}

\item  \textbf{Terms involving two different indices.} We only need to consider the case where $i_1=i_3=i_5,i_2=i_4=i_6$ and $(i_1,i_2)$ are distinct. In this case, $$\E[\bT_{i_1i_2}\bT_{i_2i_3}\bT_{i_3i_4}\bT_{i_4i_5}\bT_{i_5i_6}\bT_{i_6i_1}]=\E[\bT_{i_1i_2}^6].$$ Similar to \eqref{eq:F_ij_4}, we  have
$\E[\bT_{ij}^6]\lesssim  C_4 d^{9}$,
then all terms involving two different indices satisfy 
\begin{align}
    \frac{1}{d^{18}}\sum_{i_1\not=i_2} \E[\bT_{i_1i_2}^6] \leq C_1^2 C_4d^{-5}. \label{eq:case5}
\end{align}

In summary, based on \eqref{eq:F_3_four}, \eqref{eq:distinct_T}, \eqref{eq:case22}, \eqref{eq:case3}, \eqref{eq:case4}, and \eqref{eq:case5}, Cases $\mathrm{(i-v)}$ verify that 
$\E\|\widetilde \T\|^6\lesssim |f^{(3)}(0)|^6C_1^6 C_4 d^{-1}$. 
By Markov's inequality, with probability at least $1-d^{-\frac{1}{2}}$, 
\begin{align}\label{eq:tildeTbound}
    \| \widetilde{\T}\|\lesssim |f^{(3)}(0)| C_1 C_4^{1/6} d^{-\frac{1}{12}}.
\end{align}
\end{enumerate}

\subsubsection{Fourth-order approximation}
Now we analyze the spectral norm of $\widetilde \F$ defined in \eqref{def:tildeF}. Recall $t:=\Tr\bSigma^2=\E[\norm{\bw_i}^2]$.
We define 
$ \bF=\langle \bx_i,\bx_j\rangle^4-6t \langle \bx_i,\bx_j\rangle ^2+3t^2$.
We have
\begin{align}
      \E\|\widetilde \F\|^4&\leq \E\Tr(\widetilde \F^4)\lesssim \frac{|f^{(4)}(0)|^4}{d^{16}} \sum_{i_1,i_2,i_3,i_4\in [n]} \E[\F_{i_1i_2}\F_{i_2i_3}\F_{i_3i_4}\F_{i_4i_1} ].\label{eq:F_4_four}
\end{align}
With the explicit calculations in \eqref{eq:moment44}, \eqref{eq:moment22}, and \eqref{eq:moment42}, we  obtain that when $j\neq i$ and $j\neq k$,
\begin{align}
    &\E [\F_{ij}\F_{jk}| \bx_i,\bx_k] \notag\\
    =~& \E\left[ \left( \langle \bx_i,\bx_j\rangle^4-6t \langle \bx_i,\bx_j\rangle ^2+3t^2\right)\left( \langle \bx_k,\bx_j\rangle^4-6t \langle \bx_k,\bx_j\rangle ^2+3t^2\right) \mid \bx_i,\bx_k\right]  \notag \\
    =~& 24 \langle \bw_i,\bw_j\rangle^4+72(\norm{\bw_i}^2-t)(\norm{\bw_k}^2-t)\langle \bw_i,\bw_j\rangle^2+9(\norm{\bw_i}^2-t)^2(\norm{\bw_k}^2-t)^2. \label{eq:condition_F}
\end{align}
For simplicity, for any $j\neq i,k$, we denote $U_{i,k}: = \E [\F_{ij}\F_{jk}| \bx_i,\bx_k]$. 
When $i\neq k$, using the estimates in \eqref{eq:wiwk}, \eqref{eq:centered6}, and the explicit calculation in \eqref{eq:condition_F}, we have
$ \E[U_{i,k}^2]\lesssim C_4 d^4$, and when $i=k$,
 $\E[U_{i,i}^2]\lesssim C_4d^8$. Then, we consider the following 3 cases for the number of distinct indices involved in the summation of \eqref{eq:F_4_four}.
\begin{enumerate}[label=\textbf{Case (\roman*)},leftmargin=0ex, itemindent=12ex]
    \item 
 We first assume $i_1,i_2,i_3,i_4\in [n]$ are distinct. Conditioned on $\bx_{i_1}$ and $\bx_{i_3}$, we know that
$\E[\F_{i_1i_2}\F_{i_2i_3}\F_{i_3i_4}\F_{i_4i_1} | \bx_{i_1},\bx_{i_3}]=U_{i_1,i_3}^2$.
Thus, in this case,
\begin{equation}
    \frac{1}{d^{16}} \sum_{i_1\neq i_2\neq i_3\neq i_4\in [n]} \E[\F_{i_1i_2}\F_{i_2i_3}\F_{i_3i_4}\F_{i_4i_1} ]\lesssim C_1^4 C_4 d^{-4}.\label{eq:F_4_four_4distinct}
\end{equation}

\item Terms involving three different indices.  Without loss of generality, it suffices to consider $i_1=i_3$ and $(i_1,i_2,i_4)$ are all distinct. Similarly, in this case,
\begin{equation}
    \frac{1}{d^{16}} \sum_{i_1\neq i_2\neq i_4\in [n]} \E[\F_{i_1i_2}\F_{i_2i_1}\F_{i_1i_4}\F_{i_4i_1} ]=\frac{1}{d^{16}} \sum_{i\neq i_2\neq i_4\in [n]} \E[U_{i,i}^2] \lesssim C_1^3 C_4  d^{-2}. \label{eq:F_4_four_3distinct}
\end{equation}

\item Terms involving two different indices. By symmetry, we only need to consider the case when $i_1=i_3$, $i_2=i_4$ and $(i_1,i_2)$ are distinct. Notice that for $i\neq j$, 
 \begin{align*}
      \E[\F_{ij}^4]\lesssim \E[\langle \bx_i,\bx_j\rangle^{16}]+t^4\E [\langle \bx_i,\bx_j\rangle^8]+t^8\lesssim C_4d^{8},
 \end{align*}
 where the last inequality is due to \eqref{eq:y_moment} and \eqref{eq:bound_t}.
Hence, in this case,
\begin{equation}
    \frac{1}{d^{16}} \sum_{i_1\neq i_2\in [n]} \E[\F_{i_1i_2}^4] \lesssim C_1^2 C_4 d^{-4}. \label{eq:F_4_four_2distinct}
\end{equation}
\end{enumerate}
Combining equations \eqref{eq:F_4_four_4distinct}, \eqref{eq:F_4_four_3distinct} and \eqref{eq:F_4_four_2distinct}, we can  conclude that 
  $\E\|\widetilde \F\|^4\lesssim |f^{(4)}(0)|^4 C_1^4 C_4 d^{-2}$.
 Hence, by Markov's inequality, with probability at least $1-d^{-1/2}$, 
\begin{align}\label{eq:tildeFbound}
    \|\widetilde \bF\| \lesssim |f^{(4)}(0)| C_1 C_4^{1/4} d^{-3/8}.
\end{align}

\subsubsection{Higher-order terms}\label{sec:high_order}
In this section, we bound the spectral norm of $\widetilde \bV$ defined in \eqref{def:tildeV}.
   For any $i\not=j$, we have from \eqref{eq:y_moment},
$\mathbb E [\langle \bx_i,\bx_j\rangle ^{90}]\lesssim C_4 d^{45}$.
    By Markov's inequality, with probability at least $1-n^{-2} d^{-\frac{1}{2}}$, $|\langle \bx_i,\bx_j\rangle|\lesssim  C_1^{\frac{2}{90}}C_4^{\frac{1}{90}}d^{\frac{11}{20}}$. Then taking a union bound over all pairs of $i,j\in [n], i\not=j$, we find with probability $1-d^{-1/2}$, 
    \begin{align}\label{eq:maxxixj}
    \frac{1}{d} \max_{i\not=j} |\langle \bx_i,\bx_j\rangle|\lesssim C_1^{\frac{2}{90}}C_4^{\frac{1}{90}}d^{-\frac{9}{20}}.    
    \end{align}
Recall the definition of $\zeta_{ij}$ in \eqref{eq:taylor_off}.
From \eqref{eq:maxxixj},  we have with probability at least $1-d^{-1/2}$, 
    $\sup_{i\not=j}|\zeta_{ij}|\lesssim C_1^{\frac{2}{90}}C_4^{\frac{1}{90}}d^{-\frac{9}{20}}$.
Since $f^{(5)}$ is continuous at $0$, there exist  constants $C_5, C_6\geq 1$ depending only on $f$ such that for $d\geq C_5C_1^{\frac{1}{100}} C_4^{\frac{1}{200}}$, with probability at least $1-d^{-1/2}$, $\sup_{i\not=j}|f^{(5)}(\zeta_{ij})|\leq C_6$. 
Therefore, with probability at least $1-d^{-1/2}$, for $d\geq C_5C_1^{\frac{1}{100}} C_4^{\frac{1}{200}}$, \[\|\widetilde \V\|^2\leq \|\widetilde\V\|_{\msf F}^2\lesssim C_6^2 n^2  d^{-10} \max_{i\not=j} |\langle \bx_i,\bx_j\rangle|^{10}\lesssim C_6^2 C_1^{\frac{20}{9}}C_4^{\frac{1}{9}} d^{-1/2}.\]
Hence with probability at least  $1-d^{-1/2}$, for $d\geq C_5C_1^{\frac{1}{100}} C_4^{\frac{1}{200}}$,
\begin{align}\label{eq:tildeVbound}
     \|\widetilde \V\| \lesssim C_6 C_1^{\frac{10}{9}}C_4^{\frac{1}{18}} d^{-\frac{1}{4}}.
\end{align}

\subsection{Controlling the error in the diagonal terms}\label{sec:diag_term}

Recall from \eqref{eq:K2}, the diagonal elements of $\bKquad$ can be written as 
\begin{align}
\bKquad_{ii}&=\left(f(0) -\frac{f^{(4)}(0)(\Tr(\bSigma^2))^2}{8d^4}\right) + \left(\frac{f'(0)}{d}+\frac{\Tr(\bSigma^2)}{2d^3} \right)\|\bx_i\|^2  \notag\\
&+\left(\frac{f''(0)}{2d^2}+\frac{f^{(4)}(0)\Tr(\bSigma^2)}{4d^4}\right) \|\bx_i\|^4+ a , \notag 
\end{align}
where  $a$ is defined in \eqref{def:a}. We can reorder the terms and write
\begin{align}
 \bKquad_{ii}-f\left(\frac{\Tr\bSigma}{d}\right)=&\frac{f'(0)}{d} \left( \|\bx_i\|^2-\Tr\bSigma \right)+\frac{f''(0)}{2d^2}\left(  \|\bx_i\|^4-(\Tr\bSigma)^2\right) \\
 & +\frac{f^{(4)}(0) \Tr(\bSigma^2)}{4d^4}-\frac{f^{(4)}(0) (\Tr(\bSigma^2))^2}{8d^4}.\label{eq:Kii2}
\end{align}
And
 $\bK_{ii}-f\left(\frac{\Tr\bSigma}{d}\right)=f'\left( \frac{\Tr\bSigma}{d}\right)\left( \frac{\norm{\bx_i}^2}{d}-\frac{\Tr\bSigma}{d}\right)+\frac{f''(\xi_{ii})}{2}\left( \frac{\norm{\bx_i}^2}{d}-\frac{\Tr\bSigma}{d}\right)^2$.

Let $\widetilde \bD$ be a diagonal matrix such that 
$
    \widetilde \bD_{ii}= \bK_{ii}- \bKquad_{ii}$.
We first simplify $\bK_{ii}$ and $\bK_{ii}^{(2)}$.
Recall $\bx_i=\bSigma^{1/2} \bz_i$ from Assumption~\ref{assump:data}. With Whittle’s inequality in Lemma \ref{lem:quadratic_moments}, for any integer $s\geq 1$, 
$ \mathbb E\left(\|\bx_i\|^2-\Tr \bSigma\right)^{12}= \E \left( \bz_i^\top \bSigma \bz_i -\Tr\bSigma \right)^{12} \lesssim C_2^{12} \|\bSigma\|_{\msf F}^{12} \lesssim  C_2^{12} C_3^{12} d^6$,
where we use the inequality $\|\bSigma\|_{\msf F} \leq \sqrt{d} \|\bSigma\|\leq C_3\sqrt{d}$. By Markov's inequality and a union bound over $i\in [n]$, we have with probability at least $1-d^{-1}$,  
\begin{align} \label{eq:normxi}
\frac{1}{d} \sup_{i\in [n]}\left|  \|\bx_i\|^2-\Tr\bSigma\right|\lesssim C_1^{\frac{1}{12}}C_2C_3  d^{-\frac{1}{4}}.
\end{align}
Recall $\xi_{ii}$  in \eqref{eq:taylor_diag} is  between $0$ and $\frac{1}{d} \|\bx_i\|^2$.   From \eqref{eq:normxi}, there exist constant $C_5, C_6$ depending only on $f$ such that with probability $1-d^{-1}$, for $d\geq C_5C_1^{1/4} (C_2C_3)^4$, $\max_{i\in [n]}|f''(\xi_{ii})|\leq C_6$.
This implies  with probability $1-d^{-1}$, 
\begin{align}\label{eq:Kii_approx}
\left|\bK_{ii}-f\left(\frac{\Tr\bSigma}{d}\right)\right|\lesssim C_1^{\frac{1}{12}}C_6 C_4d^{-1/4}.
\end{align}
On the other hand, from \eqref{eq:normxi}, with probability at least $1-d^{-1}$,
  $$\max_{i\in [n]}\left|\|\bx_i\|^4-(\Tr(\bSigma))^2\right|\lesssim C_1^{\frac{1}{12}}C_4 d^{\frac{7}{4}}.$$
From \eqref{eq:Kii2}, this implies
\begin{align}\label{eq:Kii2_approx}
    \left|\bK_{ii}^{(2)}-f\left(\frac{\Tr\bSigma}{d}\right)\right|&\lesssim C_1C_4C_6 d^{-\frac{1}{4}}.
\end{align}
Therefore, from \eqref{eq:Kii_approx} and \eqref{eq:Kii2_approx}, with  probability at least $1-d^{-1}$, for $d\geq C_1C_4C_5$,
\begin{align} \label{eq:tildeDbound}
\norm{\widetilde \bD} =\max_{i\in [n]}|\bK_{ii}^{(2)}-\bK_{ii}|\lesssim C_1C_4C_6d^{-\frac{1}{4}}.
\end{align}

\subsection{Putting all bounds together}\label{sec:together}
Finally, we combine the error bounds in Sections \ref{sec:off_diagonal_term} and \ref{sec:diag_term}  to finish the proof. 
From the estimates of the spectral norm for $\widetilde \bT, \widetilde \bF, \widetilde \bV$, and $\widetilde \bD$ in \eqref{eq:tildeTbound}, \eqref{eq:tildeFbound}, \eqref{eq:tildeVbound}, \eqref{eq:tildeDbound}, respectively,  we have with probability at least $1-4d^{-1/2}$, for $d\geq C1C_4C_5$,
$
    \big\|\bK-\bKquad \big\| \leq \big\|\widetilde \bT\big\| +\big\|\widetilde \bF\big\|+\big\|\widetilde \bV\big\|+\big\|\widetilde \bD\big\| \lesssim C_1^2 C_4 C_6 d^{-\frac{1}{12}}.
$
This completes the proof of Theorem~\ref{thm:concentration}.


\section{Proof of Theorem \ref{thm:globallaw}}\label{sec:limiting_law}

Recall the reduced tensor product $\bx^{(2)}$ defined in \eqref{eq:def_reduced}.  Let $\bX^{(2)}=[ \bx_1^{(2)},\dots, \bx_n^{(2)}]^\top \in \mathbb R^{n\times \binom{d+1}{2}}$. Then from \eqref{eq:tensor_product_identity}, we have  \begin{align}\label{eq:sample_covariance_equivalence}
(\bX\bX^\top)^{\odot 2}=\bX^{(2)}{\bX^{(2)}}^\top.
\end{align}
Here, $\bX^{(2)}{\bX^{(2)}}^\top$ is a sample covariance matrix, where $\bX^{(2)}$ has independent rows. We will use Lemma \ref{lem:BaiZhou} from \cite{bai2008large} in our setting.

\subsection{Variance of random quadratic forms}

\begin{lemma} Let $\bx\in \mathbb R^d$ be a random vector with independent entries and a diagonal covariance matrix $\bSigma$, where $\norm{\bSigma}\leq C$ for constant $C>0$. Assume each entry of $\bx$ has a zero mean and bounded 8th moments.   Let $\bx^{(2)}\in \mathbb R^{\binom{d+1}{2}}$ be a corresponding reduced tensor  vector  defined in \eqref{eq:def_reduced} and we define
      \begin{align}\label{eq:centered_x2}
      \overline\bx^{(2)}:=\bx^{(2)}-\E \bx^{(2)}.
      \end{align}
Then for any deterministic matrix $\bA$ with $\|\bA\|\leq 1$,
\begin{align}\label{eq:BZ_tensor}
    \E\left| {\overline\bx^{(2)}}^\top \bA\overline\bx^{(2)} -\Tr[\bA \bSigma^{(2)}]\right|^2=O(d^3).
     \end{align}
\end{lemma}

\begin{proof}
     We let $\bA=\bD+\bB \in \mathbb R^{\binom{d+1}{2}\times \binom{d+1}{2}}$, where $\bD$ is the diagonal part of $\bA$, and $\bB$ is the off-diagonal component of $\bA$.  Here the matrix $\bA$ is index by $\{(i,j): i\leq j, \quad i,j\in [d] \}$.
      To show \eqref{eq:BZ_tensor}, it suffices to bound the contribution from $\bD$ and $\bB$.

\textbf{ (i) Diagonal part.} Recall the definition of $\bx^{(2)}$ from \eqref{eq:def_reduced}. We have
\begin{align}
    &\E\left| {\overline\bx^{(2)}}^\top \bD\overline\bx^{(2)}-\Tr[\bD\bSigma^{(2)}]\right|^2\\
    =~&\E \left( \sum_{i<j}2(\bx_i^2 \bx_j^2-\bSigma_{ij,ij}^{(2)})\bA_{ij,ij} +\sum_{i} ((\bx_i^2-\bSigma_{ii})^2-\bSigma_{ii,ii}^{(2)}) \bA_{ii,ii}\right)^2 \notag \\
     \leq~& 4\sum_{i<j,k< l}|\bA_{ij,ij}\bA_{kl,kl}|\left|\E[ (\bx_{i}
    ^2\bx_j^2-\bSigma_{ij,ij}^{(2)})(\bx_{k}^2\bx_{l}^2-\bSigma_{kl,kl}^{(2)})]\right|\label{eq:D_off_term}\\
     &~~+\sum_{i,j} |\bA_{ii,ii}\bA_{jj,jj}| \left|\E[((\bx_i^2-\bSigma_{ii})^2-\bSigma_{ii,ii}^{(2)})((\bx_j^2-\bSigma_{jj})^2-\bSigma_{jj,jj}^{(2)})]\right|. \label{eq:D_diag_term}
\end{align}
    Since the $8$-th moments of $\bx_i$ are bounded for all $i\in [d]$, the contribution from \eqref{eq:D_diag_term} is at most $O(d)$. For \eqref{eq:D_off_term}, when $i,j,k,l$ are all distinct, by the diagonal assumption on $\bSigma$, $\bx_i,\bx_j,\bx_k,\bx_l$ are independent. We have
    $\E[ (\bx_{i}
    ^2\bx_j^2-\bSigma_{ij,ij}^{(2)})(\bx_{k}^2\bx_{l}^2-\bSigma_{kl,kl}^{(2)})]=0$.
Therefore, the nonzero contribution of \eqref{eq:D_off_term} only comes from indices $(i,j,k,l)$ that are not distinct. Since $\|\bD\|\leq \|\bA\|\leq 1$, we know the contribution with repeated indices  $(i,j,k,l)$ in \eqref{eq:D_off_term} is $O(d^3)$. Therefore, the total contribution from the diagonal part is $O(d^3)$.

\textbf{ (ii) Off-diagonal part.} We have the following expansion:
\begin{small}
\begin{align}
    &\E\left| {\overline\bx^{(2)}}^\top \bB\overline\bx^{(2)}-\Tr[\bB\bSigma^{(2)}]\right|^2 
    =\sum_{(i_1,i_2)\not=(i_3,i_4), (i_5,i_6)\not=(i_7,i_8)}\bA_{i_1i_2,i_3i_4}\bA_{i_5i_6,i_7i_8}\E[\overline\bx^{(2)}_{i_1i_2}\overline\bx^{(2)}_{i_3i_4}\overline\bx^{(2)}_{i_5i_6}\overline\bx^{(2)}_{i_7i_8}] \notag\\
    \lesssim & \sum_{(i_1,i_2)\not=(i_3,i_4), (i_5,i_6)\not=(i_7,i_8)} |\bA_{i_1i_2,i_3i_4}\bA_{i_5i_6,i_7i_8}|\label{eq:off_diag_8}\\
    &\cdot| \E[(\bx_{i_1}\bx_{i_2}-\bSigma_{i_1,i_2}\delta_{i_1,i_2})(\bx_{i_3}\bx_{i_4}-\bSigma_{i_3,i_4}\delta_{i_3,i_4})(\bx_{i_5}\bx_{i_6}-\bSigma_{i_5,i_6}\delta_{i_5,i_6})(\bx_{i_7}\bx_{i_8}-\bSigma_{i_7,i_8}\delta_{i_7,i_8})]|. 
\end{align}
\end{small}
For each index sequence $i_1,\dots, i_8$, to have a nonzero contribution in 
\begin{small}
\begin{align}\label{eq:i8}
\E[(\bx_{i_1}\bx_{i_2}-\bSigma_{i_1,i_2}\delta_{i_1,i_2})(\bx_{i_3}\bx_{i_4}-\bSigma_{i_3,i_4}\delta_{i_3,i_4})(\bx_{i_5}\bx_{i_6}-\bSigma_{i_5,i_6}\delta_{i_5,i_6})(\bx_{i_7}\bx_{i_8}-\bSigma_{i_7,i_8}\delta_{i_7,i_8})]
\end{align}
\end{small}
by the independence of the entries in $\bx$, there are at most 4 distinct values among $i_1,\dots, i_8$.  For sequences with at most 3 distinct indices, their total contribution in \eqref{eq:off_diag_8} is $O(d^3)$. Therefore, it suffices to estimate \eqref{eq:off_diag_8} when the contribution of index sequences with exactly 4 distinct indices satisfies $i_1\leq i_2,   i_3\leq i_4,  i_5\leq i_6,  i_7\leq i_8$. We have only the following cases depending on the number of distinct indices in $i_1,i_2,i_3,i_4$:
\begin{enumerate}
    \item Assume there are exactly 4 distinct indices in $i_1,\dots,i_4$. Then, to have a nonzero contribution, there is a perfect matching between $\{i_1,\dots, i_4\}$ and $\{i_5,\dots, i_8\}$. Using the inequality $2|\bA_{i_1i_2,i_3i_4}\bA_{i_5i_6,i_7i_8}|\leq |\bA_{i_1i_2,i_3i_4}|^2+|\bA_{i_5i_6,i_7i_8}|^2$, for an absolute constant $C$, the contribution is bounded by 
    \begin{align}
        C \left(\sum_{i_1< i_2,i_3< i_4} |\bA_{i_1i_2,i_3i_4}|^2\right)=C\|\bA\|_{\msf F}^2\leq C d^2 \|\bA\|^2 =O(d^2).
    \end{align}
    \item Assume there are exactly three distinct indices among $i_1,\dots, i_4$.  By symmetry, we only need to consider four subcases 
    \begin{itemize}
        \item (a) $i_1=i_2$, and $i_1,i_3,i_4$ are distinct. We can rewrite \eqref{eq:i8} as 
        \begin{align}\label{eq:i8a}
            \E[(\bx_{i_1}^2-\bSigma_{ii})\bx_{i_3}\bx_{i_4}(\bx_{i_5}\bx_{i_6}-\bSigma_{i_5,i_6}\delta_{i_5,i_6})(\bx_{i_7}\bx_{i_8}-\bSigma_{i_7,i_8}\delta_{i_7,i_8})].
        \end{align}
      Since there are exactly 4 distinct indices among $i_1,\dots,i_8$, and $i_1$ appears exactly twice,  $i_3,i_4,i_5,i_6,i_7,i_8$ must be distinct from $i_1$, which implies \eqref{eq:i8a} is equal to zero by independence.

        \item (b) $i_1=i_3$, and $i_1,i_2,i_4$ are distinct. We can rewrite \eqref{eq:i8} as 
        \begin{align}\label{eq:i8b}
            \E[\bx_{i_1}^2\bx_{i_2}\bx_{i_4}(\bx_{i_5}\bx_{i_6}-\bSigma_{i_5,i_6}\delta_{i_5,i_6})(\bx_{i_7}\bx_{i_8}-\bSigma_{i_7,i_8}\delta_{i_7,i_8})].
        \end{align}
   Note that if $i_5=i_6$ and $i_1,i_2,i_4,i_5$ are distinct, the expectation in \eqref{eq:i8b} is zero. By symmetry, we only need to consider $i_5=i_7, i_5=i_8$, or $i_5=i_2$.
     \begin{itemize}
         \item  (b.1) If $i_5=i_7$ and $i_1,i_2,i_4,i_5$ are distinct, we must have (i) $i_6=i_2$, $i_8=i_4$ or (ii) $i_6=i_4, i_8=i_2$. In case (i), we can bound \eqref{eq:off_diag_8} by 
        \begin{align}
            &\sum_{i_1\leq i_2, i_4, i_5} |\bA_{i_1i_2,i_1i_4}\bA_{i_5i_2, i_5,i_4}|\cdot \E[\bx_{i_1}^2\bx_{i_2}^2\bx_{i_4}^2\bx_{i_5}^2]\\
            \lesssim  &\sum_{i_1,i_2,i_4,i_5}\bA_{i_1i_2,i_1i_4}^2 +\sum_{i_1,i_2,i_4,i_5}\bA_{i_1i_2,i_1i_4}^2\lesssim d\|\bA\|_{\msf F}^2=O(d^3).
        \end{align}
        In case (ii), similarly, we can bound \eqref{eq:off_diag_8} by 
        \begin{align}
            &\sum_{i_1\leq i_2, i_4, i_5} |\bA_{i_1i_2,i_1i_4}\bA_{i_5i_4, i_5,i_2}|\cdot \E[\bx_{i_1}^2\bx_{i_2}^2\bx_{i_4}^2\bx_{i_5}^2]=O(d^3).
        \end{align}
        \item (b.2) If $i_5=i_8$, we must have (i) $i_6=i_2, i_7=i_4$ or (ii) $i_7=i_2, i_6=i_4$. In both cases, similar to case (b.1), the contribution is $O(d^3)$. 
        \item (b.3) If $i_5=i_2$, we must have (i) $i_7=i_4,i_8=i_6$ or (ii) $i_7=i_6, i_8=i_4$, and their contribution is $O(d^3)$.
     \end{itemize}

        \item (c) $i_2=i_4$, and $i_1,i_2,i_3$ are distinct. Like Case (b), its contribution is $O(d^3)$.
    \item (d) $i_1=i_4$ and $i_1,i_2,i_3$ are distinct. The same bound $O(d^3)$ holds.
    \end{itemize}

    \item  Assume there are exactly two distinct indices among $i_1,\dots, i_4$. We must have $i_1=i_2,i_3=i_4$, $i_1\not=i_3$ due to the constraint $(i_1,i_2)\not=(i_3,i_4)$.  In the same way, we must have $i_5=i_6, i_7=i_8, i_5\not=i_7$. 
Since there are 4 distinct indices among  $i_1,\dots,i_8$,   \eqref{eq:i8} becomes
  $ 
\E[(\bx_{i_1}^2-\bSigma_{i_1,i_1})(\bx_{i_3}^2-\bSigma_{i_3,i_3})(\bx_{i_5}^2-\bSigma_{i_5,i_5})(\bx_{i_7}^2-\bSigma_{i_7,i_7})]=0$.
    Therefore, the total contribution in this case is $0$.
\end{enumerate}
By the constraint $(i_1,i_2)\not=(i_3,i_4)$, there are at least 2 distinct indices among $i_1,\dots,i_4$. Therefore, we have discussed all three cases, and the total contribution for part (ii) is $O(d^3)$.
From the estimates in parts (i) and (ii) above, \eqref{eq:BZ_tensor} holds. 
\end{proof}

\subsection{Limiting spectral distributions}
We first obtain the  limiting spectral distribution of $\frac{1}{n}(\bX  \bX^\top)^{\odot 2}$ as follows.
\begin{lemma}\label{lem:MP_tensor}
Under Assumptions \ref{assump:data}-\ref{assump:nonlinear_f} and Assumptions \ref{assump:limitratio}-\ref{assump:limitsigma}, the limiting spectral distribution of $\frac{1}{n}(\bX \bX^\top )^{\odot 2}$ is a deformed Marchenko-Pastur law $\mu_{\alpha,\bSigma^{(2)}}$ given  in \eqref{eq:defmu}.      In particular, when $\bSigma=\bI_d$, the limiting spectral distribution of $\frac{1}{2n}(\bX  \bX^\top)^{\odot 2}$ is given by 
  \begin{align} \label{eq:law_MP}
  \begin{cases}
             (1-\alpha)\delta_0+\alpha \nu_{\alpha} & 0<\alpha <1\\
             \alpha\nu_{\alpha}& \alpha\geq 1.
         \end{cases}
    \end{align}
\end{lemma}

\begin{proofoflemma}{\ref{lem:MP_tensor}}
From \eqref{eq:sample_covariance_equivalence}, the eigenvalues of $\frac{1}{n}(\bX \bX^\top )^{\odot 2}$ and $\frac{1}{n} {\bX^{(2)}}^\top \bX^{(2)}$ is the same, up to $\left| n-\binom{d+1}{2}\right|$ many zero eigenvalues.
Now, we apply  Lemma~\ref{lem:BaiZhou} to show the convergence of ESD for $\frac{1}{n} {\bX^{(2)}}^\top \bX^{(2)}$. 
Notice that
\begin{align}
 &\frac{1}{n} {\bX^{(2)}}^\top \bX^{(2)} =\frac{1}{n}   \overline{\bX}^{(2)\top}\overline{\bX}^{(2)}-\frac{1}{n} {\bX^{(2)}}^\top\E\bX^{(2)}-\frac{1}{n} {\E\bX^{(2)}}^\top\bX^{(2)}+\frac{1}{n} {\E\bX^{(2)}}^\top\E\bX^{(2)} \label{eq:X2decomposition}
\end{align}
where we define
$
    \overline{\bX}^{(2)}:={\bX^{(2)}-\E\bX^{(2)}}$,
and $\E\bX^{(2)}$ has rank at most $d=o(n)$ due to \eqref{eq:Ex2}. From Lemma~\ref{lem:lowrank}, $\frac{1}{n} {\bX^{(2)}}^\top \bX^{(2)}$ and $\frac{1}{n}   \overline{\bX}^{(2)\top}\overline{\bX}^{(2)}$ have the same limiting spectral distribution.  Since $[{\bX^{(2)}-\E\bX^{(2)}}]^\top$ has independent columns and $\binom{d+1}{2}/n \to \alpha$, by \eqref{eq:BZ_tensor}, Lemma \ref{lem:BaiZhou}, and \eqref{eq:X2decomposition}, the empirical spectral distribution  of $\frac{1}{n} {\bX^{(2)}}^\top \bX^{(2)}$ converges weakly in probability to  $\mu_{\alpha}^{\mathrm{MP}}\boxtimes \mu_{\bSigma^{(2)}}$ where $\mu_{\alpha}^{\mathrm{MP}}$ is defined by \eqref{eq:def_mualpha}.  Next, we translate the result to $\frac{1}{n}(\bX \bX^\top )^{\odot 2}$. There are two cases:
\begin{enumerate}
    \item Suppose $\alpha< 1$, then the limiting spectral distribution of  $\frac{1}{n}(\bX \bX^\top )^{\odot 2}$ has a $(1-\alpha)\delta_{0}$ singular part at zero. The remaining part with $\alpha$ probability mass is  $\alpha \left(\nu_{\alpha}\boxtimes \mu_{\bSigma^{(2)}}\right)$. So the limiting spectral distribution for  $\frac{1}{n}(\bX \bX^\top )^{\odot 2}$  is 
    $(1-\alpha)\delta_0+\alpha \left(\nu_{\alpha}\boxtimes \mu_{\bSigma^{(2)}}\right)$.
    \item Suppose $\alpha\geq 1$. Then the limiting spectral distribution of $\frac{1}{n} {\bX^{(2)}}^\top \bX^{(2)}$ is $(1-\frac{1}{\alpha})\delta_{0}+\nu_{\alpha}\boxtimes \mu_{\bSigma^{(2)}}$, and the limiting spectral distribution of $\frac{1}{n}(\bX \bX^\top )^{\odot 2}$ is given by $\alpha\left(\nu_{\alpha}\boxtimes \mu_{\bSigma^{(2)}}\right).$
\end{enumerate}
In particular, when $\bSigma=\bI$, from \eqref{eq:defSigma2}, the limiting spectral distribution of $\bSigma^{(2)}$ is $\delta_{2}$. Therefore $\frac{1}{2n}(\bX \bX^\top )^{\odot 2}$ has a limiting spectral distribution given by \eqref{eq:law_MP}.
\end{proofoflemma}

\begin{proofoftheorem}{\ref{thm:globallaw}}
Due to Theorem \ref{thm:concentration} and Lemma~\ref{lem:BS10A45}, $\bK^{(2)}-a \bI$ and $\bK-a\bI$ have the same limiting spectral distribution, where 
\begin{align} 
\bK^{(2)}=   &\left(f(0)-\frac{f^{(4)}(0)(\Tr(\bSigma^2))^2}{8d^4}\right)\mathbf{1}\mathbf{1}^\top   +\left(\frac{f'(0)}{d}+\frac{f^{(3)}(0)\Tr(\bSigma^2)}{2d^3} \right)\bX\bX ^\top\label{eq:rank-1}\\
& +\left(\frac{f''(0)}{2d^2}+ \frac{f^{(4)}(0)\Tr(\bSigma^2)}{4d^4}\right)   \left(\bX\bX^\top \right)^{\odot 2} + a \bI\label{eq:rank-d},
\end{align}
and $a$ is defined in \eqref{def:a}. The first term and the second term in \eqref{eq:rank-1} have rank 1 and rank $d$, respectively, which both are $o(n)$ in the quadratic regime $n\asymp d^2$. Therefore, by Lemma \ref{lem:lowrank}, $\frac{4\alpha}{f''(0)} \left(\bK^{(2)}-a \bI\right)$ has the same limiting spectral distribution as $\frac{1}{n}\left(\bX \bX^\top\right)^{\odot 2}$. Finally, from Lemma~\ref{lem:MP_tensor}, the limiting law for $\frac{4\alpha}{f''(0)}(\bK-a \bI)$ is $\mu_{\alpha,\bSigma^{(2)}}$ defined in \eqref{eq:defmu}.
\end{proofoftheorem}

\section{Proof of Theorem \ref{thm:train_limit}}\label{sec:krr}

\subsection{Smallest eigenvalue bounds}

\begin{lemma}\label{lem:asym_sigma_min}
Under the same assumptions as Theorem~\ref{thm:train_limit} and the additional Assumption~\ref{assump:analytic}, we have 
$
    \lambda_{\min}(\bK^{(2)})\geq a_{*}-o(1)$,
where $a_*$ is defined in \eqref{def:a_star}.
And with probability $1-O(d^{-1/2})$,
$\lambda_{\min}(\bK)\geq a_{*}-o(1)$.
In particular, for sufficiently large $n$,
$\lambda_{\min}(\bK^{(2)})\geq \frac{a_*}{2},$ and $\lambda_{\min}(\bK)\geq \frac{a_*}{2}$.
\end{lemma}
\begin{proof}
   Recall  $ \bKquad$ from \eqref{eq:Kquad}. 
Since $\one\one\tran, \bX\bX\tran$, and $(\bX\bX\tran)^{\odot 2}$ are all positive semidefinite, from Assumption~\ref{assump:analytic}, we obtain $
    \lambda_{\min}(\bK^{(2)})\geq a_*-o(1)$.
From Theorem \ref{thm:concentration},  with probability $1-O(d^{-1/2})$,
$\lambda_{\min}(\bK)\geq a_{*}-O(d^{-\frac{1}{12}})-o(1)$.
This finishes the proof.
\end{proof}

\subsection{Quadratic approximation of training errors}
We define an approximate training error by replacing the original kernel $\bK$ by $\bK^{(2)}$ in \eqref{eq:K2}:
\begin{align}\label{eq:Etrain2}
   \Etrain^{(2)}:=~&    \frac{\lambda^2}{n}\by^\top(\bK^{(2)} +\lambda\bI_n)^{-2} \by.
\end{align}
Then we show the following approximation bound of training error $\Etrain$ in \eqref{eq:Etrain} via \eqref{eq:Etrain2}.
\begin{lemma}\label{lemm:approx_train}
For any $\lambda\geq 0$, under the same assumptions as Theorem~\ref{thm:train_limit}, there exists some constant $C>0$ such that with probability at least $1-O(d^{-1/2})$ for sufficiently large $d$,
    \[
        |\Etrain -\Etrain^{(2)}|\le\frac{C\lambda^2\norm{\by}^2}{a_*^3 n} \cdot d^{-\frac{1}{12}}.
   \]
\end{lemma}
\begin{proof}
    Following the proof of \cite[Theorem 2.7]{wang2023overparameterized}, we have 
\begin{align}
     &\left| \Etrain -\Etrain^{(2)}\right|= \frac{\lambda^2}{n} \left|\Tr[(\bK+\lambda\bI_n)^{-2} \v y\v y^\top]-\Tr[(\bK^{(2)} +\lambda\bI_n)^{-2} \v y\v y^\top] \right| \notag \\
     =~& \frac{\lambda^2}{n} \left|\v y^\top\left[(\bK +\lambda\bI_n)^{-2}-(\bK^{(2)} +\lambda\bI_n)^{-2}\right] \v y\right| \notag \\
     \leq ~&\frac{\lambda^2}{n} \|(\bK +\lambda\bI_n)^{-2}-(\bK^{(2)} +\lambda\bI_n)^{-2}\|\cdot  \|\v y\|^2 \notag\\
     \leq ~& \frac{\lambda^2\|\v y\|^2}{n} \|(\bK +\lambda\bI_n)^{-1}-(\bK^{(2)}+\lambda\bI_n)^{-1}\|\cdot(\|(\bK +\lambda\bI_n)^{-1}\|+\|(\bK^{(2)} +\lambda\bI_n)^{-1}\|) \notag\\
     \leq ~&\frac{4\lambda^2\|\v y\|^2}{ a_* n} \| (\bK +\lambda\bI_n)^{-1}-(\bK^{(2)} +\lambda\bI_n)^{-1} \| 
 \notag\\
     \leq ~& \frac{4\lambda^2\|\v y\|^2}{ a_* n} \| (\bK +\lambda\bI_n)^{-1}\|\cdot\|(\bK^{(2)} +\lambda\bI_n)^{-1} \|\cdot\norm{\bK-\bK^{(2)}}
     \leq \frac{C\lambda^2\norm{\by}^2}{a_*^3 n} \cdot d^{-1/12}, \notag
\end{align}
with probability at least $1-O(d^{-1/2})$. In the fourth and the last lines, we use Theorem \ref{thm:concentration} and  the fact that for sufficiently large $d$, from Lemma~\ref{lem:asym_sigma_min} and the assumption that $a_*>0$,
\begin{equation}\label{eq:K_lambda_-1}
    \norm{(\bK^{(2)} +\lambda\bI_n)^{-1}}\leq  \frac{2}{a_*}, \quad  \norm{(\bK +\lambda\bI_n)^{-1}} \leq \frac{2}{a_*},
\end{equation}with probability at least $1-O(d^{-1/2})$.
This finishes the proof.
\end{proof}

\begin{lemma}\label{lemm:bound_y} 
   Under the same assumptions as Theorem~\ref{thm:train_limit}, $\frac{1}{n}\norm{\by}^2 d^{-\frac{1}{24}}=o(1)$ with high probability.
\end{lemma}
\begin{proof}
Denote $\vf_*=[f_*(\bx_1),\ldots,f_*(\bx_n)]^\top$. Then
    $\by= \vf_* +\bepsilon$,
and $\bepsilon$ is a sub-Gaussian vector with mean zero and variance $\sigma_{\bepsilon}^2$.
By concentration of sub-Gaussian random vectors \citep{vershynin2018high}, $\|\bepsilon\|=O(\sqrt n)$ with high probability.  Recall $f_*(\bx_i)=c_0+c_1\langle \bbeta, \bx_i\rangle+\frac{c_2}{d}\bx_i^\top \bG \bx_i$.
And from Lemma~\ref{lem:quadratic_moments_eq2}, we know
\begin{align}
   \E_{\bx,\bG} \|\vf_*\|^2&\lesssim n(c_0^2  +c_1^2  \bbeta^\top\bSigma\bbeta +\frac{c_2^2}{d^2} (2\E_{\bG}\Tr[(\bG\bSigma)^2] +\E_{\bG}[(\Tr(\bG\bSigma))^2])) \\
   &\lesssim n(c_0^2+c_1^2 +\frac{c_2^2 }{d^2} \cdot d^2)=O(n).
\end{align}
Then, by Markov's inequality, with high probability, 
$\norm{\vf_* }^2 = O(n\cdot d^{\frac{1}{24}})$.
Therefore, with high probability,
$\frac{1}{n}\norm{\by}^2 d^{-\frac{1}{24}}=o(1)$. 
\end{proof}
With Lemma~\ref{lemm:approx_train} and Lemma~\ref{lemm:bound_y}, we obtain  with high probability,
\begin{align}\label{eq:Etrain_equivalence}
       |\Etrain -\Etrain^{(2)}|=O(d^{-\frac{1}{24}}).
\end{align}
Let $\bg\in \mathbb R^{\binom{d+1}{2}}$ such that for $i\leq j$, $\bg_{ii}=\bG_{ii}, \bg_{ij}=\bG_{ij}$.
 With our definition of $\bx^{(2)}$ in \eqref{eq:def_reduced},  
 
 \begin{align}
 \bx^\top\bG\bx=2\sum_{i<j} \bG_{ij}\bx_i\bx_j +\sum_{i}\bG_{ii}\bx_i^2 &=\sqrt{2}\sum_{i<j}\bg_{ij} \bx^{(2)}(i,j) + \sum_{i} \bg_{ii} \bx^{(2)}(i,i)\\
 &=\sqrt{2}\langle\bx^{(2)},\bg\rangle-(\sqrt{2}-1)\sum_{i=1}^d \bg_{ii} \bx^{(2)}(i,i).\label{eq:extra_v}
 \end{align}
 
From the teacher model defined in \eqref{eq:teacher}, the training labels can be represented by $\by = \bu+\bepsilon\in\R^n$,
where, within the proof, we temporarily denote
    \begin{align}\label{eq:defbv}
    \bu:&=c_0\1_n+c_1 \bX\bbeta+\frac{\sqrt{2} c_2}{d}\bX^{(2)}\bg-\bv,
    \end{align}
    where from \eqref{eq:extra_v}, we have
    \begin{align}
        \bv_i=\frac{(\sqrt 2-1)c_2}{d}\sum_{j}\bg_{jj} \bx_i^{(2)}(j,j).\label{eq:defv_i}
    \end{align}
Then \eqref{eq:Etrain2} can be written as
\begin{align}\label{eq:Etrain_2_decompose}
    \Etrain^{(2)}=&\frac{\lambda^2}{n}\big[\bu^\top (\bK^{(2)} +\lambda\bI_n)^{-2}\bu+\bepsilon^\top(\bK^{(2)} +\lambda\bI_n)^{-2}\bepsilon+2\bepsilon^\top(\bK^{(2)} +\lambda\bI_n)^{-2}\bu \big]. 
\end{align}

\begin{lemma}\label{lemm:inverseK_mix_bound}
 We have deterministically,
     \begin{align}
      \left\|(\bK^{(2)} +\lambda\bI_n)^{-1/2}\1_n\1_n^\top (\bK^{(2)} +\lambda\bI_n)^{-1/2}\right\|&\leq \frac{1}{a_0}=O(1), \label{eq:inverseKX1}\\
        \left\|(\bK^{(2)} +\lambda\bI_n)^{-1/2}\bX\bX^\top(\bK^{(2)} +\lambda\bI_n)^{-1/2}\right\|&\leq \frac{1}{a_1} =O(d), \label{eq:inverseKX}\\
        \left\|(\bK^{(2)} +\lambda\bI_n)^{-1/2}\bX^{(2)}{\bX^{(2)}}^\top (\bK^{(2)} +\lambda\bI_n)^{-1/2}\right\|&\leq \frac{1}{a_2}=O(d^2) \label{eq:inverseKX2}.
     \end{align}
     Similarly, with probability $1-O(d^{-1/2})$,
          \begin{align}
      \left\|(\bK +\lambda\bI_n)^{-1/2}\1_n\1_n^\top (\bK +\lambda\bI_n)^{-1/2}\right\|&\leq \frac{1}{a_0}=O(1),  \\
        \left\|(\bK +\lambda\bI_n)^{-1/2}\bX\bX^\top(\bK  +\lambda\bI_n)^{-1/2}\right\|&\leq \frac{1}{a_1} =O(d),  \\
        \left\|(\bK +\lambda\bI_n)^{-1/2}\bX^{(2)}{\bX^{(2)}}^\top (\bK +\lambda\bI_n)^{-1/2}\right\|&\leq \frac{1}{a_2}=O(d^2)  .
     \end{align}
\end{lemma}
\begin{proof}
 Since for sufficiently large $d$, $a_0, a_1, a_2, a>0$, we have 
     \begin{align}
     a_0 \1_n\1_n^\top \preccurlyeq \bK^{(2)}+\lambda \bI_n, \quad 
          a_1 \bX\bX^\top \preccurlyeq \bK^{(2)}+\lambda \bI_n, \quad 
           a_2 \bX^{(2)}{\bX^{(2)}}^\top &\preccurlyeq \bK^{(2)}+\lambda \bI_n.
     \end{align}
          Hence,
     \begin{align}
      \left\|(\bK^{(2)} +\lambda\bI_n)^{-1/2}\1_n\1_n^\top (\bK^{(2)} +\lambda\bI_n)^{-1/2}\right\|&\leq \frac{1}{a_0}=O(1),  \\
        \left\|(\bK^{(2)} +\lambda\bI_n)^{-1/2}\bX\bX^\top(\bK^{(2)} +\lambda\bI_n)^{-1/2}\right\|&\leq \frac{1}{a_1} =O(d),  \\
        \left\|(\bK^{(2)} +\lambda\bI_n)^{-1/2}\bX^{(2)}{\bX^{(2)}}^\top (\bK^{(2)} +\lambda\bI_n)^{-1/2}\right\|&\leq \frac{1}{a_2}=O(d^2)  .
     \end{align}
     For the results of $\bK$, we can directly apply Theorem~\ref{thm:concentration} and \eqref{eq:K_lambda_-1}.
\end{proof}

\subsection{Precise asymptotics of training error}

We calculate the asymptotic value of $\Etrain^{(2)}$ by proving the following three lemmas.
\begin{lemma}\label{lem:train1}
Under the same assumptions as Theorem~\ref{thm:train_limit},  we have as $n,d\to\infty$ and $d^2/(2n)\to \alpha$,    in probability,
    $ \frac{1}{n}\bu^\top (\bK^{(2)} +\lambda\bI_n)^{-2}\bu \to \int \frac{\frac{c_2^2}{\alpha} x}{\left(\frac{f''(0)}{4\alpha}x+a_*+\lambda\right)^2}~d\mu_{\alpha,\bSigma^{(2)}}(x)$.
\end{lemma}
\begin{proof}
 Recall the definition of $\bv$ from \eqref{eq:defbv}.
 Let $\bu=\bu_1+\bu_2$ where 
    \[ \bu_1=c_0\1_n+c_1 \bX\bbeta, \quad \bu_2=\frac{\sqrt{2}c_2}{d}\bX^{(2)}\bg-\bv.\]

    Denote $\bK_{\lambda}^{(2)}=\bK^{(2)} +\lambda\bI_n$. We have the following decomposition:
       \begin{align}
          \bu^\top(\bK^{(2)} +\lambda\bI_n)^{-2}\bu&=\bu_2^\top \left( \bK_{\lambda}^{(2)}\right)^{-2}\bu_2+ \bu_1^\top \left( \bK_{\lambda}^{(2)}\right)^{-2}\bu_1+ 2\bu_1^\top \left( \bK_{\lambda}^{(2)}\right)^{-2}\bu_2 \\
           &=:S_2+S_1+ S_3, \label{eq:S1S2S3}
       \end{align}
       where, by Cauchy's inequality, we have
       \begin{align} \label{eq:S3cauchy}
       S_3:=2\bu_1^\top \left( \bK_{\lambda}^{(2)}\right)^{-2}\bu_2 \leq 2\sqrt{S_1S_2}.
       \end{align}

          \paragraph{Step 1: Computing $S_2$.} 
We first estimate $\|\bv\|$. From \eqref{eq:defv_i},
\begin{align}
    \E_{\bx_i}\E_{\bG}[\bv_i^{8}]&\lesssim \frac{1}{d^4}\E_{\bx_i}\left(d^{-1} \sum_{j\in [d]}\bx_i(j)^4\right)^4 \lesssim d^{-4} \E_{\bx_i}\left( d^{-1} \sum_{j} \bx_i(j)^{16}\right)\lesssim d^{-4},
\end{align}
where the last line is due to Jensen's inequality. Therefore with probability at least $1-d^{-3}$, $|\bv_i|\leq d^{-1/8}$. Taking a union bound over $i\in [n]$, we have with probability at least $1-d^{-1}$, 
\begin{align}\label{eq:norm_bv}
    \norm{\bv}=O(d^{7/8}).
\end{align}
We can decompose $S_2$ as 
\begin{align}
    S_2&=S_2'+\bv^\top\left( \bK_{\lambda}^{(2)}\right)^{-2}\bv -2\bv^\top\left( \bK_{\lambda}^{(2)}\right)^{-2}\frac{\sqrt 2 c_2}{d}\bX^{(2)}\bg, \label{eq:S2_3term}
\end{align}
where
    $S_2'=\bg^\top \left(\frac{2c_2^2}{d^2}  {\bX^{(2)}}^\top (\bK^{(2)} +\lambda\bI_n)^{-2} \bX^{(2)}\right)\bg$,
and
     \begin{align}
       \E_{\bg}[S_2']=\frac{2c_2^2}{d^2}\Tr  \left[ (\bK^{(2)} +\lambda\bI_n)^{-2}\bX^{(2)}{\bX^{(2)}}^\top\right].
     \end{align}
With \eqref{eq:inverseKX2}, we can apply Hanson-Wright inequality \citep{vershynin2018high} to obtain  
  \[  \frac{1}{n} S_2'-\frac{1}{n}\cdot \frac{2c_2^2}{d^2}\Tr  \left[ (\bK^{(2)} +\lambda\bI_n)^{-2}\bX^{(2)}{\bX^{(2)}}^\top\right]\to 0\] with high probability. From the  limiting spectral distribution of $\frac{4\alpha}{f''(0)}(\bK^{(2)}-a \bI)$ shown in Theorem \ref{thm:globallaw}, we have the following convergence in probability holds: 
\begin{align}
    \frac{1}{n}\cdot \frac{2c_2^2}{d^2a_2}\Tr  \left[ (\bK^{(2)} +\lambda\bI_n)^{-2}(\bK^{(2)}-a\bI_n)\right] \to \int \frac{\frac{c_2^2}{\alpha} x}{\left( \frac{f''(0)x}{4\alpha}+a_*+\lambda \right)^2}~d\mu_{\alpha,\bSigma^{(2)}}(x).
\end{align}
Moreover, due to \eqref{eq:inverseKX} and \eqref{eq:inverseKX1},
\begin{align}
&\frac{1}{n}\cdot \frac{2c_2^2}{d^2 a_2}\left[ (\bK^{(2)} +\lambda\bI_n)^{-2}(\bK^{(2)}-a\bI_n)\right]-\frac{2c_2^2}{d^2}\Tr  \left[ (\bK^{(2)} +\lambda\bI_n)^{-2}\bX^{(2)}{\bX^{(2)}}^\top\right]\\
=&\frac{1}{n}\cdot\frac{2c_2^2}{d^2}\Tr\left[(\bK^{(2)}_{\lambda})^{-2}\left(\frac{a_0}{a_2} \1\1^\top +\frac{a_1}{a_2}\bX\bX^\top \right)\right]=o(1).
\end{align}
Therefore, 
\begin{align}\label{eq:vKv}
     \frac{1}{n} S_2'\to \int \frac{\frac{c_2^2}{\alpha} x}{\left( \frac{f''(0)x}{4\alpha}+a_*+\lambda \right)^2}~d\mu_{\alpha,\bSigma^{(2)}}(x)
\end{align}
 in probability. With \eqref{eq:norm_bv}, we have with high probability,
\begin{align}
   \frac{1}{n} \bv^\top\left( \bK_{\lambda}^{(2)}\right)^{-2}\bv=O(d^{-1/4}), \quad 
2\bv^\top\left( \bK_{\lambda}^{(2)}\right)^{-2}\frac{\sqrt 2 c_2}{d}\bX^{(2)}\bg=O(d^{-1/8}),   
\end{align}
where we use Cauchy's inequality and \eqref{eq:vKv}. Then from \eqref{eq:S2_3term}, we have in probability,
\begin{align}\label{eq:vKv2}
     \frac{1}{n} S_2\to \int \frac{\frac{c_2^2}{\alpha} x}{\left( \frac{f''(0)x}{4\alpha}+a_*+\lambda \right)^2}~d\mu_{\alpha,\bSigma^{(2)}}(x).
\end{align}

\paragraph{Step 2: Controlling $S_1$.} By Cauchy's inequality, we have
\begin{align}
    \frac{1}{n} S_1\leq \frac{2 c_0^2}{n} \1_n^\top (\bK^{(2)} +\lambda\bI_n)^{-2}\1_n + \frac{2c_1^2}{n}\bbeta^\top \bX^\top (\bK^{(2)} +\lambda\bI_n)^{-2} \bX \bbeta.
\end{align}
For the first term on the right-hand side, we have
\begin{align}
   &\frac{c_0^2}{n}\1_n^\top (\bK^{(2)} +\lambda\bI_n)^{-2}\1_n=\frac{c_0^2}{n} \Tr[(\bK^{(2)} +\lambda\bI_n)^{-2}\1_n\1_n^\top]\\
   =~&\frac{c_0^2}{n} \| (\bK^{(2)} +\lambda\bI_n)^{-1}(\bK^{(2)} +\lambda\bI_n)^{-1/2}\1_n\1_n^\top(\bK^{(2)} +\lambda\bI_n)^{-1/2}\|\\
    \leq~ &\frac{2c_0^2}{a_*n}\| (\bK^{(2)} +\lambda\bI_n)^{-1/2}\1_n\1_n^\top(\bK^{(2)} +\lambda\bI_n)^{-1/2}\|\leq \frac{2c_0^2}{a_*a_0n}=O(n^{-1}),
\end{align}
where in the first identity, we use the fact $\1_n\1_n^\top$ is rank-1, and the last inequality is due to \eqref{eq:inverseKX1}. For the second term, we have 
\begin{align}
    \frac{2c_1^2}{n}\bbeta^\top \bX^\top (\bK^{(2)} +\lambda\bI_n)^{-2} \bX \bbeta &\lesssim \frac{1}{n} \|  (\bK^{(2)} +\lambda\bI_n)^{-1} \bX \|^2\\
    &\leq \frac{1}{na_*} \left\| (\bK^{(2)} +\lambda\bI_n)^{-1/2}\bX\right\|^2=O(d/n),
\end{align}
where the last inequality is due to \eqref{eq:inverseKX}. Therefore $\frac{1}{n}{S_1}=o(1)$ with high probability. Combining the estimates of $S_1, S_2$, Lemma~\ref{lem:train1} holds due to \eqref{eq:vKv2}, \eqref{eq:S1S2S3}, and \eqref{eq:S3cauchy}.
\end{proof}

\begin{lemma}\label{lem:train2}
Under the same assumptions as Theorem~\ref{thm:train_limit},  the following  holds with high probability:
 $ \left| \frac{1}{n}\bepsilon^\top(\bK^{(2)} +\lambda\bI_n)^{-2}\bepsilon- \frac{\sigma_{\bepsilon}^2}{n}\Tr(\bK^{(2)} +\lambda\bI_n)^{-2}\right|=o(1)$.
And   in probability,
  \begin{align}\label{eq:variance_limit}
      \frac{\sigma_{\bepsilon}^2}{n}\Tr(\bK^{(2)} +\lambda\bI_n)^{-2}\to  \int \frac{\sigma_{\bepsilon}^2}{\left(\frac{f''(0)}{4\alpha}x+a_*+\lambda\right)^2}~d\mu_{\alpha,\bSigma^{(2)}}(x).
  \end{align}

\end{lemma}

\begin{proof}
    The first claim follows from Hanson-Wright inequality for sub-Gaussian random vectors in \citep{rudelson2013hanson} since $\bepsilon$ is sub-Gaussian and \eqref{eq:K_lambda_-1} holds with high probability. From Theorem~\ref{thm:globallaw},
    the empirical spectral distribution of $\frac{4\alpha}{f''(0)}(\bK^{(2)}-a\bI_n)$ converges to $\mu_{\alpha,\bSigma^{(2)}}$.  Take a test function $\frac{1}{(x+a_*+\lambda)^2}$ which is bounded continuous on interval $[-a_*/2,\infty)$. From Lemma~\ref{lem:asym_sigma_min}, for sufficiently large $n$, $\lambda_{\min} (\bK^{(2)}-a\bI_n)\geq -\frac{a_*}{2}$. Therefore, \eqref{eq:variance_limit} holds from weak convergence.
\end{proof}

\begin{lemma}\label{lem:train3}
Under the same assumptions as Theorem~\ref{thm:train_limit}, with high probability, 
    $$\frac{1}{n}\bepsilon^\top(\bK^{(2)} +\lambda\bI_n)^{-2}\bu =o(1).$$
\end{lemma}
\begin{proof}
    We do a second-moment estimate. Note that
  $$
        \E_{\bepsilon} \left(\bepsilon^\top(\bK^{(2)} +\lambda\bI_n)^{-2}\bu\right)^2=\sigma_{\bepsilon}^2\bu^\top (\bK^{(2)} +\lambda\bI_n)^{-4}\bu.$$
Applying the same proof as in Lemma~\ref{lem:train1}, one can show that 
$\frac{\sigma_{\bepsilon}^2}{n}\bu^\top (\bK^{(2)} +\lambda\bI_n)^{-4}\bu$  
converges in probability to a deterministic limit. Therefore, with high probability, we have 
$\E_{\bepsilon} \left(\bepsilon^\top(\bK^{(2)} +\lambda\bI_n)^{-2}\bu\right)^2=O(n)$.
Hence, Lemma~\ref{lem:train3} holds by Markov's inequality.
\end{proof}

\begin{proofoftheorem}{\ref{thm:train_limit}}
 From \eqref{eq:Etrain_equivalence}, it suffices to analyze the asymptotic behavior of $\Etrain^{(2)}$. Therefore, from the decomposition of $\Etrain^{(2)}$ in \eqref{eq:Etrain_2_decompose}, with Lemmas \ref{lem:train1}, \ref{lem:train2}, and \ref{lem:train3}, we have 
 $   \Etrain\to  \lambda^2\int \frac{\frac{c_2^2}{\alpha} x+\sigma_{\bepsilon}^2}{\left(\frac{f''(0)}{4\alpha}x+a_*+\lambda\right)^2}~d\mu_{\alpha,\bSigma^{(2)}}(x) $
in probability. This finishes the proof.
\end{proofoftheorem}


\section{The analysis of generalization errors} \label{sec:krr_test}

\subsection{Preliminary calculations}\label{sec:appendix_prelim}

\subsubsection{Concentration of random quadratic forms}
The following lemma improves the second moment estimate in \eqref{eq:BZ_tensor}.

\begin{lemma}\label{lem:quad_high_power} 
Assume $\bx=\bSigma^{1/2} \bz\in\R^d$, and $\bSigma$ is diagonal and bounded in operator norm. $\bz$ has independent entries with 1st, 3rd, and 5th moments zero, and each entry has finite first 56-th moments.   We have for any deterministic matrix $\bA\in \mathbb R^{\binom{d+1}{2}\times \binom{d+1}{2}}$ with $\|\bA\|\leq 1$,
            \begin{align} \label{eq:BZ_tensor6}
    \E\left| {\overline\bx^{(2)}}^\top \bA\overline\bx^{(2)} -\Tr[\bA \bSigma^{(2)}]\right|^{14}=O(d^{25.5}).
     \end{align}
And under the Assumption~\ref{assump:data} for $\bX$, for all $i\in [n]$, with probability at least $1-O(d^{-\frac{1}{5}})$,
     \begin{align} \label{eq:unionBZ}
     \frac{1}{n}   \left| {\overline\bx_i^{(2)}}^\top \bA\overline\bx_i^{(2)} -\Tr[\bA \bSigma^{(2)}]\right|=O(n^{-\frac{1}{60}}).
     \end{align}
\end{lemma}
\begin{proof}
We first focus on proving \eqref{eq:BZ_tensor6}. For ease of notation, in this proof, we denote $\bx_i$ as the $i$-th entry of $\bx\in\R^d$ for $i\in [d]$.  
We decompose $\bA=\bD+\bB$, where $\bD$ is the diagonal part of $\bA$ and $\bB$ is the off-diagonal part of $\bA$, and compute their contribution below.

    \textbf{ (i) Diagonal part.} Following the same argument as in the proof of Lemma~\ref{lem:MP_tensor}, recall the definition of $\overline{\bx}^{(2)}$ from \eqref{eq:centered_x2}, we have 
    \begin{align}
    &\E\left| {\overline\bx^{(2)}}^\top \bD\overline\bx^{(2)}-\Tr[\bD\bSigma^{(2)}]\right|^{14}\\
    =&\E \left( \sum_{i<j}2(\bx_i^2 \bx_j^2-\bSigma_{ij,ij}^{(2)})\bA_{ij,ij} +\sum_{i} ((\bx_i^2-\bSigma_{ii})^2-\bSigma_{ii,ii}^{(2)}) \bA_{ii,ii}\right)^{14}\\
    \lesssim &\E \left( \sum_{i<j}(\bx_i^2 \bx_j^2-\bSigma_{ij,ij}^{(2)})\bA_{ij,ij}\right)^{14}+\E \left( \sum_{i} ((\bx_i^2-\bSigma_{ii})^2-\bSigma_{ii,ii}^{(2)}) \bA_{ii,ii}\right)^{14}. \label{eq:diagonal_6}
    \end{align}
    For the second term in \eqref{eq:diagonal_6}, by independence of entries in $\bx$, its contribution is $O(d^{14})$. 
        We now expand the first term in \eqref{eq:diagonal_6}, which gives  
        \begin{small}
    \begin{align}
        \sum_{i_1<j_1,\cdots, i_{14}<j_{14}}\bA_{i_1j_1,i_1j_1}\cdots \bA_{i_{14}j_{14},i_{14},j_{14}}\E\left[(\x_{i_1}^2\x_{j_1}^2-\bSigma_{i_1j_1,i_1j_1}^{(2)})\cdots (\x_{i_{14}}^2\x_{j_{14}}^2-\bSigma_{i_{14}j_{14},i_{14}j_{14}}^{(2)})\right].~~~~ \label{eq:14terms}
    \end{align}
    \end{small}
Since each product in the expectation is centered, to have a nonzero expectation in \eqref{eq:14terms}, each pair in $\{i_1,j_1\},\cdots \{i_{14},j_{14}\}$ must have at least one index with multiplicity at least 2. 
We now divide 14 pairs $\{i_1,j_1\},\cdots \{i_{14},j_{14}\}$ into 7  groups of 4 indices given by  \[\{i_1,j_1,i_2,j_2\},\dots, \{i_{13},j_{13},i_{14},j_{14}\}.\] 
To have zero expectation in \eqref{eq:14terms}, we claim there are at most 21 distinct indices in $i_1,j_1\dots,i_{14},j_{14}$. Otherwise, at least one group of indices only appears once. This gives zero expectation in \eqref{eq:14terms}, a contradiction. Hence, in \eqref{eq:14terms}, the total contribution is $O(d^{21})$.
Combining the two terms in \eqref{eq:diagonal_6}, the total contribution is $O(d^{21})$.

    \textbf{ (ii) Off-diagonal part.} 
Now we do the following expansion:
    \begin{align}\label{eq:quadratic14}
      & \E\left| {\overline\bx^{(2)}}^\top \bB\overline\bx^{(2)}-\Tr[\bB\bSigma^{(2)}]\right|^{14} 
    =~\E \left( \sum_{(i_1,i_2)\not= (i_3,i_4)}\bA_{i_1i_2,i_3i_4} \overline\bx^{(2)}_{i_1i_2}\overline\bx^{(2)}_{i_3i_4}\right)^{14}\\
    =& \sum_{(i_1,i_2)\not=(i_3,i_4),\cdots, (i_{53},i_{54})\not= (i_{55},i_{56})}\bA_{i_1i_2,i_3i_4}\cdots \bA_{i_{53}i_{54},i_{55}i_{56}} \E\left[ \overline\bx^{(2)}_{i_1i_2}\overline\bx^{(2)}_{i_3i_4}\cdots \overline\bx^{(2)}_{i_{53}i_{54}}\overline\bx^{(2)}_{i_{55}i_{56}} \right]\\
    \leq &\sum_{(i_1,i_2)\not=(i_3,i_4),\cdots, (i_{53},i_{54})\not= (i_{55},i_{56})} |\bA_{i_1i_2,i_3i_4}\cdots \bA_{i_{53}i_{54},i_{55}i_{56}}| \left| \E\left[ \overline\bx^{(2)}_{i_1i_2}\overline\bx^{(2)}_{i_3i_4}\cdots \overline\bx^{(2)}_{i_{53}i_{54}}\overline\bx^{(2)}_{i_{55}i_{56}} \right]\right| . 
    \end{align}
    And 
    \begin{align}\label{eq:product24}
 & \E\left[ \overline\bx^{(2)}_{i_1i_2}\overline\bx^{(2)}_{i_3i_4}\cdots \overline\bx^{(2)}_{i_{53}i_{54}}\overline\bx^{(2)}_{i_{55}i_{56}} \right]\\
  =&~\E\left[(\bx_{i_1}\bx_{i_2}-\bSigma_{i_1,i_2}\delta_{i_1,i_2})(\bx_{i_3}\bx_{i_4}-\bSigma_{i_3,i_4}\delta_{i_3,i_4})\cdots  (\bx_{i_{55}}\bx_{i_{56}}-\bSigma_{i_{55},i_{56}}\delta_{i_{55},i_{56}})\right], 
    \end{align}
    with the restriction that  
    \begin{align}\label{eq:restriction}
    i_1\leq i_2,\dots, i_{55}\leq i_{56}, \quad (i_1,i_2)\not=(i_3,i_4),\cdots, (i_{53},i_{54})\not= (i_{55},i_{56}).
\end{align}    
We estimate \eqref{eq:quadratic14} with the following three steps.

\textbf{Step 1: Preliminary estimates.} 
Suppose $i_1,i_2,i_3,i_4$ are 4 distinct indices, then by Cauchy's inequality and the fact that $\|\bA\|_{\msf F}\leq \sqrt{\binom{d+1}{2}}\|\bA\|_{\msf F}\leq d$,
\begin{align}\label{eq:4distinct}
    \sum_{i_1,i_2,i_3,i_4\in [d],  \text{  4 distinct indices}}|\bA_{i_1i_2,i_3i_4}|\leq \sqrt{d^4 \|\bA\|_{\msf F}^2}\leq d^3.
\end{align}
Similarly, if there are at most 3 distinct indices among $i_1,i_2,i_3,i_4\in [d]$, we have 
\begin{align}\label{eq:3distinct}
  \sum_{i_1,i_2,i_3,i_4\in [d], \text{ 3 distinct indices}}|\bA_{i_1i_2,i_3i_4}|\leq \sqrt{d^3 \|\bA\|_{\msf F}^2}\leq d^{2.5}. 
\end{align}
If there are two distinct indices, due to the restriction \eqref{eq:restriction}, the entries must be $A_{i_1i_1,i_2i_2}$ with $i_1\not=i_2$, and  we have from Cauchy's inequality,
\begin{align}\label{eq:2distinct}
\sum_{i_1,i_2}|\bA_{i_1i_1,i_2i_2}|\leq \sqrt{d^2\|\bA_S\|_{\msf F}^2}\leq d^{1.5},
\end{align}
where $\bA_S$ is a $d\times d$ submatrix of $\bA$ given by $\bA_{i_1i_1,i_2i_2}$ and we use the fact that  $\|\bA_S\|_{\msf F}\leq \sqrt{d}\|\bA\|\leq \sqrt{d}$. We also have the following trivial bound for all $i_1,i_2,i_3,i_4\in [d]$:
\begin{align}\label{eq:trivial}
    |\bA_{i_1i_2,i_3i_4}|\leq \|\bA\|\leq 1.
\end{align}
By the independence of entries in $\bx$, to have a nonzero expectation in \eqref{eq:product24}, there are at most 28 distinct indices in $i_1,\dots, i_{56}$. On the other hand, if there are at most $25$ distinct indices, the total contribution for those terms is at most $O(d^{25})$. Therefore, to show \eqref{eq:BZ_tensor6}, we only need to consider  $(i_1,\dots,i_{56})$ where there are $26,27$ or $28$ many distinct indices.

We group the 56 indices into 14 tuples: $(i_{4k-3},i_{4k-2},i_{4k-1},i_{4k})$ for $1\leq k\leq 14$.  To have a nonzero zero expectation in \eqref{eq:product24}, with the restriction  from \eqref{eq:restriction}, there are at least 2 distinct indices in each tuple $(i_{4k-3},i_{4k-2},i_{4k-1},i_{4k})$ for $1\leq k\leq 14$. 
Among the 14 tuples, we define a subset called \textit{good tuples} recursively. The first good tuple is $(i_1,i_2,i_3,i_4)$. If there are $s$ many distinct indices in $(i_1,i_2,i_3,i_4)$ for $s=2,3,4$, we call $(i_1,i_2,i_3,i_4)$ a \textit{good $s$-tuple}.  According to the lexicographic order, the next tuple that does not share any common indices with previous good tuples is also a good $s$-tuple if it has $s$ distinct indices.

\textbf{Step 2: An algorithm to bound \eqref{eq:quadratic14}.} We now describe an algorithm to provide a bound on \eqref{eq:quadratic14} with the following steps to bound the contribution from each tuple. The strategy is to use the better bounds \eqref{eq:4distinct}, \eqref{eq:3distinct}, and \eqref{eq:2distinct} as many times as possible. 
\begin{itemize}
    \item Start with the first good tuple $(i_1,i_2,i_3,i_4)$. Track all the tuples which coincide with at least one index in $(i_1,i_2,i_3,i_4)$. 
 Bound the contribution from all tuples which shared at least one indices with $(i_1,i_2,i_3,i_4)$ in \eqref{eq:quadratic14} using \eqref{eq:trivial} and bound the contribution of $(i_1,i_2,i_3,i_4)$ using \eqref{eq:4distinct}, \eqref{eq:3distinct}, or \eqref{eq:2distinct} depending on the number of distinct indices $s$. Without loss of generality, we may assume the second to the $(s+1)$-th tuples in lexicographical order share indices with the first tuple.
 See Figure~\ref{fig:good_3} for an example when $(i_1,i_2,i_3,i_4)$ is a good 3-tuple. In the case of Figure~\ref{fig:good_3}, We can bound \begin{align}
\sum_{i_1,i_2,\dots,i_{10}}|\bA_{i_1i_2,i_3i_4}\bA_{i_5i_6,i_7i_8}\bA_{i_9i_{10},i_{11}i_{12}}|
\leq & d^{2.5} \left(\sum_{i_6,i_7,i_8,i_9,i_{11},i_{12}} 1\right).
\end{align}
by using \eqref{eq:3distinct}, which reduces the sum of 10 indices to a sum of 6 indices.

\begin{figure}
\centering
    \begin{tikzpicture}[scale=0.7, transform shape]
  \node (1) at (0,0) {1};
  \node (2) at (1,0) {2};
  \node (3) at (2,0) {3};
  \node (4) at (3,0) {4};
  \node (5) at (5,0) {5};
  \node (6) at (6,0) {6};
  \node (7) at (7,0) {7};
  \node (8) at (8,0) {8};
  \node (9) at (10,0) {9};
  \node (10) at (11,0) {10};
  \node (11) at (12,0) {11};
  \node (12) at (13,0) {12};
  
  \draw  (-0.5,0.5) rectangle (3.5,-0.5);
  \draw  (4.5,0.5) rectangle (8.5,-0.5);
  \draw  (9.5,0.5) rectangle (13.5,-0.5);
  

 \draw (1) to[bend left] (2);
  \draw (3) to[bend left] (5);
  \draw (4) to[bend left] (10);
\end{tikzpicture}
\caption{In this example, the tuple $(i_1,i_2,i_3,i_4)$ share common indices with two tuples $(i_5,i_6,i_7,i_8)$ and $(i_9,i_{10},i_{11},i_{12})$ by identifying $i_1=i_2, i_3=i_{5},i_4=i_{10}$. The relations among $i_6,i_7,i_8,i_9,i_{11},i_{12}$ are not specified.}  
\label{fig:good_3}
\end{figure}

    \item Find the next good tuple in the lexicographical order denoted by  $$(i_{4k-3},i_{4k-2},i_{4k-1},i_{4k}),$$ bound its contribution depending on the number of distinct indices $s$ in the tuple.  Repeat this process until no more good tuples can be found. 
\item For all the remaining indices that have not been summed using \eqref{eq:4distinct}, \eqref{eq:3distinct}, or \eqref{eq:2distinct}, let $k$ be 
the number of distinct indices in the remaining indices and bound their contribution by $d^{k}$.
\end{itemize}

\textbf{Step 3: Applying the algorithm in 3 cases.}
(a) \textbf{Case 1}: For the contribution in \eqref{eq:quadratic14} with exactly 28 distinct indices in the sum, each is repeated exactly twice.
In this case, there are no good 2-tuples. To see that, suppose there exists one good 2-tuple $(i_{4k-3},i_{4k-2},i_{4k-1},i_{4k})$ with $i_{4k-3}=i_{4k-1}, i_{4k-2}=i_{4k}$ and $i_{4k-3}\not=i_{4k-2}$. Then no other tuples will share the same index with $(i_{4k-3},i_{4k-2},i_{4k-1},i_{4k})$. By independence of entries in $\bx$, this implies the contribution in \eqref{eq:product24} is zero. So below, we only need to consider sequences with good 3-tuples and 4-tuples.  By applying the algorithm we described above, there are several cases:
\begin{itemize}
    \item Suppose all the good tuples are 3-tuples. We explain this case in more detail, and other cases below follow similarly. 
    
    Since each good 3-tuple has shared indices with at most 2 tuples, among 14 tuples, there are at least 5 good 3-tuples. We may assume the 5 good 3-tuples are 
    \begin{align}\label{eq:5good}
        (i_1,i_2,i_3,i_4), (i_{13},i_{14},i_{15},i_{16}), (i_{25},i_{26},i_{27}, i_{28}), (i_{37}, i_{38}, i_{39},i_{40}), (i_{49},i_{50}, i_{51}, i_{52}).
    \end{align}

\begin{figure}
\centering
 \begin{tikzpicture}[scale=0.7, transform shape]
        \node (1) at (0,0) {1};
        \node (2) at (1,0) {2};
        \node (3) at (2,0) {3};
        \node (4) at (3,0) {4};
        \node (5) at (5,0) {5};
        \node (6) at (6,0) {6};
        \node (7) at (7,0) {7};
        \node (8) at (8,0) {8};
        \node (9) at (10,0) {9};
        \node (10) at (11,0) {10};
        \node (11) at (12,0) {11};
        \node (12) at (13,0) {12};
        
        \draw (-0.5,0.5) rectangle (3.5,-0.5);
        \draw (4.5,0.5) rectangle (8.5,-0.5);
        \draw (9.5,0.5) rectangle (13.5,-0.5);
        
        \draw (1) to[bend left] (2);
        \draw (3) to[bend left] (5);
        \draw (4) to[bend left] (10);

        \node (13) at (0,-2) {13};
        \node (14) at (1,-2) {14};
        \node (15) at (2,-2) {15};
        \node (16) at (3,-2) {16};
        \node (17) at (5,-2) {17};
        \node (18) at (6,-2) {18};
        \node (19) at (7,-2) {19};
        \node (20) at (8,-2) {20};
        \node (21) at (10,-2) {21};
        \node (22) at (11,-2) {22};
        \node (23) at (12,-2) {23};
        \node (24) at (13,-2) {24};
        
        \draw (-0.5,-1.5) rectangle (3.5,-2.5);
        \draw(4.5,-1.5) rectangle (8.5,-2.5);
        \draw  (9.5,-1.5) rectangle (13.5,-2.5);
        
        \draw (13) to[bend left] (14);
        \draw (15) to[bend left] (17);
        \draw (16) to[bend left] (22);
        
        \node (25) at (0,-4) {25};
        \node (26) at (1,-4) {26};
        \node (27) at (2,-4) {27};
        \node (28) at (3,-4) {28};
        \node (29) at (5,-4) {29};
        \node (30) at (6,-4) {30};
        \node (31) at (7,-4) {31};
        \node (32) at (8,-4) {32};
        \node (33) at (10,-4) {33};
        \node (34) at (11,-4) {34};
        \node (35) at (12,-4) {35};
        \node (36) at (13,-4) {36};
        
        \draw  (-0.5,-3.5) rectangle (3.5,-4.5);
        \draw  (4.5,-3.5) rectangle (8.5,-4.5);
        \draw  (9.5,-3.5) rectangle (13.5,-4.5);
        
        \draw (25) to[bend left] (26);
        \draw (27) to[bend left] (29);
        \draw (28) to[bend left] (34);

        \node (37) at (0,-6) {37};
        \node (38) at (1,-6) {38};
        \node (39) at (2,-6) {39};
        \node (40) at (3,-6) {40};
        \node (41) at (5,-6) {41};
        \node (42) at (6,-6) {42};
        \node (43) at (7,-6) {43};
        \node (44) at (8,-6) {44};
        \node (45) at (10,-6) {45};
        \node (46) at (11,-6) {46};
        \node (47) at (12,-6) {47};
        \node (48) at (13,-6) {48};
        
        \draw  (-0.5,-5.5) rectangle (3.5,-6.5);
        \draw  (4.5,-5.5) rectangle (8.5,-6.5);
        \draw  (9.5,-5.5) rectangle (13.5,-6.5);
        
        \draw (37) to[bend left] (38);
        \draw (39) to[bend left] (41);
        \draw (40) to[bend left] (46);

        \node (49) at (0,-8) {49};
        \node (50) at (1,-8) {50};
        \node (51) at (2,-8) {51};
        \node (52) at (3,-8) {52};
        \node (53) at (5,-8) {53};
        \node (54) at (6,-8) {54};
        \node (55) at (7,-8) {55};
        \node (56) at (8,-8) {56};
        
        \draw (-0.5,-7.5) rectangle (3.5,-8.5);
        \draw(4.5,-7.5) rectangle (8.5,-8.5);
        
        \draw (49) to[bend left] (50);
        \draw (51) to[bend left] (53);
        \draw (52) to[bend left] (54);
    \end{tikzpicture}
    \caption{An example for the index sequences $(i_1,\dots,i_{56})$ with 5 good 3-tuples.  An edge between an index from a good tuple and another index outside good tuples is drawn if the two indices are identical.}
    \label{fig:good_3_tuples}
\end{figure}

There are 15 distinct indices in \eqref{eq:5good} by definition. See Figure~\ref{fig:good_3_tuples} for an example. Applying \eqref{eq:3distinct} to the 5 good 3-tuples, and \eqref{eq:trivial} for the rest of the tuples, we can bound the contribution of this case to \eqref{eq:quadratic14} by
\begin{align}
  &d^{12.5}  \sum_{i_6,i_7,i_8,i_9,i_{11},i_{12}} ~\sum_{i_{18},i_{19},i_{20},i_{21},i_{23},i_{24}}~\sum_{i_{30},i_{31},i_{32},i_{33},i_{35},i_{36}} ~\sum_{i_{42},i_{43},i_{44},i_{45},i_{47},i_{48}}~\left(\sum_{i_{55},i_{56}}  1\right)\\
 \leq & d^{12.5} \cdot d^{28-15}=d^{25.5},
\end{align}
where in the last inequality, we use the fact that there are at most $13$ distinct indices that do not share any indices in \eqref{eq:5good}, which gives the total contribution $O(d^{25.5})$.

    \item Among 14 tuples, there are at least 3 good 4-tuples, which gives a contribution of $d^{9}$ using \eqref{eq:4distinct}. And there are $28-12=16$ distinct indices remaining, which gives a contribution of $d^{16}$. In total, in this case, the contribution is $O(d^{25})$.
    \item There are at least  2 good 4-tuples which give a contribution of $d^6$, and 1 good 3-tuples, which give a contribution of $d^{2.5}$. So the total contribution is $O(d^{25.5})$.
    \item There are at least 1 good 4-tuples and 3 good 3-tuples. Similarly, the total contribution is $O(d^{3+7.5+(28-13)})=O(d^{25.5})$.
\end{itemize}
Therefore, from all the cases discussed above, the contribution for case (a) is bounded by $O(d^{25.5})$.

   (b) \textbf{Case 2}: For the contribution of \eqref{eq:quadratic14} with exactly 27 distinct indices in the sum. By counting the multiplicity,  we must have one index appearing 4 times (since the third moment of $\bx_i$ is zero), and the rest of the 26 indices appear twice. In this case, to have a non-zero expectation, there are no good 2-tuples in \eqref{eq:quadratic14}. Otherwise, there will be at least two indices appearing 4 times.

   Without loss of generality, we may assume the first tuple $(i_1,i_2,i_3,i_4)$ contains an index with multiplicity 4. There are at most 4 tuples containing this index, and we bound their contribution with \eqref{eq:trivial}. For the remaining 10 tuples, we apply the same argument as in Case (a). We have the following cases:
   \begin{itemize}
       \item 2 good 4-tuples. The total contribution is $O(d^{6+(27-8)})=O(d^{25})$.
       \item 1 good 4-tuple and 2 good 3-tuples, the total contribution is $O(d^{3+5+(27-10)})=O(d^{25})$.
       \item 4 good 3-tuples. The total contribution is $O(d^{10+(27-12)})=O(d^{25})$.
   \end{itemize}
   Therefore, all contribution for case (b) is $O(d^{25})$.

      (c) \textbf{Case 3}: For the contribution of \eqref{eq:quadratic14} with exactly 26 distinct indices in the sum. By counting the multiplicity, under the assumption that the 3rd and 5th moments of $\bx_i$ is zero, there are two cases: 
       \begin{itemize}
       \item Case (c.1): one index appears 6 times, and the rest of the indices appear twice. To have a nonzero expectation, there are no good 2-tuples.
       By a similar argument,  assuming the index with multiplicity 6 is among the first tuple $(i_1,i_2,i_3,i_4)$ and is repeated in the first 6 tuples, we can bound their contribution using \eqref{eq:trivial} and consider the remaining 8 tuples. For the remaining 8 tuples, we apply the same argument as in Case (a) in the following cases:
       \begin{itemize}
           \item 2 good 4 tuples: the contribution is $O(d^{6+26-8})=O(d^{24})$.
           \item 1 good 4-tuple and 1 good 3-tuple: the contribution is $O(d^{5.5+26-7})=O(d^{24.5})$.
           \item 3 good 3 tuples: the contribution is $O(d^{7.5+26-9})=O(d^{24.5})$.
       \end{itemize}
       
       The total contribution in this case is $O(d^{24.5})$.

       \item Case (c.2): 2 indices appear 4 times. And the other 24 indices appear twice. In this case, we have at most one good 2-tuple. 
       
       Case (c.2.1): If there exists one good 2-tuple, then the 2 indices appearing 4 times must be in the same tuple to make a nonzero expectation. Without loss of generality, we assume $(i_1,i_2,i_3,i_4)$ is a good 2-tuple, and it shares common indices with the next 4 tuples. We may bound the contribution from the first 5 tuples using \eqref{eq:2distinct} and \eqref{eq:trivial}, which gives a contribution of $O(d^{1.5})$. There are 9 tuples left, and we have the following cases:
       \begin{itemize}
           \item 2 good 4-tuples, the total contribution is $O(d^{1.5+6+24-10})=O(d^{21.5})$.
           \item 1 good 4-tuples and 2 good 3-tuples, the total contribution is $O(d^{21.5})$
           \item 3 good 3-tuples, the total contribution is $O(d^{1.5+7.5+(24-11)})=O(d^{22})$.
       \end{itemize}
        Case (c.2.2): Suppose there is no good 2-tuple.
       Without loss of generality, we can assume $(i_1,i_2,i_3,i_4)$ contains one index with multiplicity 4, with shared indices in the first 4 tuples. We can bound the contribution with \eqref{eq:trivial}. We can repeat this argument with the next 4 tuples: assume $(i_{17},i_{18},i_{19},i_{20})$ contains one index with multiplicity 4 with shared indices in the next 3 tuples. Now we consider the remaining 6 tuples. There are several cases:
      We  could have
       \begin{itemize}
           \item 2 good 4-tuples, the total contribution is $O(d^{6+24-8})=O(d^{22})$.
           \item 1 good 4-tuple and 1 good 3-tuple, the total contribution is $O(d^{22.5})$.
           \item 2 good 3-tuples with a total contribution $O(d^{5+24-6})=O(d^{23})$.
       \end{itemize}
      
   \end{itemize} 
   Combining cases (a), (b), and (c), \eqref{eq:BZ_tensor6} holds. By Markov's inequality and a union bound over $[n]$,  \eqref{eq:unionBZ} follows.
\end{proof}

\subsubsection{Deterministic equivalence of functions of the kernel}

Next,  we prove the following limits for the sample covariance matrix $\overline{\bX}^{(2)\top}\overline{\bX}^{(2)}$, which will be utilized in the analysis of generalization error in Section~\ref{sec:proof_test_random}.

\begin{lemma}\label{eq:deterministic_equ}
    Under the assumptions of Theorem \ref{thm:globallaw}, as $n\to\infty$, we have in probability, 
    \begin{small}
    \begin{align}
a_2\Tr\big((a_2\overline{\bX}^{(2)\top}\overline{\bX}^{(2)}+(a+\lambda)\bI)^{-1}\bSigma^{(2)}\big)\to ~& \frac{f''(0)\lambda_*}{4\alpha(a_*+\lambda)}-1,\\
a_2(a+\lambda)\Tr\big((a_2\overline{\bX}^{(2)\top}\overline{\bX}^{(2)}+(a+\lambda)\bI)^{-2}\bSigma^{(2)}\big)\to ~& \frac{f''(0)\lambda_*}{4\alpha(a_*+\lambda)}-\frac{1}{1-\alpha\int_{\R}\frac{x^2}{(x+\lambda_*)^2}d\mu_{\bSigma^{(2)}}(x)},\\
        \frac{2}{d^2}\Tr\big((a +\lambda)\bI+a_2\overline\bX^{(2)\top}\overline\bX^{(2)}\big)^{-2}\bSigma^{(2)}\big)  \to ~& \frac{\mathcal{B}(\lambda_*)}{(a_*+\lambda)^2},
    \end{align}
        \end{small}
where $\lambda_*>0$ is defined by  equation~\eqref{eq:self_consist} and $\mathcal{B}(\lambda_*)$ is defined by \eqref{eq:limit_bias}.
\end{lemma}
\begin{proof} 
Let us define $z_n:=\frac{2d^2(a+\lambda)}{n f''(0)}>0$ for all $n\in\N$. Notice that 
\begin{align}
    a_2\Tr\big((a_2\overline{\bX}^{(2)\top}\overline{\bX}^{(2)}+(a+\lambda)\bI)^{-1}\bSigma^{(2)}\big) = ~&\frac{1}{n}\Tr\big((\frac{1}{n}\overline{\bX}^{(2)\top}\overline{\bX}^{(2)}+z_n\bI)^{-1}\bSigma^{(2)}\big)\\
    = ~&\frac{1}{n}\Tr\big((\frac{1}{n}\sum_{i=1}^n\overline\bx_i^{(2)}\overline\bx_i^{(2)\top}+z_n\bI)^{-1}\bSigma^{(2)}\big)
\end{align}
where $\overline\bx_i^{(2)}$ is defined by \eqref{eq:centered_x2} for $i\in[n]$. Next, we follow the proof of Lemma 2.2 in \citep{ledoit2011eigenvectors} to complete the proof (see also \citep[Theorem 10]{wang2024nonlinear}).
For any fixed $z>0$, we define $\bR(z):=(\frac{1}{n}\sum_{i=1}^n\overline\bx_i^{(2)}\overline\bx_i^{(2)\top}+z\bI)^{-1}$ and $\bR^{(k)}(z):=(\frac{1}{n}\sum_{i\in[n\setminus \{k\}]}\overline\bx_i^{(2)}\overline\bx_i^{(2)\top}+z\bI)^{-1}$ for any $k\in [n]$. Then, by the Sherman-Morrison-Woodbury formula, we have
\begin{equation}\label{eq:leave-one-resolvent}
    \frac{1}{n}\overline\bx_i^{(2)\top}\bR(z)\overline\bx_i^{(2)}=1-
    \frac{1}{1+\frac{1}{n}\overline\bx_i^{(2)\top}\bR^{(i)}(z)\overline\bx_i^{(2)}}.
\end{equation}
Notice that
$\bR(z)\Big(\frac{1}{n}\sum_{i=1}^n\overline\bx_i^{(2)}\overline\bx_i^{(2)\top}+z\bI\Big)=\bI$.
Taking trace and applying \eqref{eq:leave-one-resolvent}, we  obtain
\begin{align}\label{eq:identity}
    1+\frac{z}{n}\Tr \bR(z) = \frac{\binom{d+1}{2}}{n}+\frac{1}{n}\sum_{i=1}^n \frac{1}{1+\frac{1}{n}\overline\bx_i^{(2)\top}\bR^{(i)}(z)\overline\bx_i^{(2)}}.
\end{align}
Notice that $\norm{\bR^{(i)}(z)}\le 1/z$ for all $i\in [n]$. Then, applying \eqref{eq:BZ_tensor6} in  Lemma~\ref{lem:quad_high_power} with matrix $\bA=\bR^{(i)}(z)$ for $i\in [n]$ we have, by a union bound over $i\in [n]$,
\begin{equation}\label{eq:quad_uniform}
    \max_{i\in[n]}\left|\frac{1}{n}\overline\bx_i^{(2)\top}\bR^{(i)}(z)\overline{\bx}_i^{(2)}-\frac{1}{n}\Tr(\bR^{(i)}(z)\bSigma^{(2)})\right|=O(n^{-\frac{1}{60}})
\end{equation}
with probability at least $1-O(d^{-1/5})$, for any fixed $z>0$. Additionally, by the Sherman-Morrison-Woodbury formula, we also have
\begin{align}\label{eq:R_i-R}
    \frac{1}{n}\Big|\Tr((\bR^{(i)}(z)-\bR(z))\bSigma^{(2)})\Big|\le \frac{1}{n}\left|\frac{ \frac{1}{n}\overline\bx_i^{(2)\top}\bR^{(i)}(z)\bSigma^{(2)}\bR^{(i)}(z)\overline\bx_i^{(2)}}{1+\frac{1}{n}\overline\bx_i^{(2)\top}\bR^{(i)}(z)\overline\bx_i^{(2)}}\right|\lesssim \frac{1}{n},
\end{align}where we applied the assumption of $\bSigma^{(2)}$, $\|\bR^{(i)}(z)\|\le 1/z$ and positive definiteness of $\bR^{(i)}(z)$. Then, from \eqref{eq:identity}, \eqref{eq:quad_uniform}, and \eqref{eq:R_i-R}, we have with probability at least $1-O(d^{-1/5})$,
\begin{align} 
    1+\frac{z}{n}\Tr \bR(z) = \frac{\binom{d+1}{2}}{n}+ \frac{1}{1+\frac{1}{n}\Tr\bR(z)\bSigma^{(2)}}+o(1),
\end{align}
where we used the fact that $1+\frac{1}{n}\Tr(\bR^{(i)}(z)\bSigma^{(2)})>1$, 
for any $z>0$. Thus, applying Theorem~\ref{thm:globallaw}, we can claim that for any $z>0$, 
\begin{equation}\label{eq:TrR(z)Sigma}
    \frac{1}{n}\Tr\bR(z)\bSigma^{(2)}\to \frac{1}{z\alpha m(-z)+1-\alpha}-1 = \frac{1}{z\widetilde m(-z)}-1,
\end{equation}
in probability as $n\to \infty$, where $m(-z)$ and $\widetilde m(-z)$ are defined in Definition~\ref{def:free_convolution} with $\nu=\mu_{\bSigma^{(2)}}$ in Assumption~\ref{assump:sigma}. Consider $z:=\frac{4\alpha(a_*+\lambda)}{f''(0)}>0$. Then, the fixed point equation~\eqref{eq:self_consist} defines $\lambda_*=\frac{1}{\widetilde m(-z)}>0$. Furthermore, notice that $z_n\to z=\frac{4\alpha(a_*+\lambda)}{f''(0)}$ as $n\to \infty$. Thus, 
    $\frac{1}{n}\big|\Tr\bR(z)\bSigma^{(2)}-\Tr\bR(z_n)\bSigma^{(2)}\big|\lesssim |z-z_n|\to 0$.
This completes the proof of the first part of this lemma. 

For the second part of this lemma, we follow the proof in Lemma 7.4 of
\citep{dobriban2018high}. Notice that \eqref{eq:TrR(z)Sigma} holds for any $z\in \mathbb{C}$ with $\Re(z)>0$ and $\frac{1}{n}|\Tr\bR(z)\bSigma^{(2)}|\lesssim 1$. Based on Lemma 2.14 in \citep{bai2010spectral}, we can obtain that \begin{equation}\label{eq:TrR2(z)Sigma}
    \frac{1}{n}\Tr\bR(z)^2\bSigma^{(2)}\to \frac{\widetilde m(-z)-z\widetilde m'(-z)}{z^2\widetilde m^2(-z)},
\end{equation}
in probability, for any $z\in \mathbb{C}$ with $\Re(z)>0$. From \eqref{eq:companion}, we know that 
\begin{align}\label{eq:tilde_m'}
    \frac{\widetilde m'(-z)}{\widetilde m^2(-z)}= \frac{1}{1-\alpha \int_{\R}\frac{x^2}{(x+\lambda_*)^2}d\mu_{\bSigma^{(2)}}(x)}.
\end{align}
Then, because of 
\[a_2(a+\lambda)\Tr\big((a_2\overline{\bX}^{(2)\top}\overline{\bX}^{(2)}+(a+\lambda)\bI)^{-2}\bSigma^{(2)}\big)= z_n\cdot \frac{1}{n}\Tr\bR(z_n)^2\bSigma^{(2)}, \] 
we can similarly derive the second part of the results. Lastly, since
\begin{align}
     &\frac{2}{d^2}\Tr\big(\big((a +\lambda)\bI+a_2\overline\bX^{(2)\top}\overline\bX^{(2)}\big)^{-2}\bSigma^{(2)}\big)\\ = ~&
       \frac{4 }{f''(0)(a_*+\lambda)} \left(\lambda_*/z-\frac{1}{1-\alpha \int_{\R}\frac{x^2}{(x+\lambda_*)^2}d\mu_{\bSigma^{(2)}}(x)}\right)\\
       =~&\frac{4 }{f''(0)(a_*+\lambda)} \frac{ \alpha\lambda_*^2\int \frac{  x}{(x+\lambda_*)^2}d\mu_{\bSigma^{(2)}}(x) }{ z(1-  \alpha\int \frac{x^2 }{(x+\lambda_*)^2}d\mu_{\bSigma^{(2)}}(x))}\\  = ~&\frac{\lambda_*^2 }{ (a_*+\lambda)^2} \frac{  \int \frac{  x}{(x+\lambda_*)^2}d\mu_{\bSigma^{(2)}}(x) }{  (1-  \alpha\int \frac{x^2 }{(x+\lambda_*)^2}d\mu_{\bSigma^{(2)}}(x))},
\end{align}we can apply \eqref{eq:TrR2(z)Sigma} and \eqref{eq:tilde_m'}  to conclude the final result of this lemma.
Here we also use the fixed point equation \eqref{eq:self_consist} of $\lambda_*$:
\begin{equation} 
   1- \frac{z }{\lambda_*} = \alpha\int \frac{x}{x+\lambda_*}d\mu_{\bSigma^{(2)}}(x) =\alpha\int \frac{x^2+\lambda_* x}{(x+\lambda_*)^2}d\mu_{\bSigma^{(2)}}(x).
\end{equation}
\end{proof}


\subsubsection{Spectral norm concentrations}
Next, we provide spectral norm bounds on $\bX\bX^\top$ and $(\bX\bX\tran)^{\odot 2}$ below.  
\begin{lemma}\label{lem:bound_XX_2}
Under Assumptions~\ref{assump:ratio}, \ref{assump:data}, and \ref{assump:sigma}, with a probability of at least $1-O(d^{-\frac{1}{48}})$,  we have 
 \begin{align}
     \|\bX\bSigma\bX\tran \| \lesssim  \| \bX\bX^\top\|&\lesssim d^{2+\frac{1}{24}}, \label{eq:XXconcentration}\\
  \|(\bX\bSigma\bX\tran)^{\odot 2}\| \lesssim \|(\bX\bX\tran)^{\odot 2}\| &\lesssim  d^{3},\label{eq:XX2concentration}\\
  \|\bX^{(2)}-\E\bX^{(2)}\|&\lesssim  d^{1+\frac{1}{12}}.\label{eq:X2centered_concentration}
   \end{align}
\end{lemma}
\begin{proof}
We first show \eqref{eq:XXconcentration} with Latala's Theorem \citep{latala2005some}. We can write $\bX^\top=\bSigma^{1/2} \bZ^\top$, where $\bZ^\top=[\bz_1,\dots,\bz_n]$ is a $d\times n$ random matrix with independent entries and each entry of $\bZ$ has zero mean and finite fourth moments. By \cite[Theorem 2]{latala2005some}, we have 
 $\E\|\bZ\|\lesssim \sqrt{n}+\sqrt{d}+(nd)^{1/4} \lesssim d$.
Then by Markov's inequality, with probability at least $1-O(d^{-\frac{1}{48}})$, 
$ \| \bX\bX^\top\| \lesssim \|\bZ\|^2\lesssim d^{2+\frac{1}{24}}$. 

Next, we show \eqref{eq:XX2concentration}. Since $(\bX\bX\tran)^{\odot 2}={\bX^{(2)}} {\bX^{(2)}}^\top$, it suffices to consider ${\bX^{(2)}}^\top \bX^{(2)}=\sum_{i=1}^n \bx_i^{(2)}{\bx_i^{(2)}}^\top$, which is a sum of $n$ i.i.d. rank-1 matrices. We will use matrix Bernstein's inequality \citep[Theorem 5.4.1]{vershynin2018high} to prove \eqref{eq:XX2concentration}. Consider truncated vectors ${\bz_i^{(2)}}:= \bx_i^{(2)}\1\{ \|\bx_i^{(2)}\|\leq B d\}$ for a parameter $B=n^{\frac{1}{44}}$. Let $\bZ^{(2)}$ be the truncated version of $\bX^{(2)}$. We have that
    \begin{align}
        \P \left( \bZ^{(2)}\not=\bX^{(2)}\right)&\leq \P \left(  \max_{i\in [n]}   \|\bx_i^{(2)}\|> B d \right)\leq \frac{n\E\|\bx^{(2)}\|^{45}}{(Bd)^{45}}\lesssim \frac{n}{B^{45}}\lesssim n^{-\frac{1}{45}}. \label{eq:truncate_prob}
    \end{align}
On the other hand, almost surely,
$\left\|{\bz_i^{(2)}}{\bz_i^{(2)}}^\top-\E{\bz_i^{(2)}}{\bz_i^{(2)}}^\top\right\|\lesssim (Bd)^2$,
and 
\begin{align}
    \E \left({\bz_i^{(2)}}{\bz_i^{(2)}}^\top-\E{\bz_i^{(2)}}{\bz_i^{(2)}}^\top \right)^2 \preccurlyeq \E\left[\|\bz_i^{(2)}\|^2 {\bz_i^{(2)}}{\bz_i^{(2)}}^\top\right]\preccurlyeq(Bd)^2\bSigma^{(2)}\leq C(Bd)^2\bI
\end{align}
for some  constant $C>0$ due to Assumption~\ref{assump:sigma}. By matrix Bernstein's inequality \citep[Theorem 5.4.1]{vershynin2018high}, we have with probability at least $1-d^2 \exp(-\frac{5}{66}d)$, 
    \begin{align}
        \left\| {\bZ^{(2)}}^\top \bZ^{(2)}-\E{\bZ^{(2)}}^\top \bZ^{(2)}\right\| \lesssim d^{2+\frac{1}{6}}.
    \end{align}
We also have
    $\E{\bZ^{(2)}}^\top \bZ^{(2)}\lesssim n \E \bx^{(2)}{\bx^{(2)}}^\top\leq Cd^3\bI$,
    where we use the definition of $\bx^{(2)}$ from \eqref{eq:def_reduced}.
    Together with \eqref{eq:truncate_prob}, we have with probability at least $1-O(d^{-\frac{2}{45}})$,
     $\|(\bX\bX\tran)^{\odot 2}\| \lesssim  d^{3}$.

     For \eqref{eq:X2centered_concentration}, we have 
\begin{align}\label{eq:X2EX2}
    \|\bX^{(2)}-\E \bX^{(2)}\|\leq \|\bX^{(2)}-\bZ^{(2)}\|+\|\bZ^{(2)}-\E\bZ^{(2)}\|+\|\E\bX^{(2)}-\E\bZ^{(2)}\|.
\end{align}
From \eqref{eq:truncate_prob}, with probability $1-O(n^{-1/45})$, the first term in \eqref{eq:X2EX2} is zero. For the second term in \eqref{eq:X2EX2}, we consider $\|\bZ^{(2)}-\E\bZ^{(2)}\|^2=\|(\bZ^{(2)}-\E\bZ^{(2)})(\bZ^{(2)}-\E\bZ^{(2)})^\top \|$,   
where 
\begin{align}
  (\bZ^{(2)}-\E\bZ^{(2)})(\bZ^{(2)}-\E\bZ^{(2)})^\top= \sum_{i=1}^n (\bz_i^{(2)}-\E  \bz_i^{(2)})(\bz_i^{(2)}-\E  \bz_i^{(2)})^\top,
\end{align}
and apply the matrix Bernstein's inequality.
We have almost surely, $\| (\bz_i^{(2)}-\E  \bz_i^{(2)})(\bz_i^{(2)}-\E  \bz_i^{(2)})^\top\| \leq 4(Bd)^2$.
And for some constant $C>0$,
\begin{align}
    \E \left((\bz_i^{(2)}-\E  \bz_i^{(2)})(\bz_i^{(2)}-\E  \bz_i^{(2)})^\top\right)^2&=\E \norm{\bz_i^{(2)}-\E  \bz_i^{(2)}}^2(\bz_i^{(2)}-\E  \bz_i^{(2)})(\bz_i^{(2)}-\E  \bz_i^{(2)})^\top \\
    &\leq 4(Bd)^2 \E (\bz_i^{(2)}-\E  \bz_i^{(2)})(\bz_i^{(2)}-\E  \bz_i^{(2)})^\top\\
    &\leq 4(Bd)^2 \bSigma^{(2)}\lesssim C(Bd)^2\bI.
\end{align}
With matrix Bernstein's inequality \citep[Theorem 5.4.1]{vershynin2018high}, we have with probability at least $1-d^2 \exp(-\frac{5}{66}d)$,
$\|(\bZ^{(2)}-\E\bZ^{(2)})(\bZ^{(2)}-\E\bZ^{(2)})^\top\|\lesssim d^{2+\frac{1}{6}}$.
Hence with probability $1-O(d^{-\frac{2}{45}})$, from \eqref{eq:X2EX2},
$
 \|\bX^{(2)}-\E \bX^{(2)}\| \lesssim d^{1+\frac{1}{12}} + \|\E\bX^{(2)}-\E\bZ^{(2)}\|$.  Since each column of $\bX^{(2)}$ has the same distribution,    $\E\bX^{(2)}-\E\bZ^{(2)}$ is of rank 1. We obtain 
  \begin{align}
   \|\E\bX^{(2)}-\E\bZ^{(2)}\|&=\|\E\bX^{(2)}-\E\bZ^{(2)}\|_{\msf F}= \sqrt{n} \mathbb E[\|\bx^{(2)}\|\mathbf{1}\{\|\bx^{(2)}\|\geq Bd \}] \\
   &\leq \sqrt{n} \sqrt{\E[\|\bx^{(2)}\|^2]}\sqrt{\mathbb P(\|\bx^{(2)}\|\geq Bd)}\lesssim \sqrt{nd^2B^{-45}}\lesssim \sqrt{d^2n^{-\frac{1}{44}}}=d^{1-\frac{1}{44}},
  \end{align}
  where in the second inequality we use \eqref{eq:truncate_prob}. Therefore we obtain with probability $1-O(d^{-\frac{2}{45}})$,
$\|\bX^{(2)}-\E \bX^{(2)}\| \lesssim d^{1+\frac{1}{12}}$ 
as desired. This finishes the proof.
\end{proof}

\subsubsection{Kernel function expansion}
Recall $\bx=\bSigma^{1/2}\bz$ and $\w_i=\bSigma^{1/2} \bx_i$ for $i\in[n]$ and $\bz\sim\cN(0,\bI)$. Let $t_{i}=\bx_i^\top\bSigma\bx_i=\|\w_i\|^2$ and $\bu_i=\frac{\w_i}{\|\w_i\|}$. Then 
\begin{align}\label{eq:xx_i}
  \langle \bx_i,\bx\rangle=\sqrt{t_i} \langle \bu_i, \bz\rangle,
\end{align}
and for $j=0,\dots, 8$ and $i\in[n]$, define
\begin{align}\label{eq:Tij}
    \T_i^{(j)}:=t_i^{j/2}\sqrt{j!}\cdot h_j\left( \langle \bu_i, \bz\rangle \right),
\end{align}
where $h_j$ is the $j$-th normalized Hermite polynomial defined in Definition~\ref{def:hermitepolynomial}. 
    
\begin{lemma}\label{lem:hermite}
Under Assumption~\ref{assumption: testdata}, we have for any $i,j\in [n]$,
    $\E_{\bx}[\T_i^{(k)}\T_j^{(\ell)}]=0$ if $k\neq\ell$ and $k+\ell\leq 15$, and  for all $k=0,1,\ldots,8$,
    $\E_{\bx}[\T_i^{(k)}\T_j^{(k)}]=k!\langle \bw_i, \bw_j\rangle^{k}$, where $\bw_i:=\bSigma^{1/2}\bx_i$. 
\end{lemma}
\begin{proof}
Since the calculation of $\E_{\bx}[\T_i^{(k)}\T_j^{(\ell)}]$ involves only  the first 16th moments of $\bz$ for $k+\ell\leq 15$,  by the orthogonality property of $h_j$ in Lemma~\ref{lem:NM20D2} and  assumption~\ref{assumption: testdata},
    \begin{align}
        \E_{\bx}[\T_i^{(k)}\T_j^{(\ell)}]=~&t_i^{k/2}t_j^{\ell/2}\sqrt{k!\ell !}\cdot\E_{\bz}[h_k\left( \langle \bu_i, \bz\rangle \right)h_{\ell}(\langle \bu_j, \bz\rangle )]\\
        =~&\delta_{k,\ell}\cdot k! t_i^{k/2}t_j^{k/2}\langle \bu_i, \bu_j\rangle^{k}=\delta_{k,\ell} \cdot k!\langle \bw_i, \bw_j\rangle^{k}.
    \end{align}
   Hence, $\E_{\bx}[\T_i^{(k)}\T_j^{(\ell)}]=0$ if $k\not=\ell$. This finishes the proof.
\end{proof}

For any $i\in[n]$, let us apply the Taylor expansion  of $f$ as in \eqref{eq:taylor_off}  to get
\begin{align} 
K(\bx_i,\bx)=\sum_{k=0}^{8}\frac{f^{(k)}(0)}{k!d^k} \langle \bx_i,\bx \rangle^k+\frac{f^{(9)}(\zeta_{i})}{9!d^9}\langle \bx_i,\bx \rangle^9,
\end{align}
where $\zeta_{i}$ is between $0$ and $\frac{1}{d}\langle \bx_i, \bx \rangle$. 

Recall \eqref{eq:xx_i}, we have 
   $$ \sum_{k=0}^8 \frac{f^{(k)}(0)}{k!d^k} \langle \bx_i,\bx \rangle^k=\sum_{k=0}^8 \frac{f^{(k)}(0)}{k!d^k} t_i^{k/2}\langle \bu_i,\bz \rangle^k,$$
where  $t_{i}:=\bx_i^\top\bSigma\bx_i$ for $i\in[n]$.
With Lemma \ref{lem:hermite} and \eqref{eq:Tij}, we can rewrite $K(\bx_i,\bx)$ as
\begin{align}\label{eq:taylor_expansion}
    K(\bx_i,\bx)=~&\sum_{k=0}^8 b_{k,i}\bT_{i}^{(k)}+\frac{f^{(9)}(\zeta_{i})}{9!d^9}\langle \bx_i,\bx \rangle^9 .
\end{align} 
By orthogonality of the normalized Hermite polynomials, we have  
\begin{align}
    b_{0,i}=~&f(0)+t_{i}\cdot \frac{f^{(2)}(0)}{2!d^2}+3t_{i}^2\cdot \frac{f^{(4)}(0)}{4!d^4}+15t_{i}^3\cdot \frac{f^{(6)}(0)}{6!d^6},\label{eq:b0i}\\
    b_{1,i}=~&\frac{f^{(1)}(0)}{ d }+3t_{i}\cdot \frac{f^{(3)}(0)}{3!d^3}+15t_{i}^2\cdot \frac{f^{(5)}(0)}{5!d^5}+105t_{i}^3\cdot \frac{f^{(7)}(0)}{7!d^7},\label{eq:b1i}\\
    b_{2,i}=~&\frac{f^{(2)}(0)}{2!d^2}+6t_{i}\cdot \frac{f^{(4)}(0)}{4!d^4}+45t_{i}^2\frac{f^{(6)}(0)}{6!d^6}, \\
    b_{3,i}=~&\frac{f^{(3)}(0)}{3!d^3}+10t_{i}\cdot \frac{f^{(5)}(0)}{5!d^5}+105t_{i}^2\cdot \frac{f^{(7)}(0)}{k!d^7}.
\end{align}
In general, 
for $0\leq k\leq 8$, $
    b_{k,i}t_i^{k/2}\sqrt{k!}
    =\sum_{s=k}^8 t_i^{s/2}\frac{f^{(s)}(0)}{s!d^s}\E_{g\sim \mathcal N(0,1)}[g^s h_k(g)].$
Therefore,  
\begin{align}\label{eq:bki}
|b_{k,i}| \lesssim \sum_{s=k}^8 d^{-s} t_i^{(s-k)/2}.
\end{align}
Utilizing \eqref{eq:centered6}, we can easily check that 
\begin{equation}\label{eq:t_i_concen}
    |t_i-\Tr\bSigma^2|\lesssim d^{\frac{1}{2}+\frac{1}{30}},
\end{equation}uniformly for all $i\in [n]$ with probability at least $1-d^{-1}$.  Thus, $0\le t_i\lesssim d$. Therefore, from \eqref{eq:bki}, for $k=0,1,\ldots,8$ and all $i\in [n]$, with probability  at least $1-d^{-1}$,
\begin{equation}\label{eq:b_ki_bound}
    |b_{k,i}| \lesssim d^{-k}.
\end{equation}

\begin{lemma}\label{lemm:b_k,i}
Let us denote that 
\begin{align}
    \widetilde b_{0,i}:=~f(0)+t_{i}\cdot \frac{f^{(2)}(0)}{2!d^2}, \quad 
    \widetilde b_{1,i}:=~\frac{f^{(1)}(0)}{ d }+3t_{i}\cdot \frac{f^{(3)}(0)}{3!d^3}\label{eq:tildeb1i}
\end{align}
for any $i\in [n]$. Then, under Assumption~\ref{assump:nonlinear_f}, we have
  \begin{align}
    \max_{i\in [n]}|\widetilde b_{0,i}-b_{0,i}|\lesssim~ d^{-2} ,\quad 
    \max_{i\in [n]}|\widetilde b_{1,i}-b_{1,i}|\lesssim~ d^{-3}, \quad 
     \max_{i\in [n]}| a_{2}-b_{2,i}|\lesssim~ d^{-3.4}
\end{align}
with probability at least $1-d^{-1}$, where $ a_2$ is defined in \eqref{def:a2}.
\end{lemma}
\begin{proof}
The first two bounds are directly from \eqref{eq:t_i_concen}. Recall the definition of $a_2$ in \eqref{def:a2}. Then for the last bound, we have
\begin{align}
b_{2,i} -a_{2} & = \frac{f^{(4)}(0)}{4d^4}(t_i-\Tr(\bSigma^2))+45t_{i}^2\frac{f^{(6)}(0)}{6!d^6}.  
\end{align}
Applying \eqref{eq:t_i_concen}, we can derive that 
 $|b_{2,i} -a_{2}| \lesssim \frac{1}{ d^4}|t_i-\Tr(\bSigma^2)|+\frac{1}{ d^6}|t_{i}^2| \lesssim d^{-3.4}$
uniformly for all $i\in [n]$ with probability at least $1-d^{-1}$. 
\end{proof}

\subsubsection{Approximation of product of kernel functions}\label{sec:quadratic_appro}
Denote  $\bM :=~ \E[K(\bX,\bx) K(\bx,\bX)|\bX], 
    \bv := ~ \E_\bx[f_*(\bx)K(\bX,\bx)]$,
where \[K(\bX,\bx)=[K(\bx_1,\bx),\dots, K(\bx_n,\bx)]^\top\in \R^{n}\] and $\E_\bx[\cdot]$ denotes the expectation only with respect to $\bx$.  Notice that for any $i,j\in[n]$,
    \begin{align*}
        \bM_{ij}=~(\E[K(\bX,\bx) K(\bx,\bX)])_{ij}=\E_\bx[K(\bx_i,\bx)K(\bx,\bx_j)],\quad 
        \bv_{i}=~\E_\bx[K(\bx,\bx_i)f_*(\bx)].
    \end{align*}
We define
\begin{align}
    \bb_{0}=~&(b_{0,1},\dots,b_{0,n})^\top \in \R^n,\quad
    \bb_{1}=( b_{1,1},\dots, b_{1,n})^\top \in \R^n, \label{eq:b_0}\\
    \widetilde\bb_{0}=~&(\widetilde b_{0,1},\dots,\widetilde b_{0,n})^\top \in \R^n,\quad
    \widetilde\bb_{1}=( \widetilde b_{1,1},\dots, \widetilde b_{1,n})^\top \in \R^n,\label{eq:tilde_b_0}
\end{align}
where $b_{0,i},b_{1,i}, \widetilde{b}_{0,i}$, and $\widetilde{b}_{1,i}$ are defined in \eqref{eq:b0i},~\eqref{eq:b1i},~\eqref{eq:tilde_b_0}, and~\eqref{eq:tildeb1i}, respectively. Denote
\begin{align}
    \bM^{(2)}:= \bb_0 \bb_0^\top+ \diag(\bb_1)\bX\bSigma\bX^\top\diag(\bb_1) +2a^2_2\bM^{(2)}_0,\quad 
    \bM^{(2)}_0:= (\bX\bSigma\bX^\top)^{\odot 2}.\label{eq:M(2)}
\end{align}
In the following, we first provide an approximation of $\bM$ in terms of $\bM^{(2)}$.

\begin{lemma}\label{lemma:diff_M_v}
Under the same assumptions as Theorem~\ref{thm:concentration}, we have that
    $
        \|{\bM-\bM^{(2)}}\| \lesssim ~ \frac{1}{d^{9/4}}$,

with probability $1-O(d^{-1/48})$.
\end{lemma}
\begin{proof}
For $i,j\in[n]$, we can apply the orthogonality property in Lemma~\ref{lem:hermite} to get
\begin{align*}
\bM_{ij}=~&\sum_{k=0}^{8}b_{k,i }b_{k,j}\cdot \E_{\bx}[\T_i^{(k)}\T_j^{(k)}] 
 +\sum_{k=0}^{8}\E_{\bx}\Big[b_{k,i} \T_i^{(k)}\frac{f^{(9)}(\zeta_{j})}{9!d^9}\langle \bx_j,\bx \rangle^9\Big]  \\
 + &\sum_{k=0}^8 \E_{\bx} \big[ b_{k,j}\T_j^{(k)}\frac{f^{(9)}(\zeta_{i})}{9!d^9}\langle \bx_i,\bx \rangle^9 \Big] 
+\E_{\bx}\Big[ \frac{f^{(9)}(\zeta_{i})f^{(9)}(\zeta_{j})}{(9!)^2d^{18}} \langle \bx_i,\bx \rangle^9  \langle \bx_j,\bx \rangle^9  \Big]\\
=:~&  \bL_{i,j}+\V_{i,j}^{(1)}+\V_{i,j}^{(2)}+\V_{i,j}^{(3)}.
\end{align*}
Recall that $\bw_i=\bSigma^{1/2}\bx_i$ for all $i\in[n]$. By the assumption that $f^{(9)}(x)$ is uniformly bounded in Assumption~\ref{assumption:C8}, we have from \eqref{eq:b_ki_bound}, with probability  $1-O(d^{-1})$,
\begin{align}
    |\V_{i,j}^{(1)}|\lesssim~& \sum_{k=0}^8 \frac{1}{d^{9+k}}\E_{\bx}[|\bT_{i}^{(k)}
    \langle \bx_j,\bx\rangle ^9|]\lesssim \sum_{k=0}^8 \frac{1}{d^{9+k}}\sqrt{\E_{\bx}|\bT_{i}^{(k)}|^2}
    \sqrt{\E_{\bx}\langle \bx_j,\bx\rangle ^{18}}\\
    &\lesssim\sum_{k=0}^8\frac{1}{d^{k+9}}\|\bw_i\|^k\|\bw_j\|^9,
\end{align}
where in the last inequality, we use Lemma~\ref{lem:hermite} and Lemma~\ref{lemm:wick} under the Gaussian moment matching condition in Assumption~\ref{assumption: testdata}. 
Similarly, 
\begin{align}
  |\V_{i,j}^{(2)}|\lesssim~ \sum_{k=0}^8\frac{1}{d^{k+9}}\|\bw_j\|^k\|\bw_i\|^9, \quad 
    |\V_{i,j}^{(3)}|\lesssim~  \frac{1}{d^{18}}\|\bw_i\|^9\|\bw_j\|^9.
\end{align}
Notice that  the leading order $|\V_{i,j}^{(\ell)}|\lesssim \frac{1}{d^8}\|\bw_i\|^8$ for $\ell=1,2$. Recall \eqref{eq:wi}, i.e., $\E\left[\norm{\bw_i}^{2s}\right] =\E [\| \bSigma \bz_i\|^{2s}]\lesssim  d^{s}$ for any $1\le s\le 45$. Thus,  Markov's inequality implies that
$\P(|\V_{i,j}^{(\ell)}|>t)\le \frac{1}{(d^{4.5}t)^s}$
for all $i,j\in [n]$ and $\ell=1,2$. Then taking $t=d^{-17/4}$ and $s=18$, then taking union bounds for all $i,j\in [n]$, we can derive that $\norm{\V^{(\ell)}}\le \norm{\V^{(\ell)}}_{\mathrm{F}}\lesssim d^{-9/4}$ with probability at least $1-cd^{-1/2}$ for some constant $c>0$ and $\ell =1,2$. Similarly, we can verify the same bound holds for $\ell=3$.

 Let us further define matrices $\L^{(k)}$ whose $(i,j)$ entry is given  by
\[\L^{(k)}_{i,j}:=b_{k,i}b_{k,j}\cdot \E_{\bx}[\T_i^{(k)}\T_j^{(k)}]=k!b_{k,i}b_{k,j}\langle \bw_j,\bw_i\rangle^k\] for $i,j\in[n]$ and $0\le k\le 8$, where we applied Lemma~\ref{lem:hermite}. 
We next employ \eqref{eq:wi} and \eqref{eq:centered6} to deduce that
$\norm{\L^{(k)}} \lesssim  \frac{1 }{d^{9/4}}$,
for $3\le k\le 8$, with probability at least $1-O(d^{-1/2})$. 
Let us extract the diagonal matrix of  $\L^{(k)}$ by denoting $\L^{(k)}_{\text{diag}}$. Set $\L^{(k)}_{\text{off}}:=\L^{(k)}-\L^{(k)}_{\text{diag}}$. Then, we bound the operator norms of $\L^{(k)}_{\text{off}}$ and $\L^{(k)}_{\text{diag}}$ separately. First, 
\begin{align}
    \norm{\L^{(k)}_{\text{off}}}\le \norm{\L^{(k)}_{\text{off}}}_{\mathrm{F}}\lesssim \frac{n}{d^{2k}}\max_{i\neq j}\langle \bw_j,\bw_i\rangle^k\lesssim \frac{1}{d^{2.5}},
\end{align}
with probability at least $1-O(d^{-1/2})$, for $3\le k\le 8$. Next, for the diagonal part, we have
$ \norm{\L^{(k)}_{\text{diag}}}\lesssim \frac{1}{d^{2k}}\max_{i\in [n]}\norm{\bw_i}^{2k}\lesssim \frac{1}{d^{3}}$,
with probability at least $1-O(d^{-1/2})$, for $3\le k\le 8$. 

Lastly, let us denote that $\bb_2=[b_{2,1},\ldots,b_{2,n}]^\top$. Hence, $$\L^{(2)}=2\diag(\bb_2)(\bX\bSigma\bX^\top)^{\odot 2}\diag(\bb_2).$$ Lemma~\ref{lemm:b_k,i} proves that $|b_{2,i}-a_2|\lesssim 1/d^{3.4}$ and $|b_{2,i}|\lesssim 1/d^2$ with probability  $1-d^{-1}$ for all $i\in[n]$. Moreover, $|a_{2}|\lesssim 1/d^2$. Then, by Lemma~\ref{lem:bound_XX_2},   with probability at least $1-O(d^{-\frac{1}{48}})$,
\begin{align}
    \norm{\L^{(2)}-2a^2_2\bM_0^{(2)}}&\lesssim \left(\norm{\diag(\bb_2)(\bX\bSigma\bX^\top)^{\odot 2}}+a_2\norm{ (\bX\bSigma\bX^\top)^{\odot 2}}\right)\max_{i\in [n]}|b_{2,i}-a_2|\lesssim d^{-{2.4}}.
\end{align}
Then, we complete the proof of the approximation on $\bM$ by $\bM^{(2)}$. 
\end{proof}

\begin{lemma}\label{lemm:M_02_decomp}
With Assumption~\ref{assump:limitsigma}, we have
\begin{equation}\label{eq:M_02_decomp}
    \bM_0^{(2)}=\frac{1}{2}\bX^{(2)}\bSigma^{(2)}\bX^{(2)\top}-\frac{1}{2}\sum_{k=1}^d\bSigma_{kk}^2\vnu_k\vnu_k^\top,
\end{equation}
where $\vnu_k:=[\bx_1(k)^2,\ldots,\bx_n(k)^2]^\top$ for $k\in [d]$ and $\bSigma^{(2)}$ is defined by \eqref{eq:defSigma2}. Moreover, under the Assumption~\ref{assumption: testdata}, we have
$\norm{\vnu_k}\lesssim d^{1+\frac{1}{22}}$ for all $k\in[d]$, with probability at least $1-d^{-1}$. 
\end{lemma}
\begin{proof}
  By the definition of $\bSigma^{(2)}$ in \eqref{eq:defSigma2}, we can easily check \eqref{eq:M_02_decomp}. Notice that $\E[\vnu_k]=\bSigma_{kk}\1$ and $\norm{\E[\vnu_k]}\lesssim \sqrt{n}$. By the Assumptions~\ref{assumption: testdata} and \ref{assump:limitsigma}, we know that 
  $
      \E[\norm{\vnu_k}^{2s}]= \E[(\sum_{i=1}^n{\bx_{i}(k)^4})^s]\lesssim d^{2s}$, for $0\le 4s\le 90$. Then, we can conclude the final bound of this lemma by taking $s=22$ and applying Markov inequality for $\norm{\vnu_k}$.
\end{proof}

\subsubsection{Resolvent calculations}\label{sec:resolvent_calc}

\begin{lemma}\label{lemm:oneK2one}
Under the assumptions of Theorem~\ref{thm:concentration}, we have
\begin{align}
    \ones^\top(\bK+\lambda\bI)^{-2}\ones\lesssim  d^{-\frac{23}{24}},  \quad 
     \1^\top(\bK+\lambda\bI)^{-1}\ones  \lesssim  1 , \quad 
          \big|1-  b_0\1^\top\bK_{\lambda}^{-1} \1\big|\lesssim ~ d^{-\frac{23}{24}}
\end{align}
 with probability at least $1-O(d^{-1/48})$, where $b_0:=f(0)$.
\end{lemma}
\begin{proof}
Denote $\bK_{\lambda}^{-1}:=(\bK+\lambda\bI)^{-1}$. From Theorem~\ref{thm:concentration}, there exists a matrix $\bK_*\in\R^{n\times n}$ such that with probability at least $1-O(d^{-1/2})$,
\[\bK_\lambda=\bK_*+a_0\1\1^\top, \quad \norm{\bK_*-a_1 \bX\bX\tran + a_2 (\bX\bX\tran)^{\odot 2} + (a+\lambda)\bI_n}\lesssim d^{-\frac{1}{12}}.\]
Thus, by Assumption~\ref{assump:analytic} and Lemma~\ref{lem:bound_XX_2},  $c\bI \preccurlyeq \bK_*\preccurlyeq Cd^{1+\frac{1}{24}}\bI$, for some constants $c,C>0$ with probability  $1-O(d^{-1/48})$. By the Sherman-Morrison-Woodbury formula, we have 
\begin{align}\label{eq:sm_formula}
    \bK_{\lambda}^{-1} = \bK_*^{-1}-a_0\frac{\bK_*^{-1}\1\1^\top\bK_*^{-1}}{1+a_0\1^\top\bK_*^{-1}\1}.
\end{align}
Therefore, we can obtain that
\begin{align}
    &\1^\top\bK_{\lambda}^{-2}  \\   =&\1^\top\bK_*^{-2}+\frac{(a_0\1^\top\bK_*^{-1}\1)(a_0\1^\top\bK_*^{-2}\1)}{(1+a_0\1^\top\bK_*^{-1}\1)^2}\1^\top\bK_*^{-1}-\frac{a_0\1^\top\bK_*^{-2}\1\1^\top\bK_*^{-1}}{1+a_0\1^\top\bK_*^{-1}\1}-\frac{a_0\1^\top\bK_*^{-1}\1\1^\top\bK_*^{-2}}{1+a_0\1^\top\bK_*^{-1}\1}\\
    =&-\frac{a_0\1^\top\bK_*^{-2}\1\1^\top\bK_*^{-1}}{(1+a_0\1^\top\bK_*^{-1}\1)^2}+\frac{\1^\top\bK_*^{-2}}{1+a_0\1^\top\bK_*^{-1}\1}.
\end{align}

Thus, we have
\begin{align}
    \mathbf{1}_n^\top\bK_{\lambda}^{-2}\mathbf{1}_n = \frac{ \1^\top\bK_*^{-2}\1 }{(1+a_0\1^\top\bK_*^{-1}\1)^2}\le \frac{1}{ca_0^2}\frac{\mathbf{1}_n^\top\bK_{*}^{-1}\mathbf{1}_n}{(\mathbf{1}_n^\top\bK_{*}^{-1}\mathbf{1}_n)^2}\lesssim \frac{ d^{1+1/24}}{  \|\mathbf{1}_n\|^2}\lesssim \frac{1}{d^{23/24}}.\label{eq:term1_trace}
\end{align}with probability at least $1-O(d^{-1/48})$. The second bound in this lemma comes directly from \eqref{eq:sm_formula} since 
$a_0\1^\top(\bK+\lambda\bI)^{-1}\ones=\frac{a_0\1^\top\bK_*^{-1}\1}{1+a_0\1^\top\bK_*^{-1}\1}\le 1$.
Lastly, \eqref{eq:sm_formula} implies that 
$1-a_0\1^\top\bK_{\lambda}^{- 1}\1 = \frac{1}{1+a_0\1^\top\bK_*^{-1}\1}$.
The same bound as \eqref{eq:term1_trace} can be employed here to get
$|1-a_0\1^\top\bK_{\lambda}^{- 1}\1|\lesssim d^{-\frac{23}{24}}$,
with probability at least $1-O(d^{-1/48})$. Hence,
\[|1-b_0\1^\top\bK_{\lambda}^{- 1}\1|\le |1-a_0\1^\top\bK_{\lambda}^{- 1}\1|+|a_0-b_0|\cdot\1^\top\bK_{\lambda}^{-1}\1\lesssim d^{-\frac{23}{24}},\]with probability at least $1-O(d^{-1/48})$.
\end{proof}

Let us denote  
\begin{equation}\label{eq:vmu}
    \vmu^\top := [t_1,t_2,\ldots,t_n],
\end{equation}
where $t_i=\bx_i^\top\bSigma\bx_i$, for $i\in [n]$. Recall $\overline{\bX}^{(2)}=\bX^{(2)}-\E[\bX^{(2)}]$ and notice that
\begin{align}
       (\bX\bX^\top)^{\odot 2}
       =~&  \lbX  \lbXt  + \left(\bX^{(2)}\E[\bX^{(2)}]^\top-\E[\bX^{(2)}]\E[\bX^{(2)}]^\top+\E[\bX^{(2)}]\bX^{(2)\top}\right),
\end{align} 
where 
\begin{align}
     \bX^{(2)}\E[\bX^{(2)}]^\top=~\vmu\1^\top, \quad 
     \E[\bX^{(2)}]^\top\bX^{(2)}=~\1\vmu^\top, \quad 
\E[\bX^{(2)}]\E[\bX^{(2)}]^\top=~\Tr(\bSigma^2)\cdot\1\1^\top.
\end{align}
Thus, we define $\bU:=[\1,\vmu]\in\R^{n\times 2}$. Then, 
\begin{equation}\label{eq:decomp:X^2}
    a_2(\bX\bX^\top)^{\odot 2} = \bK_*^{(2)}+a_2\bU  \begin{pmatrix}
        -\Tr(\bSigma^2) &1\\
        1& 0
    \end{pmatrix}\bU^\top
\end{equation}
where
\begin{equation}\label{eq:K_*(2)}
    \bK_*^{(2)}:=a_2(\bX^{(2)}-\E[\bX^{(2)}])(\bX^{(2)}-\E[\bX^{(2)}])^\top.
\end{equation}

\begin{lemma}\label{lemm:uKu}
Under the assumptions of Theorem~\ref{thm:concentration} and Assumption~\ref{assump:limitsigma},  with probability at least $1-O(d^{-1/2})$,
$
    \frac{1}{d^4}  \vmu^\top\bK_{\lambda}^{-1}\vmu\lesssim d^{-0.8},
$
 where $\vmu$ is defined by \eqref{eq:vmu}. 
As a corollary, we also have
$
    \frac{1}{d^2}  \1^\top\bK_{\lambda}^{-1}\vmu\lesssim d^{-0.4}$.
\end{lemma}
\begin{proof}
Let $\vmu_0:=\E\vmu = \Tr(\bSigma^2)\1$. Due to \eqref{eq:t_i_concen}, we can conclude that
    \begin{align}\label{eq:vmu-concen_norm}
        \norm{\vmu-\vmu_0}\lesssim d^{1.6},
    \end{align}with probability at least $1-O(d^{-1})$. 
Thus,
\begin{align}
     \vmu^\top\bK_{\lambda}^{-1}\vmu= ~&  (\vmu-\vmu_0)^\top\bK_{\lambda}^{-1}(\vmu-\vmu_0)+ \vmu_0^\top\bK_{\lambda}^{-1}\vmu_0+2  (\vmu-\vmu_0)^\top\bK_{\lambda}^{-1}\vmu_0.
\end{align}
Here, we know that $\frac{1}{d^4}  (\vmu-\vmu_0)^\top\bK_{\lambda}^{-1}(\vmu-\vmu_0)\le \frac{1}{d^4}\norm{\vmu-\vmu_0}^2\le d^{-0.8}$,
and $$\frac{1}{d^4}  \vmu_0^\top\bK_{\lambda}^{-1}\vmu_0=\frac{\Tr(\bSigma^2)^2}{d^4}  \1^\top\bK_{\lambda}^{-1}\1\lesssim d^{-2}$$ 
with probability at least $1-O(d^{-1/2})$, because of \eqref{eq:K_lambda_-1} and Lemma~\ref{lemm:oneK2one}. Moreover, the last term can be bounded by Cauchy-Schwartz inequality:
\begin{align}
    \frac{1}{d^4}  |(\vmu-\vmu_0)^\top\bK_{\lambda}^{-1}\vmu_0|\le \frac{1}{d^4} \left((\vmu-\vmu_0)^\top\bK_{\lambda}^{-1}(\vmu-\vmu_0)\right)^{1/2}\left(\vmu_0^\top\bK_{\lambda}^{-1}\vmu_0\right)^{1/2}\lesssim d^{-1.4}.
\end{align}
Then we complete the proof of the lemma.
\end{proof}

\begin{lemma}\label{lemm:b0Kb0}
Under the assumptions of Theorem~\ref{thm:concentration} and Assumption~\ref{assump:limitsigma}, 
we have  with probability at least $1-O(d^{-1/48})$,
$\bb_0^\top(\bK+\lambda\bI)^{-2}\bb_0\lesssim  d^{-0.8}$ and  $\bb_0^\top(\bK+\lambda\bI)^{-1}\bb_0\lesssim 1$.
\end{lemma}
\begin{proof}
 Recall the definition of $\widetilde \bb_0$ in \eqref{eq:simple_v(2)}. We have
\begin{align}
    \norm{(\bK+\lambda\bI)^{-1}\bb_0}^2\le ~&2\norm{(\bK+\lambda\bI)^{-1}(\widetilde \bb_0 -\bb_0)}^2+2 \norm{(\bK+\lambda\bI)^{-1}\widetilde \bb_0}^2\\
    \lesssim ~& n\cdot \max_{i\in [n]}|\widetilde b_{0,i}-b_{0,i}|^2+\ones^\top(\bK+\lambda\bI)^{-2}\ones + \frac{1}{d^4}  \vmu^\top\bK_{\lambda}^{-1}\vmu\lesssim d^{-0.8},
\end{align}
with probability at least $1-O(d^{-1/48})$, where we use Lemma~\ref{lemm:b_k,i}, \eqref{eq:K_lambda_-1}, Lemma~\ref{lemm:oneK2one} and Lemma~\ref{lemm:uKu}. 
Similarly, by Lemmas~\ref{lemm:b_k,i},~\ref{lemm:oneK2one}, and~\ref{lemm:uKu}, and \eqref{eq:K_lambda_-1}, we have
\begin{align}
    \bb_0^\top(\bK+\lambda\bI)^{-1}\bb_0\lesssim ~&\norm{(\bK+\lambda\bI)^{-1/2}(\widetilde \bb_0 -\bb_0)}^2+\norm{(\bK+\lambda\bI)^{-1/2}\widetilde\bb_0}^2\\
    \lesssim ~& n\cdot \max_{i\in [n]}|\widetilde b_{0,i}-b_{0,i}|^2+\ones^\top(\bK+\lambda\bI)^{-1}\ones+ \frac{1}{d^4}  \vmu^\top\bK_{\lambda}^{-1}\vmu\lesssim 1,
\end{align}
with probability at least $1-O(d^{-1/48})$.
\end{proof}

\subsection{Proof of Theorem~\ref{thm:test_limit}}\label{sec:proof_test_random}
In this section, we analyze the asymptotic behavior of the generalization error of KRR when $f'(0)=f^{(3)}(0)=0$ in the approximated kernel \eqref{def:a1} and $f_*(\bx)=\bx^\top\bG\bx/d$ is a pure quadratic function where $\bG\in\R^{d\times d} $ is a symmetric random matrix satisfying
$
    \E[\bG_{i,j}]=0,~\E[\bG_{i,j}^2]=1$ for all $i,j\in[n]$.
Hence, under the settings of Theorem~\ref{thm:test_limit}, the prediction risk of KRR  defined in \eqref{eq:test_K} can written as 
\begin{align}
   \cR(\lambda) 
    =~&\E_{\bx,\bG} [|\vf_*(\bx)|^2]+ \Tr(\bK+\lambda\bI)^{-1} \bM (\bK+\lambda\bI)^{-1}\E_{\bG}[\vf_*\vf_*^\top] \\
    ~&+ \sigma^2_{\bepsilon} \Tr(\bK+\lambda\bI)^{-1} \bM (\bK+\lambda\bI)^{-1}   -2\Tr (\bK+\lambda\bI)^{-1}\bV  .\label{eq:decomp_test_1}
\end{align}  
where we only take expectation with respect to $\bG$, test data point $\bx$ and noise $\bepsilon$. In \eqref{eq:decomp_test_1}, $\bM$ is defined in Lemma~\ref{lemma:diff_M_v}, $\vf_*:=[f_*(\bx_1),\ldots,f_*(\bx_n)]^\top$,  with   $f_*(\bx_i)=\frac{1}{d}\bx_i^\top \bG \bx_i$ and
$\bV:= \E[\vf_*f_*(\bx)K(\bX,\bx)|\bX]\in\R^{n\times n}$,
where $K(\bX,\bx)=[K(\bx_1,\bx),\dots, K(\bx_n,\bx)] \in \R^{n}$. Notice that for any $i,j\in[n]$,
        $\bV_{i,j}= \E[K(\bx,\bx_j)f_*(\bx)f_*(\bx_i)|\bX]$.
Furthermore, Assumption~\ref{assumption:C8} provides a simpler approximation of $\bM$, and 
\begin{equation}\label{eq:simple_v(2)}
   \widetilde \bb_0= b_0\1+\frac{f^{(2)}(0)}{2d^2}\vmu,\quad \widetilde\bb_1=0, \quad a_1=0,
\end{equation} 
where $\vmu$ is defined in \eqref{eq:vmu}, and $\widetilde \bb_0$ and $\widetilde\bb_1$ are defined by \eqref{eq:tilde_b_0}.

\begin{lemma}\label{lemma:diff_v}
Under the same assumptions as Theorem~\ref{thm:concentration}, we have that
    $
        \|\bV-\bV^{(2)}\|\le ~ \frac{c }{d^{2.4}},
    $
    with probability  at least $1-O(d^{-1/48})$ for some constant $c>0$, where
    \begin{equation}
        \bV^{(2)}:=\frac{1}{d^2}(\vmu\bb_0^\top+2a_2\bM_0^{(2)})
    \end{equation} and $\bb_0$, $\bM_0^{(2)}$, and $\vmu$ are defined by \eqref{eq:b_0}, \eqref{eq:M(2)}, and \eqref{eq:vmu}.
\end{lemma}
\begin{proof}
For any $j,i\in [n]$, by the definition of $f_*(\bx)$, we have
\begin{align}
\bV_{j,i}&=\E_[K(\bx,\bx_i)f_*(\bx)f_*(\bx_j)|\bX]\\
&=\sum_{k=0}^8b_{k,i}\E_{\bG}[\E_{\bx}[ \bT_{i}^{(k)}f_*(\bx)]f_*(\bx_j)]+\E_{\bx,\bG}\left[\frac{f^{(9)}(\zeta_{i})}{9!d^9}f_*(\bx_j)f_*(\bx)\langle \bx_i,\bx \rangle^9\right]\\
&= \frac{1}{d^2}\bx_j^\top\bSigma\bx_j b_{0,i} + \frac{b_{2,i}}{d}\E_{\bG}[f_*(\bx_j)\bx_i^\top\bSigma\bG\bSigma\bx_i]+\E_{\bx,\bG}\left[\frac{f^{(9)}(\zeta_{i})}{9!d^9}f_*(\bx_j)f_*(\bx)\langle \bx_i,\bx \rangle^9\right]\\
&= \frac{1}{d^2}\bx_j^\top\bSigma\bx_j b_{0,i} + \frac{2b_{2,i}}{d^2}(\bx_j^\top\bSigma\bx_i)^2+\E_{\bx,\bG}\left[\frac{f^{(9)}(\zeta_{i})}{9!d^9}f_*(\bx_j)f_*(\bx)\langle \bx_i,\bx \rangle^9\right]
\end{align}
where in the second line we applied \eqref{eq:taylor_expansion}, Lemmas~\ref{lem:hermite} and~\ref{lem:quadratic_moments_eq2}. 
Therefore,
\begin{align}
    \|\bV-\bV^{(2)}\|\le ~&\frac{2}{d^2}\|(\bX\bSigma\bX^\top)^{\odot 2}\|\cdot\max_{i\in [n]}|a_2-b_{2,i}|+\frac{n}{d^{11}}\max_{i,j\in [n]}|\E_{\bx,\bG}[\bx^\top\bG\bx\bx_j^\top\bG\bx_j(\bx_i^\top\bx)^9]|\\
    \lesssim~& \frac{1}{d^{5.4}}\|(\bX\bSigma\bX^\top)^{\odot 2}\|+\frac{1}{d^{9}}\max_{i,j\in [n]}|\E_{\bx}[(\bx^\top\bx_j)^2 (\bx_i^\top\bx)^9]|\\
    \lesssim~& \frac{1}{d^{2.4}} +\frac{1}{d^9} \max_{i,j\in [n]}\|\bw_j\|^2\cdot\|\bw_j\|^{9}\lesssim d^{-2.4}, 
\end{align} 
with probability at least $1-O(d^{-\frac{1}{48}})$, where we utilize Lemmas~\ref{lem:bound_XX_2} and~\ref{lemm:b_k,i}, and the definition of $f_*$.  This completes the proof of the lemma.
\end{proof}
 
In the following lemma, we further approximate each term in $\bar\cR(\lambda)$. Define 
\begin{align}
    \widetilde\cR(\lambda):=~& \E [|f_*(\bx)|^2]+ \Tr(\bK+\lambda\bI)^{-1} \bM^{(2)} (\bK+\lambda\bI)^{-1}\E_{\bG}[\vf_*\vf_*^\top]\nonumber\\
    ~&+\sigma^2_{\bepsilon}\Tr(\bK+\lambda\bI)^{-1}\bM^{(2)}(\bK +\lambda\bI)^{-1}  
     -2\Tr (\bK+\lambda\bI)^{-1}\bV^{(2)}. \label{eq:decomp_test_2}
\end{align}

\begin{lemma}\label{lemma:diff_R_Rtilde}
 Under the same assumptions as Theorem~\ref{thm:train_limit}, for any $\lambda\ge 0$, we have that
$
    |\cR(\lambda)-\widetilde\cR(\lambda)|\le  c d^{{-}\frac{1}{4}} ,
$
conditioning on $\bG$ in $f_*$ defined in \eqref{eq:teacher}, with probability at least $1-O(d^{-1/48})$, for some $c>0$,  where $\cR(\lambda)$ is defined by \eqref{eq:decomp_test_1}.
\end{lemma}
\begin{proof}
Notice that
$
    \E_{\bG}[\|\vf_*\|^2]=\frac{1}{d^2}\sum_{i=1}^n \E_{\bG}[(\bx_i^\top\bG\bx_i)^2]\lesssim \max_{i\in[n]}\|\bx_i\|^4\lesssim  
    d^2$, with probability at least $1-O(d^{-1})$,
because of \eqref{eq:normxi}.
Applying Lemmas~\ref{lemma:diff_M_v} and~\ref{lemma:diff_v}, we can get
\begin{align}
    \Big|\widetilde\cR(\lambda)-\cR(\lambda)\Big|\le~& \Big|\Tr\bK_{\lambda}^{-1}(\bM^{(2)}-\bM)\bK_{\lambda}^{-1} \E_{\bG}[\vf_*\vf_*^\top ] \Big|
    +2\Big|\Tr\bK_{\lambda}^{-1}(\bV^{(2)}-\bV)\Big|\\
    ~& +\sigma^2_{\bepsilon}\Big|\Tr\bK_{\lambda}^{-1}(\bM^{(2)}-\bM)\bK_{\lambda}^{-1}\Big|\\
    \le ~&  (n\sigma^2_{\bepsilon}+\E_{\bG}[\|\vf_*\|^2])\|\bK_{\lambda}^{-1}\|^2 \|\bM^{(2)}-\bM\|+2n\|\bK_{\lambda}^{-1}\|  \|\bV^{(2)}-\bV\|  
    \lesssim  d^{\frac{-1}{4}},
\end{align}
with probability  $1-O(d^{-1/48})$, where in the last line, we utilize \eqref{eq:K_lambda_-1} and Lemma~\ref{lemm:bound_y}.
\end{proof}
Hence, below, we will analyze $\widetilde\cR(\lambda)$ instead of prediction risk $\cR(\lambda)$.

\begin{lemma}\label{lemm:risk_decomp}
    Under the assumptions of Theorem~\ref{thm:test_limit}, we have $|\widetilde\cR(\lambda)-( \sigma_\varepsilon^2\cV+ \cB)|\lesssim d^{-0.4}$ with probability at least $1-O(d^{-1/48})$, where
\begin{align}
    \cV:=~& 2a_2^2\Tr(\bK+\lambda\bI)^{-1} \bM_0^{(2)} (\bK+\lambda\bI)^{-1} \\
    \cB:=~&  \frac{2}{d^2}(\Tr\bSigma)^2 +\frac{4a_2^2}{d^2}\Tr\bK_{\lambda}^{-1}\bM_0^{(2)}\bK_{\lambda}^{-1}(\bX \bX^\top)^{\odot 2}- \frac{4a_2}{d^2}\Tr\bX^{(2)}\bSigma^{(2)}\bX^{(2)\top}\bK_{\lambda}^{-1}.
\end{align}

\end{lemma}
\begin{proof}
Recall the assumption of $\bG$ in $f_*(\bx)=\bx^\top\bG\bx/d$ from Theorem~\ref{thm:test_limit}. By taking expectation for $\bG$, we can easily simplify the expression of $ \widetilde\cR(\lambda) $. 
Notice that given any deterministic matrix $\bA\in\R^{n\times n}$, we have
\begin{align}
    \E_{\bG}[\vf_*^\top\bA\vf_*|\bX]=~& \frac{2}{d^2}\Tr\bA\bX^{(2)}\bX^{(2)\top}-\frac{1}{d^2}\sum_{k=1}^d\vnu_k^\top\bA\vnu_k, \label{eq:fAf_control} 
\end{align}where $\vnu_k\in\R^n$ are defined by Lemma~\ref{lemm:M_02_decomp}. Considering \eqref{eq:sample_covariance_equivalence}, Lemma~\ref{lemm:M_02_decomp} and \eqref{eq:simple_v(2)}, we have
\begin{align}
     \widetilde\cR(\lambda) =~& \E[|f_*(\bx)|^2]+\sigma^2_{\bepsilon}\Tr(\bK+\lambda\bI)^{-1}\bM^{(2)}(\bK +\lambda\bI)^{-1}\\
    ~&   + 2a_2^2\Tr(\bK+\lambda\bI)^{-1}\bM^{(2)}(\bK +\lambda\bI)^{-1} \E[\vf_*\vf_*^\top|\bX] 
     -2\Tr(\bK +\lambda\bI)^{-1}\bV^{(2)}\\
     =~& \E[|f_*(\bx)|^2]+2a_2^2\sigma^2_{\bepsilon}\Tr\bK_{\lambda}^{-1}\bM_0^{(2)}\bK_{\lambda}^{-1}\\
    ~&   + 2a_2^2\Tr\bK_{\lambda}^{-1}\bM_0^{(2)}\bK_{\lambda}^{-1} \E[\vf_*\vf_*^\top|\bX] 
     -\frac{4a_2}{d^2}\Tr\bK_{\lambda}^{-1}\bM_0^{(2)}\\
     ~& +  \bb_0^\top\bK_{\lambda}^{-1} \E[\vf_*\vf_*^\top|\bX]\bK_{\lambda}^{-1} \bb_0-\frac{2}{d^2}\bb_0^\top\bK_{\lambda}^{-1}\vmu\\
     ~& +\Tr\bK_{\lambda}^{-1}\diag(\bb_1-\widetilde\bb_1)\bX\bSigma\bX^\top\diag(\bb_1-\widetilde\bb_1) \bK_{\lambda}^{-1} \E[\vf_*\vf_*^\top|\bX] \\
    ~& +\sigma^2_{\bepsilon}\Tr\bK_{\lambda}^{-1}\bb_0\bb_0^\top\bK_{\lambda}^{-1}+\sigma^2_{\bepsilon}\Tr\bK_{\lambda}^{-1}\diag(\bb_1-\widetilde\bb_1)\bX\bSigma\bX^\top\diag(\bb_1-\widetilde\bb_1)\bK_{\lambda}^{-1}\\
    =~& \sigma^2_{\bepsilon}\cV+\cB+\cR_{\mix}-J_1+J_2,
\end{align}
where
\begin{align}
\cR_{\mix}:=~& \frac{1}{d^2} \Tr(\bSigma^2)+ \bb_0^\top\bK_{\lambda}^{-1} \E[\vf_*\vf_*^\top|\bX]\bK_{\lambda}^{-1} \bb_0-\frac{2}{d^2}\bb_0^\top\bK_{\lambda}^{-1}\vmu\\
     ~& +\Tr\bK_{\lambda}^{-1}\diag(\bb_1-\widetilde\bb_1)\bX\bSigma\bX^\top\diag(\bb_1-\widetilde\bb_1) \bK_{\lambda}^{-1} \E[\vf_*\vf_*^\top|\bX] \\
    ~& +\sigma^2_{\bepsilon}\Tr\bK_{\lambda}^{-1}\bb_0\bb_0^\top\bK_{\lambda}^{-1}+\sigma^2_{\bepsilon}\Tr\bK_{\lambda}^{-1}\diag(\bb_1-\widetilde\bb_1)\bX\bSigma\bX^\top\diag(\bb_1-\widetilde\bb_1)\bK_{\lambda}^{-1}\\
    J_1:=~& \frac{2a_2^2}{d^2}\sum_{k=1}^d\vnu_k^\top\bK_{\lambda}^{-1}\bM_0^{(2)}\bK_{\lambda}^{-1}\vnu_k, \quad 
    J_2:=~ \frac{4a_2}{d^2}\sum_{k=1}^d\bSigma_{kk}^2\vnu_k^\top\bK_{\lambda}^{-1}\vnu_k.
\end{align}
Here, we use $\bM^{(2)}=\bb_0\bb_0^\top+ \diag(\bb_1-\widetilde\bb_1)\bX\bSigma\bX^\top\diag(\bb_1-\widetilde\bb_1) +2a_2^2\bM_0^{(2)}$, and $\bb_0,\bb_1,\widetilde\bb_0$, and $\widetilde\bb_1$ are defined in \eqref{eq:b_0} and \eqref{eq:tilde_b_0}. Notice that $\widetilde\bb_1=0$. Thus, It suffices to control $J_1,J_2$ and $\cR_{\mix}$ below. Notice that with probability  $1-d^{-1}$,  due to Lemmas~\ref{lemm:inverseK_mix_bound} and~\ref{lemm:M_02_decomp}, and \eqref{eq:K_lambda_-1},
\begin{align}
J_1\lesssim ~&  \frac{1}{d^4}\sum_{k=1}^d\vnu_k^\top\bK_{\lambda}^{-1} \vnu_k\lesssim d^{-\frac{10}{11}}.
\end{align}  Similarly, we have $J_2\lesssim d^{-\frac{10}{11}}$ as well. Next, we further decompose $\cR_{\mix}$ as 
\begin{align}
\cR_{\mix}=~&\cR_{\mix}^{(0)}+\cR_{\mix}^{(1)}+\cR_{\mix}^{(2)}, \quad 
    \cR_{\mix}^{(0)}:=~\frac{1}{d^2} \Tr(\bSigma^2)+ \sigma^2_{\bepsilon} \bb_0^\top\bK_{\lambda}^{-2}\bb_0-\frac{2}{d^2}\bb_0^\top\bK_{\lambda}^{-1}\vmu,\\
    \cR_{\mix}^{(1)}:=~& \bb_0^\top\bK_{\lambda}^{-1} \E[\vf_*\vf_*^\top|\bX]\bK_{\lambda}^{-1} \bb_0,\\
    \cR_{\mix}^{(2)}:=~& \Tr\bK_{\lambda}^{-1}\diag(\bb_1-\widetilde\bb_1)\bX\bSigma\bX^\top\diag(\bb_1-\widetilde\bb_1) \bK_{\lambda}^{-1} (\sigma^2_{\bepsilon} \bI+\E[\vf_*\vf_*^\top|\bX] ).
\end{align}
Based on Assumption~\ref{assump:sigma} and Lemmas~\ref{lemm:oneK2one} and~\ref{lemm:uKu}, we can verify that $|\cR_{\mix}^{(0)}|\lesssim d^{-0.4}$  with probability at least $1-O(d^{-1/48})$. From \eqref{eq:fAf_control}, we know that 
$\E[\vf_*\vf_*^\top|\bX]=\frac{1}{d^2} \bX^{(2)}\bD_*\bX^{(2)\top}$, 
where $\bD_*\in\R^{\binom{d+1}{2}\times\binom{d+1}{2}}$ is a diagonal matrix with 
\begin{align}  
(\bD_*)_{ij, k\ell} =\begin{cases}
    0 & \text{ if } (i,j)\not=(k,\ell),\\
    2  &\text{ if } i\not=j, (i,j)=(k,\ell), \\
    1 &\text{ if } i=j=k=\ell.
\end{cases}
\end{align} 
Hence, $\bD_*\preccurlyeq 2\bI$ and 
\begin{equation}\label{eq:bound_ff^T}
    \E[\vf_*\vf_*^\top|\bX]\preccurlyeq \frac{2}{d^2} \bX^{(2)} \bX^{(2)\top}.
\end{equation}
Then by Lemma~\ref{lemm:oneK2one},
 $|\cR_{\mix}^{(1)}|\lesssim \frac{1}{d^2}\bb_0^\top\bK_{\lambda}^{-1}\bX^{(2)} \bX^{(2)\top}\bK_{\lambda}^{-1} \bb_0\lesssim a_2\bb_0^\top \bK_{\lambda}^{-1}(\bX\bX^\top)^{\odot 2} \bK_{\lambda}^{-1}\bb_0$.
Then, \eqref{eq:decomp:X^2} allows us to get
$ |\cR_{\mix}^{(1)}|\lesssim  \bb_0^\top \bK_{\lambda}^{-1}\bK_*^{(2)}\bK_{\lambda}^{-1}\bb_0+\bb_0^\top\bK_{\lambda}^{-1}\bU\bD\bU^\top\bK_{\lambda}^{-1}\bb_0,$
where $\bK_*^{(2)}$ is defined in \eqref{eq:K_*(2)}. Hence, Lemmas~\ref{lem:bound_XX_2} and~\ref{lemm:oneK2one} imply 
\begin{align}
   \bb_0^\top\bK_{\lambda}^{-1}\bK_*^{(2)}\bK_{\lambda}^{-1}\bb_0\lesssim \bb_0^\top\bK_{\lambda}^{-2} \bb_0\lesssim d^{-0.8}
\end{align}  
with probability at least $1-O(d^{-1/48})$. Then, recall \eqref{eq:decomp:X^2} and Lemma~\ref{lemm:uKu}. We can apply the Cauchy-Schwarz inequality again to get
\begin{align}
&|\bb_0^\top\bK_{\lambda}^{-1}\bU\bD\bU^\top\bK_{\lambda}^{-1}\bb_0|\\\le~&  a_2|\bb_0^\top\bK_{\lambda}^{-1}\1|\cdot\big(\Tr(\bSigma^2)|\bb_0^\top\bK_{\lambda}^{-1}\1|+|\bb_0^\top\bK_{\lambda}^{-1}\vmu|\big)\\
\lesssim~&  \frac{1}{d^2}\Tr(\bSigma^2)\cdot(\bb_0^\top\bK_{\lambda}^{-1}\bb_0)(\1^\top\bK_{\lambda}^{-1}\1) +(\bb_0^\top\bK_{\lambda}^{-1}\bb_0)^{\frac{1}{2}}\big(\frac{1}{d^4}\vmu^\top\bK_{\lambda}^{-1}\vmu\big)^{\frac{1}{2}} 
\lesssim~ d^{-0.4},
\end{align} 
with probability at least $1-O(d^{-1/48})$. Lastly, because of \eqref{eq:K_lambda_-1} and~\eqref{eq:bound_ff^T}, we have
\begin{align}
    |\cR_{\mix}^{(2)}|\lesssim ~& d \cdot \|\diag(\bb_1-\widetilde\bb_1)\|^2\|\bX\bX^\top\|(\sigma^2_{\bepsilon}+\frac{2}{d^2} \|\bX^{(2)} \bX^{(2)\top}\|)
    \lesssim \frac{1}{d}
\end{align}
with probability at least $1-O(d^{-1/48})$, where we apply Lemma~\ref{lemm:b_k,i} for $\|\diag(\bb_1-\widetilde\bb_1)\|$ and Lemma~\ref{lem:bound_XX_2} for $\|\bX\bX^\top\|$ and $\|\bX^{(2)} \bX^{(2)\top}\|$.
\end{proof}

\begin{lemma}\label{lemm:approx_variance}
Denote by
$\cV_0:= a_2^2\Tr\big(a_2\bX^{(2)}\bX^{(2)\top}+(\lambda+a )\bI\big)^{-2}  {\bX}^{(2)}\bSigma^{(2)} {\bX}^{(2)\top}$.  
Under the assumptions of Theorem~\ref{thm:test_limit}, there exist some constants $c,C>0$ such that 
$ 
    \left|\cV-\cV_0\right|\le C d^{-\frac{1}{12}},
$
with probability at least $1-cd^{-\frac{1}{48}}$ for all large $d$ and $n$, and some constant $c>0$.
\end{lemma}
\begin{proof}
Denote that $\bK _{\lambda,(2)}:= (\bK^{(2)}+\lambda\bI) $. Because of \eqref{eq:K_lambda_-1}, we know that $\|\bK _{\lambda,(2)}^{-1}\|\lesssim 1$ and $ \|\bK _{\lambda}^{-1}\|\lesssim 1$. Denote by $\cV^{(2)}:=  2a_2^2\Tr\bK _{\lambda,(2)}^{-1} \bM_0^{(2)} \bK _{\lambda,(2)}^{-1}$. We first control
\begin{align}
    \left|\cV-\cV^{(2)}\right|\lesssim  \frac{a_2}{d^2}|\Tr(\bK_{\lambda}^{-1}- \bK^{-1}_{\lambda,(2)})\bM_0^{(2)} \bK_{\lambda}^{-1} |+\frac{a_2}{d^2}|\Tr \bK^{-1}_{\lambda,(2)}\bM_0^{(2)} (\bK_{\lambda}^{-1}- \bK^{-1}_{\lambda,(2)})|.\label{eq:V-V(2)1}
\end{align}
Notice that
\begin{align}
    \frac{a_2}{d^2}|\Tr(\bK_{\lambda}^{-1}- \bK^{-1}_{\lambda,(2)})\bM_0^{(2)} \bK_{\lambda}^{-1} |=~&\frac{a_2}{d^2}|\Tr\bK^{-1}_{\lambda,(2)}(\bK^{(2)}-\bK)\bK_{\lambda}^{-1}\bM_0^{(2)} \bK_{\lambda}^{-1} |\\
    \lesssim~& \frac{1}{d^2}\|\bK^{(2)}-\bK\|\cdot |\Tr \bK_{\lambda}^{-1}(a_2\bM_0^{(2)}) \bK_{\lambda}^{-1} |\\
    \lesssim~& d^{-\frac{1}{12}}\cdot \frac{n}{d^2}\norm{\bK_{\lambda}^{-1}(a_2\bX\bX^\top)^{\odot 2} \bK_{\lambda}^{-1}}\lesssim d^{-\frac{1}{12}},\label{eq:V-V(2)2}
\end{align}
with probability at least $1-O(d^{-1/2})$, where we apply Lemma~\ref{lemm:inverseK_mix_bound}  and 
 Theorem~\ref{thm:concentration}. We can get a similar argument for the second term:
 \begin{align}
     \frac{a_2}{d^2}|\Tr \bK^{-1}_{\lambda,(2)}\bM_0^{(2)} (\bK_{\lambda}^{-1}- \bK^{-1}_{\lambda,(2)})|\le ~& \frac{a_2}{d^2}|\Tr \bK^{-1}_{\lambda,(2)}\bM_0^{(2)} \bK^{-1}_{\lambda,(2)}(\bK - \bK^{(2)})\bK_{\lambda}^{-1}|\\
     \lesssim~& d^{-\frac{1}{12}}.\label{eq:V-V(2)3}
 \end{align}

Next, we approximate $\cV^{(2)}$ by $\cV_0$. Let us denote by $\cV_0^{(2)}:=a_2^2\Tr\bK _{\lambda,(2)}^{-2} \bX^{(2)}\bSigma^{(2)}\bX^{(2)\top}$. From Lemma~\ref{lemm:M_02_decomp}, we know that
$
    \cV^{(2)}= \cV_0^{(2)}- \sum_{k=1}^d\bSigma_{kk}^2a_2^2 \vnu_k^\top  \bK _{\lambda,(2)}^{-2} \vnu_k$, where the second term on the right-hand side satisfies
\begin{align}
    \left|\sum_{k=1}^d\bSigma_{kk}^2a_2^2 \vnu_k^\top  \bK _{\lambda,(2)}^{-2} \vnu_k\right|\lesssim ~& \frac{1}{d^3} \max_{k\in [d]} \vnu_k^\top  \bK _{\lambda,(2)}^{-2} \vnu_k \lesssim \frac{1}{d^3} \max_{k\in [d]} \|\vnu_k\|^2 \lesssim d^{-\frac{10}{11}},\label{eq:V-V(2)4}
\end{align} with probability at least $1-d^{-1}$. Thus, it suffices to control the difference between $\cV_0^{(2)}$ and $\cV_0$. Notice that 
$
   \cV_0^{(2)} =  a_2^2\Tr \big(a_0\1\1^\top+\bK_*\big)^{-2}  \bX^{(2)}\bSigma^{(2)}\bX^{(2)\top}$, where we define
\begin{equation}\label{eq:K_*_no_11}
    \bK_*:=a_2\bX^{(2)}\bX^{(2)\top}+(\lambda+a )\bI.
\end{equation}
Analogously to the proof of Lemma~\ref{lemm:oneK2one}, the Sherman-Morrison-Woodbury formula implies
$\big(a_0\1\1^\top+\bK_*\big)^{-1} =\bK_*^{-1}-a_0\frac{\bK_*^{-1}\1\1^\top\bK_*^{-1}}{1+a_0\1^\top\bK_*^{-1}\1}.
$ Thus, we have
\begin{small}
\begin{align}
   &\cV_0^{(2)} = \cV_0\\
  &   +\frac{a_2^2 (a_0\1^\top \bK_*^{-2}\1)\cdot (a_0\1^\top \bK_*^{-1}\bX^{(2)}\bSigma^{(2)}\bX^{(2)\top}\bK_*^{-1}\1)}{(1+a_0\1^\top\bK_*^{-1}\1)^2} -\frac{2a_2^2\cdot a_0\1^\top \bK_*^{-1}\bX^{(2)}\bSigma^{(2)}\bX^{(2)\top}\bK_*^{-2}\1 }{1+a_0\1^\top\bK_*^{-1}\1}.
\end{align}
\end{small}
Hence, we only need to control the last two terms on the right-hand side of the above equation.
By Assumption~\ref{assump:analytic} and Lemma~\ref{lem:bound_XX_2}, we know 
$cd^{-1}\bI\preccurlyeq \bK_*^{-1}\preccurlyeq C\bI$,
 with probability at least $1-O(d^{-1/48})$, for some constants $c,C>0$. And Lemma~\ref{lemm:inverseK_mix_bound} indicates that
\[a_2\bK_*^{-1/2}\bX^{(2)}\bSigma^{(2)}\bX^{(2)\top}\bK_*^{-1/2}\preccurlyeq C\cdot a_2\bK_*^{-1/2}(\bX\bX^\top)^{\odot 2} \bK_*^{-1/2}\preccurlyeq C.\]
Therefore, 
\begin{align}
    &\frac{a_2^2 (a_0\1^\top \bK_*^{-2}\1)\cdot (a_0\1^\top \bK_*^{-1}\bX^{(2)}\bSigma^{(2)}\bX^{(2)\top}\bK_*^{-1}\1) }{(1+a_0\1^\top\bK_*^{-1}\1)^2}\\
    = ~& \frac{ a_0\1^\top \bK_*^{-1}\1 }{1+a_0\1^\top\bK_*^{-1}\1} \frac{a_2\cdot ( a_0\1^\top \bK_*^{-1}(a_2\bX^{(2)}\bSigma^{(2)}\bX^{(2)\top})\bK_*^{-1}\1) }{ 1+a_0\1^\top\bK_*^{-1}\1 }\\
    \le~& Ca_2\cdot\frac{ a_0\1^\top\bK_*^{-1}\1}{{ 1+a_0\1^\top\bK_*^{-1}\1 }}\lesssim \frac{1}{d^2}.
\end{align} 
Similarly, we have
$\frac{2a_2^2\cdot a_0\1^\top \bK_*^{-1}\bX^{(2)}\bSigma^{(2)}\bX^{(2)\top}\bK_*^{-2}\1 }{1+a_0\1^\top\bK_*^{-1}\1}\le 2Ca_2\frac{ a_0\1^\top \bK_*^{-1}\1 }{1+a_0\1^\top\bK_*^{-1}\1}  \lesssim \frac{1}{d^2}$. Hence, we complete the proof of this lemma.
\end{proof}

\begin{lemma}\label{lemm:approx_bias}
Denote $$\cB_0:= \frac{2}{d^2 } \Tr\bSigma^{(2)}  + \frac{ 2a_2^2}{d^2}\Tr \bK_{*}^{-1}\bX^{(2)}\bSigma^{(2)}\bX^{(2)\top}\bK_{*}^{-1}(\bX\bX^\top)^{\odot 2}-\frac{4a_2}{d^2}\Tr\bK_{*}^{-1}\bX^{(2)}\bSigma^{(2)}\bX^{(2)\top}
$$ 
where $\bK_*$ is defined in \eqref{eq:K_*_no_11}. Under the assumptions of Theorem~\ref{thm:concentration}, there exist some constants $c,C>0$ such that 
$
    \left|\cB-\cB_0\right|\le C d^{-\frac{1}{12}},
$
with probability at least $1-cd^{-\frac{1}{48}}$.
\end{lemma}
\begin{proof}
Recall $\bK_{\lambda,(2)} = (\bK^{(2)}+\lambda\bI) $ and the definition of $\cB$ in Lemma~\ref{lemm:risk_decomp}. Define
\[\cB^{(2)}:=\frac{2}{d^2}(\Tr\bSigma)^2 + \frac{4a_2^2}{d^2}\Tr \bK_{\lambda,(2)}^{-1}\bM_0^{(2)}\bK_{\lambda,(2)}^{-1} (\bX\bX^\top)^{\odot 2}-\frac{4a_2}{d^2}\Tr\bK_{\lambda,(2)}^{-1}\bX^{(2)}\bSigma^{(2)}\bX^{(2)\top}.\]
Then, following the same analysis as \eqref{eq:V-V(2)1}, \eqref{eq:V-V(2)2}, and \eqref{eq:V-V(2)3}, we can obtain that
\begin{align}
       |\cB^{(2)}- \cB|
       \lesssim ~&   \frac{a_2^2}{d^2}|\Tr(\bK_{\lambda}^{-1}-\bK_{\lambda,(2)}^{-1})\bM_0^{(2)}\bK_{\lambda}^{-1}(\bX\bX^\top)^{\odot 2}|\\
       ~&+\frac{a_2^2}{d^2}|\Tr \bK_{\lambda,(2)}^{-1} \bM_0^{(2)}(\bK_{\lambda}^{-1}-\bK_{\lambda,(2)}^{-1})(\bX\bX^\top)^{\odot 2}|\\
       ~&+\frac{ a_2}{d^2}|\Tr\bX^{(2)}\bSigma^{(2)}\bX^{(2)\top}(\bK_{\lambda}^{-1}-\bK_{\lambda,(2)}^{-1})|\\
       \lesssim ~&  \|\bK-\bK^{(2)}\|\cdot\big(a_2^2 \|\bK_{\lambda}^{-1}\bM_0^{(2)}\bK_{\lambda,(2)}^{-1}(\bX\bX^\top)^{\odot 2}\bK_{\lambda}^{-1}\|\\
       &+a_2^2 \|\bK_{\lambda,(2)}^{-1}\bM_0^{(2)}\bK_{\lambda,(2)}^{-1}(\bX\bX^\top)^{\odot 2}\bK_{\lambda}^{-1}\|+a_2 \|\bK_{\lambda }^{-1}\bX^{(2)}\bSigma^{(2)}\bX^{(2)\top}\bK_{\lambda,(2)}^{-1} \|\big)\\
       \lesssim ~& d^{-\frac{1}{12}},
\end{align}with probability at least $1-O(d^{-1/2})$,
where we apply Theorem~\ref{thm:concentration} and Lemma~\ref{lemm:inverseK_mix_bound}. Next, we apply Lemma~\ref{lemm:M_02_decomp} and define
\begin{align}
    \cB^{(2)} =~&\cB_0^{(2)}-\bDelta_{\cB},\\
    \cB_0^{(2)}:=~& \frac{2}{d^2 } \Tr\bSigma^{(2)}  + \frac{ 2a_2^2}{d^2}\Tr \bK_{\lambda,(2)}^{-1}\bX^{(2)}\bSigma^{(2)}\bX^{(2)\top}\bK_{\lambda,(2)}^{-1} (\bX\bX^\top)^{\odot 2}\\
    ~&-\frac{4a_2}{d^2}\Tr\bK_{\lambda,(2)}^{-1}\bX^{(2)}\bSigma^{(2)}\bX^{(2)\top},\\
    \bDelta_{\cB} :=~& \frac{4\Tr(\bSigma^2)}{d^2}+\frac{ a_2^2}{d^2}\sum_{k=1}^d \bSigma_{kk}^2\vnu_k^\top\bK_{\lambda,(2)}^{-1} (\bX\bX^\top)^{\odot 2}\bK_{\lambda,(2)}^{-1} \vnu_k.
\end{align}
Then, analogously to \eqref{eq:V-V(2)4}, we can have
$ |\bDelta_{\cB} |\lesssim \frac{4\Tr(\bSigma^2)}{d^2}+\frac{ a_2 }{d^2}\sum_{k=1}^d \vnu_k^\top\bK_{\lambda,(2)}^{-1}   \vnu_k\lesssim d^{-\frac{10}{11}}$.
with probability at least $1-O(d^{-1})$. Finally, the difference between $\cB_0^{(2)}$ and $\cB_0$ can be controlled similar as the bound of $|\cV_0-\cV_0^{(2)}|$ from the proof of Lemma~\ref{lemm:approx_variance}. We ignore the details for the last step here.
\end{proof}

\begin{proofoftheorem}{\ref{thm:test_limit}}
Based on all above Lemmas~\ref{lemma:diff_R_Rtilde},~\ref{lemm:risk_decomp},~\ref{lemm:approx_variance}, and \ref{lemm:approx_bias}, we have already known that $|\cR_0-\cR(\lambda)|\to 0$ in probability, as $d^2/(2n)\to\alpha$ and  $d\to\infty$, where $\cR_0:=\sigma_{\bepsilon}^2\cV_0+\cB_0$.
Here $\cV_0$ and $\cB_0$ are defined in Lemmas~\ref{lemm:approx_variance}, and \ref{lemm:approx_bias}, respectively. Hence, to prove Theorem~\ref{thm:test_limit}, it suffices to analyze the asymptotic behavior of $\cR_0$, as $d^2/(2n)\to\alpha$ and $d\to\infty$. Recall the definition of $\bK_*$ in \eqref{eq:K_*_no_11} and $(\bX\bX^\top)^{\odot 2} = \bX^{(2)}\bX^{(2)\top}$. As $d\to\infty$ and $d^2/(2n)\to\alpha\in(0,\infty)$, it is easy to check that
\begin{small}
\begin{align}
    \cB_0 =~&   \frac{2}{d^2 } \Tr\bSigma^{(2)}  + \frac{ 2a_2^2}{d^2}\Tr \bK_{*}^{-1}\bX^{(2)}\bSigma^{(2)}\bX^{(2)\top}\bK_{*}^{-1}(\bX\bX^\top)^{\odot 2}-\frac{4a_2}{d^2}\Tr\bK_{*}^{-1}\bX^{(2)}\bSigma^{(2)}\bX^{(2)\top}\\
    =~& \frac{2}{d^2}\Tr \big(\bI-a_2\bX^{(2)\top}\bK_{*}^{-1}\bX^{(2)}\big)\bSigma^{(2)}\big(\bI-a_2\bX^{(2)\top}\bK_{*}^{-1}\bX^{(2)}\big)\\
    =~& \frac{2(a+\lambda)^2}{d^2}\Tr\big((a +\lambda)\bI+a_2\bX^{(2)\top}\bX^{(2)}\big)^{-1}\bSigma^{(2)} \big((a +\lambda)\bI+a_2\bX^{(2)\top}\bX^{(2)}\big)^{-1}\\
     =~& \frac{2(a_*+\lambda)^2}{d^2}\Tr\big((a +\lambda)\bI+a_2\overline\bX^{(2)\top}\overline\bX^{(2)}\big)^{-1}\bSigma^{(2)} \big((a +\lambda)\bI+a_2\overline\bX^{(2)\top}\overline\bX^{(2)}\big)^{-1}+o(1),
\end{align}
\end{small}
and
\begin{align}
    \cV_0 =~&a_2^2\Tr\big(a_2\bX^{(2)}\bX^{(2)\top}+(\lambda+a )\bI\big)^{-2}  {\bX}^{(2)}\bSigma^{(2)} {\bX}^{(2)\top}\label{eq:cV_0}\\
    =~& a_2 \Tr\big((a +\lambda)\bI+a_2\bX^{(2)\top}\bX^{(2)}\big)^{-1}\bSigma^{(2)} \big((a+\lambda)\bI+a_2\bX^{(2)\top}\bX^{(2)}\big)^{-1}\big(a_2\bX^{(2)\top}\bX^{(2)}\big)\\=~&  a_2 \Tr\big((a +\lambda)\bI+a_2\bX^{(2)\top}\bX^{(2)}\big)^{-1}\bSigma^{(2)}\\
    ~&-  a_2(a+\lambda)\Tr\big((a+\lambda)\bI+a_2\bX^{(2)\top}\bX^{(2)}\big)^{-1}\bSigma^{(2)} \big((a+\lambda)\bI+a_2\bX^{(2)\top}\bX^{(2)}\big)^{-1}\\
    =~&  a_2 \Tr\big((a_*+\lambda)\bI+a_2\overline\bX^{(2)\top}\overline\bX^{(2)}\big)^{-1}\bSigma^{(2)}\\
    ~&- a_2(a_*+\lambda)\Tr\big((a+\lambda)\bI+a_2\overline\bX^{(2)\top}\overline\bX^{(2)}\big)^{-1}\bSigma^{(2)} \big((a_*+\lambda)\bI+a_2\overline\bX^{(2)\top}\overline\bX^{(2)}\big)^{-1}+o(1),
\end{align} 
where \(\bSigma^{(2)}\) is the population covariance matrix of $\bx_i^{(2)}$ defined in \eqref{eq:defSigma2}. Recall that $\bSigma^{(2)}$ has a limiting spectral distribution $\mu_{\bSigma^{(2)}}$ as $d^2/(2n)\to\alpha$ and $n\to\infty$. Therefore, we can apply Lemma~\ref{eq:deterministic_equ} to conclude this theorem.
 
\end{proofoftheorem}

\subsection{Proof of Theorem~\ref{thm:test_limit_deterministic}}
Following the same notions in Section~\ref{sec:quadratic_appro}, in the setting of Theorem~\ref{thm:test_limit_deterministic}, we know that  
\begin{align}\label{eq:decomp_test_3}
    \cR(\lambda)
    =~&\E_{\bx} [|\vf_*(\bx)|^2]+ \vf_*^\top (\bK+\lambda\bI)^{-1} \bM (\bK+\lambda\bI)^{-1}\vf_*\\
    ~&+ \sigma^2_{\bepsilon} \Tr(\bK+\lambda\bI)^{-1} \bM (\bK+\lambda\bI)^{-1}   -2\bv^\top  (\bK+\lambda\bI)^{-1}\vf_* . 
\end{align}
Let us redefine that
\begin{align}
    \bv^{(2)}:=~& \frac{1}{d}\Tr(\bSigma^2) \bb_0   +\frac{2 a_2}{d}\bv^{(2)}_0,\quad
    \bv^{(2)}_0: =[\bx_1^\top\bSigma^3\bx_1,\ldots,\bx_n^\top\bSigma^3\bx_n]^\top.\label{eq:v_02}
\end{align}
In the following, we first provide the approximations of $\bv$ in terms of $\bv^{(2)}$. And analogously to Lemma~\ref{lemma:diff_R_Rtilde}, in the following, we will use 
\begin{align}\label{eq:decomp_test_3_appro}
   \widetilde \cR(\lambda)
    =~&\E_{\bx} [|\vf_*(\bx)|^2]+ \vf_*^\top (\bK+\lambda\bI)^{-1} \bM^{(2)}(\bK+\lambda\bI)^{-1}\vf_*\\
    ~&+ \sigma^2_{\bepsilon} \Tr(\bK+\lambda\bI)^{-1} \bM^{(2)} (\bK+\lambda\bI)^{-1}   -2\bv^{(2)\top}  (\bK+\lambda\bI)^{-1}\vf_*
\end{align}
to approximate generalization error $\cR(\lambda)$. 
Notice that, under the assumptions of Theorem~\ref{thm:test_limit_deterministic}, $\vf_*=\frac{1}{d}\vmu$ where $\vmu$ is defined by \eqref{eq:vmu}, and 
\[\bM^{(2)} = \bb_0\bb_0^\top+\diag(\bb_1-\widetilde\bb_1)\bX\bSigma\bX^\top\diag(\bb_1-\widetilde\bb_1)+2a_2^2\bM_0^{(2)}.\]

\begin{lemma}\label{lemma:diff_v_det}
Under the same assumptions as Theorem~\ref{thm:concentration}, we have that
        $\|\bv-\bv^{(2)}\|\le ~ \frac{c }{d^{2}}$,
    with probability  at least $1-O(d^{-1})$ for some constant $c>0$.
\end{lemma}
\begin{proof}
For any $i\in [n]$, by the definition of $f_*(\bx)$ and \eqref{eq:taylor_expansion}, we have
\begin{align}
\bv_i=\E_{\bx}[K(\bx,\bx_i)f_*(\bx)]
&=\sum_{k=0}^8b_{k,i}\E_{\bx}[ \bT_{i}^{(k)}f_*(\bx)]+\E_{\bx}\left[\frac{f^{(9)}(\zeta_{i})}{9!d^9}f_*(\bx)\langle \bx_i,\bx \rangle^9\right]\\
&=\frac{b_{0,i}}{d}\Tr(\bSigma^2) + \frac{2b_{2,i}}{d}\bx_i^\top\bSigma^3\bx_i+\E_{\bx}\left[\frac{f^{(9)}(\zeta_{i})}{9!d^9}f_*(\bx)\langle \bx_i,\bx \rangle^9\right]
\end{align}
where in the second line we applied Lemmas~\ref{lem:quadratic_moments_eq2} and~\ref{lem:hermite}. Notice that
\begin{align}\label{eq:wi_quard}
0<\bx_i^\top\bSigma^3\bx_i=\bw_i^\top\bSigma^{2} \bw_i\le \norm{\bw_i}^2\norm{\bSigma}^2\lesssim d^{1+\frac{1}{15}},
\end{align}with probability at least $1-d^{-1}$ for all $i\in[n]$,
where we applied \eqref{eq:wi}.
Therefore,
\begin{align}
    \|\bv-\bv^{(2)}\|\le ~&\frac{2}{d}\|\bv_0^{(2)}\|\cdot\max_{i\in [n]}|a_2-b_{2,i}|+\frac{C}{d^9}\E_{\bx}[\|(\bX\bx)^{\odot 9} f_*(\bx)\|]\\
    \lesssim~& \frac{1}{d^{4.4}}\|\bv_0^{(2)}\|+\frac{1}{d^9} \cdot \E[\|(\bX\bx)^{\odot 9}\|^2]^{1/2}\E[f_*(\bx)^2]^{1/2}\\
    \lesssim~& \frac{\sqrt{n}}{d^{4.4}}\max_{i\in [n]}\bx_i^\top\bSigma^3\bx_i+\frac{\sqrt{n}}{d^9} \max_{i\in [n]}\norm{\bw_i}^9\lesssim d^{-2.3}, 
\end{align} 
with probability at least $1-O(d^{-1})$, where we utilize \eqref{eq:wi_quard}, \eqref{eq:wi}, Lemma~\ref{lemm:b_k,i}, and the definition of $f_*$. This completes the proof of the lemma.
\end{proof}

\begin{lemma}\label{lemma:diff_R_Rtilde_det}
 Under the same assumptions as Theorem~\ref{thm:test_limit_deterministic}, for any $\lambda\ge 0$, we have that
   $ | \cR(\lambda)-\widetilde\cR(\lambda)|\lesssim d^{{-}\frac{1}{4}}$,
with probability  $1-O(d^{-1/48})$,  where $\cR(\lambda)$ is defined by \eqref{eq:decomp_test_3}.
\end{lemma}
\begin{proof}
Since $\vf_*=\frac{1}{d}\vmu$, \eqref{eq:vmu-concen_norm} implies that $\norm{\vf_*}\lesssim d$ with probability at least $1-O(d^{-1})$. Then, applying Lemmas~\ref{lemma:diff_M_v} and~\ref{lemma:diff_v_det}, we can get
\begin{small}
\begin{align}
    \Big|\widetilde\cR(\lambda)- \cR(\lambda)\Big|\le~& \Big|\vf_*^\top \bK_{\lambda}^{-1}(\bM^{(2)}-\bM)\bK_{\lambda}^{-1} \vf_* \Big|
    +2\Big|\vf_*^\top\bK_{\lambda}^{-1}(\bv^{(2)}-\bv)\Big|\\
    ~& +\sigma^2_{\bepsilon}\Big|\Tr\bK_{\lambda}^{-1}(\bM^{(2)}-\bM)\bK_{\lambda}^{-1}\Big|\\
    \le ~& (n\sigma^2_{\bepsilon}+\|\vf_*\|^2)\|\bK_{\lambda}^{-1}\|^2\cdot \|\bM^{(2)}-\bM\|+2\|\vf_*\|\cdot\|\bK_{\lambda}^{-1}\| \cdot \|\bv^{(2)}-\bv\|  
    \lesssim  d^{\frac{-1}{4}},
\end{align}
\end{small}
with probability at least $1-O(d^{-1/48})$, where in the last line, we also utilize \eqref{eq:K_lambda_-1}.
\end{proof}

Notice that $\widetilde\cR(\lambda)$ defined in \eqref{eq:decomp_test_3_appro} can be further decomposed by 
\begin{equation}\label{eq:decomp_test_3_appro_decomp}
    \widetilde\cR(\lambda) = \sigma_\varepsilon^2\cV+\cR_1+\cR_2 +\cR_{\mix},
\end{equation}  
where $\cV$ is defined in Lemma~\ref{lemm:risk_decomp}, and we redefine the terms:
\begin{align}
\cR_1:=~& \big(d^{-1}\Tr(\bSigma^2)-(a_2\vmu+a_0\1)^\top\bK_{\lambda}^{-1} \vf_*)^2\label{eq:bias1_deterministic}\\
\cR_2:=~& \frac{2}{d^2}\Tr( \bSigma^4) + 2a_2^2\vf_*^\top \bK_{\lambda}^{-1}(\bX\bSigma\bX^\top)^{\odot 2}\bK_{\lambda}^{-1} \vf_*-\frac{4a_2}{d}\bv_0^{(2)}\bK_{\lambda}^{-1} \vf_* \label{eq:bias_deterministic}\\
\cR_{\mix}:=~&   \vf_*^\top \bK_{\lambda}^{-1}(\bb_0\bb_0^\top-\widetilde\bb\widetilde\bb_0^\top)\bK_{\lambda}^{-1} \vf_*\\
~&+\vf_*^\top \bK_{\lambda}^{-1}\diag(\bb_1-\widetilde\bb_1)\bX\bSigma\bX^\top\diag(\bb_1-\widetilde\bb_1)\bK_{\lambda}^{-1} \vf_* -2\frac{\Tr( \bSigma^2)}{d}(\bb_0-\widetilde\bb )^\top\bK_{\lambda}^{-1} \vf_*\\
~&+\sigma^2_{\bepsilon}\Tr\bK_{\lambda}^{-1}\bb_0\bb_0^\top\bK_{\lambda}^{-1}+\sigma^2_{\bepsilon}\Tr\bK_{\lambda}^{-1}\diag(\bb_1-\widetilde\bb_1)\bX\bSigma\bX^\top\diag(\bb_1-\widetilde\bb_1)\bK_{\lambda}^{-1}.\label{eq:bias2_deterministic}
\end{align}
Here, we denote \begin{align}\label{eq:def_tilde_b}
\widetilde\bb:=a_2\vmu+a_0\1,
\end{align}
and $ \bb_0, \bb_1,$  $\widetilde\bb_1$ are defined in Lemma~\ref{lemm:b_k,i}.
The analysis of $\cV$ is the same as the proof of Theorem~\ref{thm:test_limit}. Now recall some notations introduced in Section~\ref{sec:resolvent_calc}. We denote by
\begin{align}
    \bU=~&[\1,\vmu]\in\R^{n\times 2}\label{def:U}\\
    \bD:=~& \begin{pmatrix}
        a_0-a_2 \Tr(\bSigma^2) & a_2 \\
        a_2 & 0
    \end{pmatrix}\label{def:D}
\end{align}
Then, we have 
$
    \bK_{\lambda } = \bU\bD\bU^\top+\bK_*$,
where $\bK_*$ satisfies
\begin{equation}\label{eq:bound_K_*}
    c\bI \preccurlyeq\bK_*\preccurlyeq Cd^{\frac{1}{6}}\bI,
\end{equation}
with probability at least $1-O(d^{-\frac{1}{48}})$, for some constants $c,C>0$. This is based on Theorem~\ref{thm:concentration} and Lemma~\ref{lem:bound_XX_2}. Then, applying the Sherman-Morrison-Woodbury formula again, we can derive that
\begin{align}
    \bU^\top\bK_{\lambda }^{-1}\bU  =  ~& \bU^\top\bK_*^{-1}\bU-  \bU^\top\bK_*^{-1}\bU(\bD^{-1}+\bU^\top\bK_*\bU)^{-1}\bU^\top\bK_*^{-1}\bU \\
    =~&(\bI-  \bU^\top\bK_*^{-1}\bU(\bD^{-1}+\bU^\top\bK_*\bU)^{-1})\bU^\top\bK_*^{-1}\bU \\
    =~& \bD^{-1}(\bD^{-1}+\bU^\top\bK_*\bU)^{-1}\bU^\top\bK_*^{-1}\bU\\
  =  ~& \bD^{-1}-\bD^{-1}(\bD^{-1}+\bU^\top\bK_*\bU)^{-1}\bD^{-1}\\
    =~& \bD^{-1}-(\bD+\bD\bU^\top\bK_*\bU\bD)^{-1}.\label{eq:UKU_decomp}
\end{align}
\begin{lemma}\label{lemm:bias1_deterministic}
    Under the assumptions of Theorem~\ref{thm:test_limit_deterministic}, we have
    $|\cR_1|\lesssim d^{-0.4}$,
    with probability at least $1-O(d^{-\frac{1}{48}})$,
    where $\cR_1$ is defined in \eqref{eq:bias1_deterministic}.
\end{lemma}
\begin{proof}
Recall that $ \vmu =d\cdot \vf_* = [\bx_1^\top\bSigma\bx_1,\ldots,\bx_n^\top\bSigma\bx_n]^\top$. Then $\E[ \vmu ]=\Tr( \bSigma^2)\1$. 
 Define $\bar \vmu:= \vmu-\Tr(\bSigma^2)\ones$. Thus, \eqref{eq:vmu-concen_norm} indicates that
 \begin{align}\label{eq:vmu-concen_norm2}
     \|\bar \vmu\|\lesssim d^{1.6},\quad \| \vmu\|\lesssim d^{2},
 \end{align}with probability at least $1-d^{-1}$.
Recall the definitions of $\bU$ and $\bD$ in \eqref{def:U} and \eqref{def:D}. From the definition of $\cR_1$,  we can simplify it as 
\begin{align}
    \cR_1=~&\frac{1}{d^2}\big( \Tr(\bSigma^2)-(a_2\vmu+a_0\1)^\top\bK_{\lambda}^{-1} \vmu)^2 \\
    =~& \frac{1}{d^2}\left(\Tr(\bSigma^2)- \begin{pmatrix}
        \frac{a_0}{\sqrt{a_2}}& \sqrt{a_2}
    \end{pmatrix}\bU^\top\bK_{\lambda }^{-1}\bU\begin{pmatrix}
        0\\ \sqrt{a_2}
    \end{pmatrix} \right)^2.
\end{align}
Then, applying \eqref{eq:UKU_decomp}, we can get
\begin{align}
   & \Tr(\bSigma^2)-  \begin{pmatrix}
      \frac{a_0}{\sqrt{a_2}}& \sqrt{a_2}
    \end{pmatrix}\bU^\top\bK_{\lambda }^{-1}\bU\begin{pmatrix}
        0\\ \sqrt{a_2}
    \end{pmatrix}
    =\begin{pmatrix}
      \frac{a_0}{\sqrt{a_2}}& \sqrt{a_2}
    \end{pmatrix}(\bD+\bD\bU^\top\bK_*\bU\bD)^{-1}\begin{pmatrix}
        0\\ \sqrt{a_2}
    \end{pmatrix},
\end{align}
where we employ the identity:
$
    \begin{pmatrix}
        \frac{a_0}{\sqrt{a_2}}& \sqrt{a_2}
    \end{pmatrix}\bD^{-1} \begin{pmatrix}
        0\\ \sqrt{a_2}
    \end{pmatrix}=\Tr(\bSigma^2)$.
Moreover, by calculation of the inverse of the $2\times 2$ matrix, we know that
\begin{small}
\begin{align}
    &\begin{pmatrix}
      \frac{a_0}{\sqrt{a_2}}& \sqrt{a_2}
    \end{pmatrix}(\bD+\bD\bU^\top\bK_*\bU\bD)^{-1}\begin{pmatrix}
        0\\ \sqrt{a_2}
    \end{pmatrix}=\\
    ~& \frac{(a_0-a_2\Tr(\bSigma^2))(\ones^\top\bK_*^{-1}\bar\vmu)+a_2\vmu^\top\bK_*^{-1}\vmu-a_2\Tr(\bSigma^2)\vmu^\top\bK_*^{-1}\ones }{-1-a_0\ones^\top\bK_*^{-1}\ones +2a_2\ones^\top \bK_*^{-1}\bar\vmu-a_2\Tr(\bSigma^2)\ones^\top\bK_*^{-1}\ones +a_2^2(\vmu^\top\bK_*^{-1}\vmu\cdot \ones^\top\bK_*^{-1}\ones -(\ones^\top\bK_*^{-1}\vmu)^2)}.
\end{align}
\end{small}
Then, we control each term in the above fraction. For the numerator, by \eqref{eq:vmu-concen_norm2}, we have
\begin{align}\label{eq:numerator}
    \big|(a_0-a_2\Tr(\bSigma^2))(\ones^\top\bK_*^{-1}\bar\vmu)+a_2\vmu^\top\bK_*^{-1}\vmu-a_2\Tr(\bSigma^2)\vmu^\top\bK_*^{-1}\ones \big|\lesssim d^{2.6}
\end{align}
with probability at least $1-d^{-1}$. For the denominator, from \eqref{eq:bound_K_*}, we can easily see that
\begin{align}\label{eq:denominator1}
     O(d^{\frac{11}{6}})=nd^{-\frac{1}{6}}\lesssim a_0\ones^\top\bK_*^{-1}\ones\lesssim d^2,
\end{align} with high probability. Meanwhile, by \eqref{eq:bound_K_*} and \eqref{eq:vmu-concen_norm2}, 
\begin{align}\label{eq:denominator2}
    a_2|\ones^\top \bK_*^{-1}\bar\vmu|\lesssim d^{0.6},\quad a_2\Tr(\bSigma^2)\ones^\top\bK_*^{-1}\ones \lesssim d
\end{align}with high probability. Lastly, \eqref{eq:bound_K_*} and \eqref{eq:vmu-concen_norm2} also indicate that
\begin{align}
    &a_2^2(\vmu^\top\bK_*^{-1}\vmu\cdot \ones^\top\bK_*^{-1}\ones -(\ones^\top\bK_*^{-1}\vmu)(\ones^\top\bK_*^{-1}\vmu))\\
    =~& a_2^2(\vmu^\top\bK_*^{-1}\vmu\cdot \ones^\top\bK_*^{-1}\ones -(\ones^\top\bK_*^{-1}\bar\vmu+\Tr(\bSigma^2)\cdot\ones^\top\bK_*^{-1}\ones)(\ones^\top\bK_*^{-1}\vmu))\\
    =~& a_2^2(\bar\vmu^\top\bK_*^{-1}\vmu\cdot \ones^\top\bK_*^{-1}\ones -(\ones^\top\bK_*^{-1}\bar\vmu)(\ones^\top\bK_*^{-1}\vmu))=O(d^{1.6}) \label{eq:denominator3}
\end{align}with high probability. Combining \eqref{eq:denominator1}, \eqref{eq:denominator2}, and \eqref{eq:denominator3}, we can get 
\begin{small}
  \begin{equation}
   \big| -1-a_0\ones^\top\bK_*^{-1}\ones +2a_2\ones^\top \bK_*^{-1}\bar\vmu-a_2\Tr(\bSigma^2)\ones^\top\bK_*^{-1}\ones +a_2^2(\vmu^\top\bK_*^{-1}\vmu\cdot \ones^\top\bK_*^{-1}\ones -(\ones^\top\bK_*^{-1}\vmu)^2)\big|\ge d^{\frac{11}{6}}.
\end{equation}  
\end{small}
Therefore, with \eqref{eq:numerator}, we can conclude this lemma.
\end{proof}

\begin{lemma}\label{lemm:bias2_deterministic}
    Under the assumptions of Theorem~\ref{thm:test_limit_deterministic}, we have $|\cR_2|\lesssim d^{-1/2}$,
    with probability at least $1-O(d^{-1/2})$,
    where $\cR_2$ is defined in \eqref{eq:bias_deterministic}.
\end{lemma}
\begin{proof}
    By the assumption of $\bSigma$, we know that $|\Tr[\bSigma^4]|\lesssim d$ and $ \vmu:=d\vf_*$. Then for the second term in $\cR_2$, we have
    \begin{align}
        a_2^2\vf_*^\top \bK_{\lambda}^{-1}(\bX\bSigma\bX^\top)^{\odot 2}\bK_{\lambda}^{-1} \vf_*\lesssim~&  \frac{1}{d^4} \vmu^\top\bK_{\lambda}^{-1}a_2(\bX\bSigma\bX^\top)^{\odot 2}\bK_{\lambda}^{-1}  \vmu
        \lesssim~ \frac{1}{d^4} \vmu^\top\bK_{\lambda}^{-1} \vmu \lesssim \frac{1}{d} 
    \end{align}
    with probability at least $1-O(d^{-\frac{1}{2}})$, where we employ Lemmas~\ref{lemm:inverseK_mix_bound} and~\ref{lemm:uKu}. 
    Lastly, in the third term of $\cR_2$, by the definition of $\bv_0^{(2)}$ in \eqref{eq:v_02}, with a slight modification of Lemma~\ref{lemm:uKu}, we can derive $
        \frac{4a_2}{d}|\bv_0^{(2)}\bK_{\lambda}^{-1} \vf_*|\lesssim \frac{1}{d^4}|\bv_0^{(2)}\bK_{\lambda}^{-1} \vmu |\lesssim \frac{1}{d}$
 with probability at least $1-O(d^{-\frac{1}{2}})$.
\end{proof}

\begin{lemma}\label{lemm:bias3_deterministic}
Under the assumptions of Theorem~\ref{thm:test_limit_deterministic}, we have
    $|\cR_{\mix}|\lesssim d^{-0.3}$,
    with probability at least $1-O(d^{-\frac{1}{48}})$,
    where 
  $\cR_{\mix}$ is defined  by \eqref{eq:bias2_deterministic}.
\end{lemma}
\begin{proof}
We control the terms in \eqref{eq:bias2_deterministic}, respectively. Firstly, recall $\widetilde\bb:=a_2\vmu+a_0\1$ from \eqref{eq:def_tilde_b} and $\bb_0$ from Lemma~\ref{eq:b0i}. Then, for any $i\in [n]$, the $i$-th entry 
\begin{align}
   ( \bb_0-\widetilde\bb)_i= \frac{f^{(4)}(0)}{8d^4}(t_i-\Tr(\bSigma^2))^2+\frac{15t_i^3f^{(6)}(0)}{6!d^6}.
\end{align}Therefore, by \eqref{eq:t_i_concen}, we know that
$\| \bb_0-\widetilde\bb\|\lesssim d^{-1.9}$, with probability at least $1-O(d^{-1})$. Hence, by \eqref{eq:vmu-concen_norm2} and \eqref{eq:K_lambda_-1}, we have 
  $$  |\vf_*^\top \bK_{\lambda}^{-1}(\bb_0-\widetilde\bb)|\lesssim \frac{1}{d}\norm{\vmu}\cdot\norm{\bb_0-\widetilde\bb}\lesssim d^{-0.9}.$$
Moreover, Lemma~\ref{lemm:uKu} verifies that with probability at least $1-O(d^{-1/2})$,
 $$ |\vf_*^\top \bK_{\lambda}^{-1} \widetilde\bb |\lesssim \frac{1}{d^3}\vmu^\top \bK_{\lambda}^{-1}\vmu+\frac{1}{d}|\ones^\top \bK_{\lambda}^{-1} \vmu|\lesssim  d^{0.6}.$$
Thus, combining all the above, we have with probability at least $1-O(d^{-1/2})$,
\begin{align}
   \big|\vf_*^\top \bK_{\lambda}^{-1}(\bb_0\bb_0^\top-\widetilde\bb\widetilde\bb_0^\top)\bK_{\lambda}^{-1} \vf_*\big|
   \le& |\vf_*^\top \bK_{\lambda}^{-1}(\bb_0-\widetilde\bb)|^2+|\vf_*^\top \bK_{\lambda}^{-1}(\bb_0-\widetilde\bb)||\vf_*^\top \bK_{\lambda}^{-1} \widetilde\bb |\lesssim d^{-0.3}.
\end{align} 
Similarly, we can verify
$\big|\frac{\Tr( \bSigma^2)}{d}(\bb_0-\widetilde\bb )^\top\bK_{\lambda}^{-1} \vf_*\big|\lesssim d^{-0.9}$.
Next, by \eqref{eq:K_lambda_-1}, Lemmas~\ref{lem:bound_XX_2}, ~\ref{lemm:b_k,i} and~\ref{lemm:uKu}, we have with probability at least $1-O(d^{-\frac{1}{48}})$,
\begin{align}
    &\vf_*^\top \bK_{\lambda}^{-1}\diag(\bb_1-\widetilde\bb_1)\bX\bSigma\bX^\top\diag(\bb_1-\widetilde\bb_1)\bK_{\lambda}^{-1} \vf_*\\
    \le~& \frac{1}{d^2}\vmu^\top\bK_{\lambda}^{-1}\vmu\cdot\max_{i\in[n]}|b_{1,i}-\widetilde b_{1,i}|^2\cdot \|\bX\bSigma\bX^\top\|\lesssim d^{-2}.
\end{align}
Moreover, Lemma~\ref{lemm:b0Kb0} shows that $\Tr\bK_{\lambda}^{-1}\bb_0\bb_0^\top\bK_{\lambda}^{-1} = \bb_0^\top\bK_{\lambda}^{-2}\bb_0\lesssim d^{-0.8}$ with probability $1-O(d^{-\frac{1}{48}})$. Lastly, by \eqref{eq:K_lambda_-1}, Lemmas~\ref{lem:bound_XX_2}, ~\ref{lemm:b_k,i} and~\ref{lemm:uKu},  with probability $1-O(d^{-\frac{1}{48}})$,
\begin{align} 
 &\Tr\bK_{\lambda}^{-1}\diag(\bb_1-\widetilde\bb_1)\bX\bSigma\bX^\top\diag(\bb_1-\widetilde\bb_1)\bK_{\lambda}^{-1}\\
 \le ~& \sqrt{d} \|\bK_{\lambda}^{-1}\|^2\cdot \|\bX\bSigma\bX^\top\|\cdot \max_{i\in[n]}|b_{1,i}-\widetilde b_{1,i}|^2\lesssim d^{-3}.
\end{align}
\end{proof}

\begin{proofoftheorem}{\ref{thm:test_limit_deterministic}}
Combining Lemmas~\ref{lemma:diff_R_Rtilde_det}, \ref{lemm:bias1_deterministic}, \ref{lemm:bias2_deterministic}, and \ref{lemm:bias3_deterministic},  we can obtain that
$|\cR(\lambda)-\sigma_\varepsilon^2\cV|\lesssim d^{-1/4}$,
with probability at least $1-O(d^{-1/48})$ for any $\lambda\ge 0$.
Here we utilized the decomposition of $\widetilde\cR(\lambda)$ in \eqref{eq:decomp_test_3_appro_decomp}. Hence, it suffices to analyze the limit of the variance term $\cV$ defined in Lemma~\ref{lemm:risk_decomp}. Because of Lemma~\ref{lemm:approx_variance} and the approximation of $\cV_0$ in \eqref{eq:cV_0}, we can copy the analysis of $\cV_0$ in the proof of Theorem~\ref{thm:test_limit} to conclude that 
$|\cR(\lambda)-\sigma_{\varepsilon}^2\cV(\lambda_*)|\to 0$,
in probability, as $d\to\infty$ and $d^2/(2n)\to\alpha$, for any $\lambda\ge 0$, where $\cV(\lambda_*) $ is defined in \eqref{eq:variance_limit}. This completes the proof of 
Theorem~\ref{thm:test_limit_deterministic}.  
\end{proofoftheorem}

\subsection{Proof of Corollary~\ref{cor:gcv}}
    Based on the proof of Theorem~\ref{thm:train_limit} and Theorem~\ref{thm:globallaw}, we have
    \begin{align}
      \frac{1}{n}\by^\top (\bK+\lambda \bI)^{-2} \by&\to \lambda^2\int \frac{\frac{1}{\alpha} x+\sigma_{\bepsilon}^2}{\left(\frac{f''(0)}{4\alpha}x+a_*+\lambda\right)^2}~d\mu_{\alpha,\bSigma^{(2)}}(x),\\
        \frac{1}{n}\Tr ((\bK+\lambda \bI)^{-1}) &\to \left(\frac{4\alpha}{f''(0)}\right) \cdot \int \frac{1}{(x+\frac{4\alpha}{f''(0)}(a_*+\lambda))}d\mu_{\bSigma^{(2)}}( x),
    \end{align}
    in probability. For simplicity, we denote $A=\frac{4\alpha}{f''(0)}$ and $z = -A(a_*+\lambda)$. Let the Stieltjes transform of $\mu_{\bSigma^{(2)}}$ be $m(z)$. Then, we have
    \begin{align}
        \mathrm{GCV_{\lambda}(\bK,\by)}&\to \frac{A}{\alpha}\left(-z\frac{m'(-z)}{m(-z)}+1+\frac{1}{2z}-\frac{1}{2zm(-z)}\right)+\sigma_{\bepsilon}^2A \frac{m'(-z)}{m(-z)}.\label{eq:gcv_limit}
    \end{align} in probability as $n\to\infty$. Recall the companion Stieltjes transform $\widetilde m(z)$ for $m(z)$ defined in Definition~\ref{def:free_convolution} and the relation between $m(z)$ and $\widetilde m(z)$: $\widetilde m(z):=\alpha m(z)+(1-\alpha)(-1/z).$
Then we can rewrite \eqref{eq:gcv_limit} in terms of $\widetilde m(z)$.
Then, we can apply (4) and Lemma 2.2 by \cite{dobriban2018high}, and the proof of Theorem~\ref{thm:test_limit} to conclude the proof.


\bibliography{aux/ref}

\end{document}